\documentclass[twoside,11pt]{article}
\usepackage{jmlr2e}
\usepackage{hyperref,url}            
\usepackage{booktabs,nicefrac,microtype}      
\usepackage{times,graphicx,subfigure}
\usepackage{algorithm,algorithmic}
\usepackage{amsmath}\allowdisplaybreaks
\usepackage{amsfonts,amssymb}
\usepackage{lastpage}
\newtheorem{asm}{Assumption}
\newtheorem{thm}{Theorem}
\newtheorem{lem}{Lemma}
\newtheorem{cor}{Corollary}

\jmlrheading{20}{2019}{1-\pageref{LastPage}}{12/17; Revised
8/18}{2/19}{17-773}{Luo Luo, Cheng Chen, Zhihua Zhang, Wu-Jun Li and Tong Zhang}
\ShortHeadings{RFD with Application in Online Learning}{Luo, Chen, Zhang, Li and Zhang}
\firstpageno{1}

\begin{document}

\def\f {{\bf f}}
\def\A {{\bf A}}
\def\tA {\tilde{\bf A}}
\def\B {{\bf B}}
\def\hB {{\hat{\bf B}}}
\def\tB {{\tilde{\bf B}}}
\def\C {{\bf C}}
\def\D {{\bf D}}
\def\G {{\bf G}}
\def\bH {{\bf H}}
\def\I {{\bf I}}
\def\bL {{\bf L}}
\def\bLambda {{\bf \Lambda}}
\def\bOmega {{\bf \Omega}}
\def\M {{\bf M}}
\def\P {{\bf P}}
\def\Q {{\bf Q}}
\def\U {{\bf U}}
\def\V {{\bf V}}
\def\K {{\bf K}}
\def\bS {{\bf S}}
\def\bSigma {{\bf \Sigma}}
\def\T {{\bf T}}
\def\W {{\bf W}}
\def\X {{\bf X}}
\def\oW {{\tilde\W}}
\def\a {{\bf a}}
\def\bu {{\bf u}}
\def\g {{\bf g}}
\def\hg {\hat{\bf g}}
\def\w {{\bf w}}
\def\x {{\bf x}}
\def\y {{\bf y}}
\def\z {{\bf z}}
\def\ut {{(t)}}
\def\utp {{(t+1)}}
\def\utm {{(t-1)}}
\def\utt {{(t')}}
\def\uttm {{(t'-1)}}
\def\ui {{(i)}}
\def\uip {{(i+1)}}
\def\uim {{(i-1)}}
\def\dalpha {{\Delta\alpha}}
\def\fb {{\mathcal B}}
\def\fK {{\mathcal K}}
\def\fS {{\mathcal S}}
\def\fC {{\mathcal C}}
\def\fN {{\mathcal N}}
\def\fO {{\mathcal O}}
\def\BE {{\mathbb E}}
\def\BR {{\mathbb R}}
\def\bz {{\bf 0}}
\def\argmax {\mathop {\rm argmax}}
\def\argmin {\mathop {\rm argmin}}
\def\diag {{\rm diag}}
\def\svd {{\rm svd}}
\def\rk {{\rm rank}}
\def\tr {{\rm tr}}
\def\regret {{\rm regret}}
\def\uin {{|\!|\!|}}
\def\hat {\widehat}

\title{Robust Frequent Directions with Application in Online Learning}

\author{\name Luo Luo \email ricky@sjtu.edu.cn \\
        \name Cheng Chen  \email jack\_chen1990@sjtu.edu.cn \\
       \addr Department of Computer Science and Engineering \\
       Shanghai Jiao Tong University \\
       800 Dongchuan Road, Shanghai, China 200240
       \AND
       \name Zhihua Zhang\thanks{Corresponding author.} \email zhzhang@math.pku.edu.cn \\
       \addr  National Engineering Lab for Big Data Analysis and Applications \\
        School of Mathematical Sciences \\
       Peking University \\
       5 Yiheyuan Road, Beijing, China 100871
       \AND
       \name Wu-Jun Li \email liwujun@nju.edu.cn \\
       \addr
        National Key Laboratory for Novel Software Technology \\
        Collaborative Innovation Center of Novel Software Technology and Industrialization \\
        Department of Computer Science and Technology \\
        Nanjing University \\
        163 Xianlin Avenue, Nanjing, China 210023
       \AND
       \name Tong Zhang \email tongzhang@tongzhang-ml.org \\
       \addr
       Computer Science \& Mathematics \\
       Hong Kong University of Science and Technology \\
       Hong Kong}
\editor{Qiang Liu}

\maketitle

\begin{abstract}
The frequent directions (FD) technique is a deterministic approach for
online sketching that has many applications
in machine learning.
The conventional FD is a heuristic procedure that often outputs rank deficient matrices.
To overcome the rank deficiency problem, we propose a new sketching strategy called robust frequent directions (RFD) by introducing a regularization term.
RFD can be derived from an optimization problem. It updates the sketch matrix and the regularization term adaptively and jointly.
RFD reduces the approximation error of FD without increasing the computational cost.
We also apply RFD to online learning and propose an effective hyperparameter-free online Newton algorithm.
We derive a regret bound for our online Newton algorithm based on RFD,
which guarantees the robustness of the algorithm.
The experimental studies demonstrate that the proposed method outperforms state-of-the-art second order online learning algorithms.
\end{abstract}


\section{Introduction}

The sketching  technique is a powerful tool to deal with large scale
matrices \citep{DBLP:journals/siamcomp/GhashamiLPW16,halko2011finding,woodruff2014sketching}, and it has been widely used to speed up machine learning algorithms such as second order optimization algorithms
\citep{DBLP:conf/nips/ErdogduM15,luo2016efficient,DBLP:journals/siamjo/PilanciW17,
DBLP:journals/corr/Roosta-Khorasani16,DBLP:journals/corr/Roosta-Khorasani16a,DBLP:conf/nips/XuYRRM16,YeLuoZhang2017}.
There exist several families of matrix sketching strategies,
including sparsification, column  sampling, random projection \citep{achlioptas2003database,indyk1998approximate,kane2014sparser,DBLP:journals/jmlr/WangLZ16}, and
frequent directions (FD) \citep{DBLP:journals/tkde/DesaiGP16,DBLP:journals/siamcomp/GhashamiLPW16,
DBLP:conf/icml/Huang18,liberty2013simple,DBLP:conf/aistats/MrouehMG17,DBLP:conf/ijcai/YeLZ16}.

Sparsification techniques \citep{achlioptas2007fast,achlioptas2013near,arora2006fast,drineas2011note}
generate a sparse version of the matrix by element-wise sampling,
which allows the matrix multiplication to be more efficient with lower space.
Column (row) sampling algorithms \citep{DBLP:journals/ftml/Mahoney11} include the importance sampling \citep{drineas2006fasta,drineas2006fastb,frieze2004fast} and leverage score sampling \citep{drineas2012fast,drineas2008relative,papailiopoulos2014provable}.
They define a probability for each row (column) and select a subset by the probability to construct the estimation.
Random projection maps the rows (columns) of the matrix into lower dimensional space by a projection matrix.
The projection matrix can be constructed in various ways \citep{woodruff2014sketching} such as
Gaussian random projections \citep{johnson1984extensions,sarlos2006improved},
fast Johnson-Lindenstrauss transforms \citep{ailon2006approximate,ailon2009fast,ailon2013almost,kane2014sparser} and
sparse random projections \citep{clarkson2013low,nelson2013osnap}.
The frequent directions \citep{DBLP:journals/siamcomp/GhashamiLPW16,liberty2013simple} is a deterministic sketching algorithm  and achieves optimal tradeoff between approximation error and space.

In this paper we are especially concerned with the FD sketching  \citep{DBLP:journals/tkde/DesaiGP16,DBLP:journals/siamcomp/GhashamiLPW16,liberty2013simple},
because it is a stable online sketching approach. FD considers the matrix
approximation in the streaming setting.
In this case, the data is available in a sequential order and should be processed immediately.
Typically, streaming algorithms can only use limited memory at any time.
The FD algorithm extends the method of frequent items \citep{misra1982finding} to matrix approximation  and has tight approximation error bound.
However, FD usually leads to a rank deficient approximation, which in turn makes its applications less robust.
For example, Newton-type algorithms require a non-singular and well-conditioned approximated Hessian matrix but FD sketching usually generates  low-rank matrices.
An intuitive and simple way to conquer this gap is to introduce a regularization term to enforce the matrix to be invertible \citep{luo2016efficient,DBLP:journals/corr/Roosta-Khorasani16,DBLP:journals/corr/Roosta-Khorasani16a}. Typically, the regularization parameter is regarded as a hyperparameter and its  choice  is separable from the sketching procedure. Since the regularization parameter affects the the performance heavily in practice, it should be chosen carefully.

To overcome the weakness of the FD algorithm, we propose a new sketching approach that we call  \emph{robust frequent directions} (RFD).
Unlike conventional sketching methods which only approximate the matrix with a low-rank structure,
RFD constructs the low-rank part and updates the regularization term simultaneously. In particular,
the update procedure of RFD can be regarded as solving an optimization problem (see Theorem \ref{thm:optRFD}).
This method is different from the standard FD, giving rise to a tighter error bound.

Note that \citet{zhang2014matrix} proposed matrix ridge approximation
(MRA) to approximate a positive semi-definite matrix using an idea similar to RFD.  There are two main differences between RFD and MRA.
First, RFD is designed for the case that data samples come sequentially and memory is limited,
while MRA has to access the whole data set.
Second, MRA aims to minimize the approximation error with respect to the Frobenius norm
while RFD tries to minimize the spectral-norm approximation error.
In general, the spectral norm error bound is more meaningful than the Frobenius norm error bound \citep{DBLP:journals/ftml/Tropp15}.

In a recent study,  \citet{luo2016efficient} proposed a FD-based
sketched online Newton (SON) algorithm (FD-SON) to accelerate the
standard online Newton algorithms.
Owing to the shortcoming of FD, the performance of FD-SON
is significantly affected by the choice of the hyperparameter.
Naturally, we can leverage RFD to improve online Newton algorithms.
Accordingly, we propose a sketched online Newton step based on RFD (RFD-SON).
Different from conventional sketched Newton algorithms, RFD-SON is hyperparameter-free.
Setting the regularization parameter to be zero initially,
RFD-SON will adaptively increase the regularization term. The approximation Hessian will be well-conditioned after a  few iterations.
Moreover, we prove that RFD-SON has a more robust regret bound than FD-SON, and the experimental results also validate better performance of RFD-SON.

The remainder of the paper is organized as follows.
In Section~\ref{sec:notation} we  present notation and preliminaries.
In Section~\ref{sec:introduction} we review the background of second order online learning and its sketched variants.
In Sections~\ref{sec:rfd} and \ref{sec:ons-rfd} we propose our robust frequent directions (RFD) method and the applications in online learning, with some related theoretical analysis.
In Section~\ref{sec:exp} we demonstrate empirical comparisons with baselines on serval real-world data sets to show the superiority of our algorithms. Finally, we conclude our work in Section~\ref{sec:concl}.

\section{Notation and Preliminaries}
\label{sec:notation}

We let $\I_d$ denote the $d{\times} d$ identity matrix.
For a matrix $\A=[A_{ij}]\in\BR^{n\times d}$ of rank
$r$ where $r\leq \min(n, d)$, we let the condensed singular value
decomposition (SVD) of $\A$ be $\A = \U\bSigma\V^\top$ where
$\U\in\BR^{n\times r}$ and $\V\in\BR^{d\times r}$ are column orthogonal
and $\bSigma=\diag(\sigma_1,\sigma_2,\dots,\sigma_r)$ with $\sigma_1 \geq \sigma_2 \geq \dots, \geq \sigma_r > 0$ places the nonzero singular values
on its diagonal entries.

We use $\sigma_{\max}(\A)$ to denote the largest singular
value and $\sigma_{\min}(\A)$ to denote the smallest non-zero
singular value. Thus, the condition number of $\A$ is $\kappa(\A)=\frac{\sigma_{\max}(\A)}{\sigma_{\min}(\A)}$.
The matrix pseudoinverse of $\A$ is defined by $\A^\dag=\V\bSigma^{-1}\U^\top \in \BR^{d\times n}$.

Additionally, we let $\|\A\|_F = \sqrt{\sum_{i,j}A_{ij}^2}=\sqrt{\sum_{i=1}^r \sigma_{i}^2}$ be the Frobenius norm and
$\|\A\|_2 = \sigma_{\max}(\A)$ be the spectral norm.
A matrix norm $\uin \cdot \uin$ is said to be unitarily invariant if
$\uin\P\A\Q \uin=\uin \A \uin$ for any unitary matrices $\P\in\BR^{n\times n}$ and $\Q\in\BR^{d\times d}$.
It is easy to verify that both the Frobenius norm and spectral norm are unitarily invariant.
We define $[\A]_k=\U_k\bSigma_k\V_k^\top$ for $k\leq r$, where $\U_k\in\BR^{n\times k}$ and $\V_k\in\BR^{d\times k}$ are the first $k$ columns of $\U$ and $\V$,
and $\bSigma_k=\diag(\sigma_1,\sigma_2,\dots,\sigma_k)\in\BR^{k\times k}$. Then $[\A]_k$ is the best rank-$k$ approximation to $\A$ in
both the Frobenius and spectral norms, that is,
\begin{align*}
    [\A]_k = \argmin_{\rk(\hat \A)\leq k} \| \A-\hat \A\|_F = \argmin_{\rk(\hat \A)\leq k} \| \A-\hat \A\|_2.
\end{align*}
Given a positive semidefinite matrix $\bH\in\BR^{d\times d}$, the notation $\|\x\|_\bH$ is called $\bH$-norm of vector $\x\in\BR^d$, that is, $\|\x\|_\bH=\sqrt{\x^\top\bH\x}$. If matrices $\A$ and $\B$ have the same size, we let $\langle \A, \B \rangle$ denote $\sum_{i,j} A_{ij}B_{ij}$.

\subsection{Frequent Directions}

We give a brief review of frequent directions (FD) \citep{DBLP:journals/siamcomp/GhashamiLPW16,liberty2013simple},
because it is closely related to our proposed method.
FD is a deterministic matrix sketching in the row-updates model.
For any input matrix $\A\in\BR^{T\times d}$ whose rows come  sequentially,
it maintains a sketch matrix $\B\in\BR^{(m-1)\times d}$ with $m\ll T$ to approximate $\A^\top\A$ by $\B^\top\B$.

We present the detailed implementation of FD in Algorithm \ref{alg:FD}.
The intuition behind FD is similar to that of frequent items.
FD periodically shrinks orthogonal vectors by roughly the same amount (Line 5 of Algorithm \ref{alg:FD}).
The shrinking step reduces the square Frobenius norm of the sketch reasonable and guarantees that no direction
is reduced  too much.

\begin{algorithm}[ht]
    \caption{Frequent Directions}
	\label{alg:FD}
	\begin{algorithmic}[1]
    \STATE {\textbf{Input:}} $\A=[\a^{(1)},\dots,\a^{(T)}]^\top\in\BR^{T\times d}$,
            $\B^{(m-1)}=[\a^{(1)},\dots,\a^{(m-1)}]^\top$ \\[0.1cm]
    \STATE {\textbf{for}} $t=m,\dots,T$ {\textbf{do}} \\[0.2cm]
    \STATE \quad ${\hat\B}^\utm=\begin{bmatrix}\B^\utm \\ (\a^\ut)^\top \end{bmatrix}$ \\[0.2cm]
    \STATE \quad Compute SVD: ${\hat\B}^\utm =\U^\utm\bSigma^\utm(\V^\utm)^\top$ \\[0.2cm]
    \STATE \quad $\B^\ut = \sqrt{\big(\bSigma_{m-1}^\utm\big)^2-\big(\sigma_{m}^\utm\big)^2\I_{m-1}}\cdot \big(\V_{m-1}^\utm\big)^\top$
    \\[0.2cm]
    \STATE {\textbf{end for}} \\[0.2cm]
    \STATE {\textbf{Output:}} $\B=\B^{(T)}$
	\end{algorithmic}
\end{algorithm}

FD has the following error bound for any $k<m$,
\begin{align}
&\|\A^\top\A-\B^\top\B\|_2 \leq \frac{1}{m-k} \|\A-[\A]_k\|_F^2.  \label{bound:FD}
\end{align}
The above result means that the space complexity of FD is optimal regardless of streaming issues
because any algorithm satisfying $\|\A^\top\A-\B^\top\B\|_2\leq \|\A-[\A]_k\|^2_F/(m-k)$ requires $\fO(md)$ space to represent matrix $\B$
\citep{DBLP:journals/siamcomp/GhashamiLPW16}.
The dominated computation of the algorithm is computating the SVD of ${\hat\B}^\utm$,
which costs $\fO(m^2d)$ by the standard SVD implementation. However, the total cost can be reduced
from $\fO(Tm^2d)$ to $\fO(Tmd)$ by doubling the space (Algorithm \ref{alg:FFD} in Appendix \ref{appendix:alg}) or
using the Gu-Eisenstat procedure \citep{gu1994stable}.

\citet{DBLP:journals/tkde/DesaiGP16} proposed some extensions of FD. More specifically,
Parameterized FD (PFD) uses an extra hyperparameter to
describe the proportion of singular values shrunk in each iteration.
PFD improves the performance empirically, but has worse error bound  than  FD by a constant.
Compensative FD (CFD)  modifies the output of FD by increasing the singular values and
keeps the same error guarantees as FD.

\section{Online Newton Methods}
\label{sec:introduction}

For ease of demonstrating our work,
we would like to introduce sketching techniques in online learning scenarios.
First of all, we introduce the background of  convex online learning including online Newton step algorithms.
Then we discuss the connection between online learning and sketched second order methods,
which motivates us to propose a more robust sketching algorithm.

\subsection{Convex Online Learning}

Online  learning is performed in a sequence of consecutive rounds \citep{shalev2011online}.
We consider the problem of convex online  optimization as follows.
For a sequence of examples $\{\x^{(t)}\in\BR^d\}$, and
convex smooth loss functions $\{f_t: \fK_t\rightarrow\BR\}$ where $f_t(\w) \triangleq  \ell_t(\w^\top \x^{(t)})$
and $\fK_t\subset\BR^d$ are convex compact sets,
the learner outputs a predictor $\w^\ut$ and suffers the loss $f_t(\w^\ut)$ at the $t$-th round.
The cumulative regret at round $T$ is defined as:
\begin{align*}
    R_T(\w^*) = \sum_{t=1}^T f_t(\w^\ut) - \sum_{t=1}^T f_t(\w^*),
\end{align*}
where $\w^* = \argmin_{\w\in\fK}\sum_{t=1}^T f_t(\w)$ and $\fK={\mathop{\bigcap}}_{t=1}^T\fK_t$.

We make the following assumptions on the loss functions.
\begin{asm}\label{asm:bound}
    The loss functions $\ell_t$ satisfy $|\ell_t^\prime(z)|\leq L$ whenever $|z|\leq C$, where $L$ and $C$ are positive constants.
\end{asm}
\begin{asm}\label{asm:curv}
There exists a $\mu_t\leq0$ such that for all $\bu,\w \in\fK$, we have
\begin{align*}
        f_t(\w) \geq f_t(\bu) + \nabla f_t(\bu)^\top(\w-\bu) + \frac{\mu_t}{2}\big(\nabla f_t(\bu)^\top(\w-\bu)\big)^2.
\end{align*}
\end{asm}
Note that for a loss function $f_t$ whose domain and gradient have bounded diameter, holding Assumption \ref{asm:curv} only requires the exp-concave property,
which is more general than strong convexity \citep{hazan2016introduction}.
For example, the square loss function $f_t(\w)=(\w^\top\x_t-y_t)^2$ satisfies Assumption \ref{asm:curv}
with $\mu_t=\frac{1}{8C^2}$ if the function is subject to constraints $|\w^\top\x_t|\leq C$  and  $y_t \leq C$ \citep{luo2016efficient}, but it is not strongly convex.

One typical online learning algorithm is online gradient descent (OGD) \citep{hazan2007logarithmic,zinkevich2003online}.
At the ($t {+}1$)-th round,  OGD exploits the following update rules:
\begin{align*}
     \bu^\utp &=  \w^\ut-\beta_t \g^\ut, \\
     \w^\utp  &= \argmin_{\w\in\fK_{t+1}} \| \w - \bu^\utp \|,
\end{align*}
where $\g^\ut=\nabla f_t(\w^\ut)$ and $\beta_t$ is the learning rate.
The algorithm has linear computation cost and achieves $\fO(\frac{L^2}{H}\log T)$ regret bound for the $H$-strongly convex loss.

In this paper, we are more interested in online Newton step algorithms \citep{hazan2007logarithmic,luo2016efficient}.
The standard online Newton step keeps the curvature information in the matrix $\bH^{\ut}\in\BR^{d\times d}$ sequentially and iterates as follows:
\begin{align}
     \bu^\utp &=  \w^\ut-\beta_t(\bH^\ut)^{-1} \g^\ut, \nonumber \\
     \w^\utp &= \argmin_{\w\in\fK_{t+1}} \| \w - \bu^\utp \|_{\bH^\ut}.  \label{update:ONS}
\end{align}
The matrix $\bH^{\ut}$ is constructed by the outer product of historical gradients \citep{duchi2011adaptive,luo2016efficient}, such as
\begin{align}
  & \bH^{\ut}= \sum_{i=1}^t \g^\ui(\g^\ui)^\top + \alpha_0\I_d, \label{hessian:ONS0}\\
  & \text{or} ~~~~  \bH^{\ut}=\sum_{i=1}^t (\mu_t+\eta_t)\g^\ui(\g^\ui)^\top + \alpha_0\I_d, \label{hessian:ONS}
\end{align}
where $\alpha_0\geq0$ is a fixed regularization parameter, $\mu_t$ is the constant in Assumption \ref{asm:curv}, and
$\eta_t$ is typically chosen as $\fO(1/t)$.
The second order algorithms enjoy logarithmical regret bound without the strongly convex assumption
but require quadratical space and computation cost.
Some variants of online Newton algorithms have been applied to optimize neural networks \citep{martens2015optimizing,grosse2016kronecker,ba2017distributed}, but they do not provide theoretical guarantee on nonconvex cases.

\subsection{Efficient Algorithms by Sketching}

To make the online Newton step scalable, it is natural to use sketching techniques \citep{woodruff2014sketching}.
The matrix $\bH^{\ut}$ in online learning has the form $\bH^{\ut}=(\A^\ut)^\top\A^\ut + \alpha_0\I_d$,
where $\A^\ut\in\BR^{t\times d}$ is the corresponding term of (\ref{hessian:ONS0}) or (\ref{hessian:ONS})
such as
\begin{align*}
  \A^\ut=[\g^{(1)},\dots,\g^{(t)}]^\top,    ~~~~\text{or}~~~~
 \A^\ut=[\sqrt{\mu_1+\eta_1}\g^{(1)},\dots,\sqrt{\mu_t+\eta_t}\g^{(t)}]^\top.
\end{align*}
The sketching algorithm employs an approximation of $(\A^\ut)^\top\A^\ut$ by $(\B^\ut)^\top\B^\ut$,
where the sketch matrix $\B^\ut\in\BR^{m\times d}$ is much smaller than $\A^\ut$ and $m \ll d$.
Then we can use $(\B^\ut)^\top\B^\ut + \alpha_0\I_d$ to replace $\bH^{\ut}$ in update (\ref{update:ONS}) \citep{luo2016efficient}.
By the Woodbury identity formula, we can reduce the computation of the update from $\fO(d^2)$ to $\fO(m^2d)$ or $\fO(md)$.
There are several choices of sketching techniques, such as random projection \citep{achlioptas2003database,indyk1998approximate,kane2014sparser},
frequent directions \citep{DBLP:journals/siamcomp/GhashamiLPW16,liberty2013simple}
and Oja's algorithm \citep{oja1982simplified,oja1985stochastic}.
However, all above methods treat $\alpha_0$ as a given hyperparameter which is independent of the sketch matrix $\B^\ut$.
In practice, the performance of sketched online Newton methods is sensitive to the choice of the hyperparamter $\alpha_0$.

\section{Robust Frequent Directions}
\label{sec:rfd}
In many machine learning applications such as online learning
\citep{hazan2006efficient,hazan2007logarithmic,hazan2016introduction,luo2016efficient},
Gaussian process regression \citep{DBLP:books/lib/RasmussenW06} and kernel ridge regression \citep{DBLP:journals/jmlr/DrineasM05},
we usually require an additional regularization term to make the matrix invertible and well-conditioned, while conventional sketching methods only focus on the low-rank approximation.
On the other hand, the update of frequent directions is not optimal in the view of minimizing the approximation error in each iteration.
Both of them motivate us to propose robust frequent directions (RFD) that incorporates the update of sketch matrix
and the regularization term into one framework.

\subsection{The Algorithm}

The RFD approximates $\A^\top\A$ by $\B^\top\B + \alpha\I_d$ with $\alpha>0$.
We demonstrate the detailed implementation of RFD in Algorithm \ref{alg:RFD}.

\begin{algorithm}
    \caption{Robust Frequent Directions}
	\label{alg:RFD}
	\begin{algorithmic}[1]
    \STATE {\textbf{Input:}} $\A=[\a^{(1)},\dots,\a^{(T)}]^\top\in\BR^{T\times d}$,
            $\B^{(m-1)}=[\a^{(1)},\dots,\a^{(m-1)}]^\top$, $\alpha^{(m-1)}=0$ \\[0.1cm]
    \STATE {\textbf{for}} $t=m,\dots,T$ {\textbf{do}} \\[0.2cm]
    \STATE \quad ${\hat\B}^\utm=\begin{bmatrix}\B^\utm \\ (\a^\ut)^\top \end{bmatrix}$ \\[0.2cm]
    \STATE \quad Compute SVD: ${\hat\B}^\utm =\U^\utm\bSigma^\utm(\V^\utm)^\top$ \\[0.2cm]
    \STATE \quad $\B^\ut = \sqrt{\big(\bSigma_{m-1}^\utm\big)^2-\big(\sigma_{m}^\utm\big)^2\I_{m-1}}\cdot \big(\V_{m-1}^\utm\big)^\top$ \\[0.2cm]
    \STATE \quad $\alpha^\ut = \alpha^\utm + \big(\sigma^\utm_{m}\big)^2 / 2$ \\[0.1cm]
    \STATE {\textbf{end for}} \\[0.1cm]
    \STATE {\textbf{Output:}} $\B=\B^{(T)}$ and $\alpha=\alpha^{(T)}$.
	\end{algorithmic}
\end{algorithm}

The main difference between RFD and conventional sketching algorithms is the additional term $\alpha\I_d$.
We can directly use Algorithm \ref{alg:RFD} to approximate $\A^\top\A$ with $\alpha^{(m-1)}=\alpha_0>0$ if the target matrix is $\A^\top\A+\alpha_0\I_d$.
Compared with the standard FD,  RFD only needs to maintain one extra variable $\alpha^\ut$ by scalar operations in each iteration,
hence the cost of RFD is almost the same as FD.
Because the value of $\alpha^\ut$ is typically increasing from the $(m+1)$-th round in practice,
the resulting $\B^\top\B+\alpha\I_d$ is positive definite even the initial $\alpha^{(0)}$ is zero.
Also, we can further accelerate the algorithm by doubling the space.

\subsection{Theoretical Analysis}

Before demonstrating the theoretical results of RFD,
we review FD from the aspect of low-rank approximation which provides a motivation to the design of our algorithm.
At the $t$-th round iteration of FD (Algorithm \ref{alg:FD}), we have the matrix $\B^\utm$ which is used to approximate $(\A^\utm)^\top\A^\utm$ by $(\B^\utm)^\top\B^\utm$ and we aim to construct a new approximation which includes the new data $\a^\ut$, that is,
\begin{align}
(\B^\ut)^\top\B^\ut\approx (\B^\utm)^\top\B^\utm + \a^\ut(\a^\ut)^\top = ({\hat\B}^\utm)^\top{\hat\B}^\utm \label{FD:approx}.
\end{align}
The straightforward way to find $\B^\ut$ is to minimize the approximation error of (\ref{FD:approx}) based on the spectral norm with low-rank constraint:
\begin{align}
    \B'^\ut = \argmin_{\rk(\C)=m-1} \big\|(\hB^\utm)^\top\hB^\utm  - \C^\top\C\big\|_2. \label{FD:simple}
\end{align}
By the SVD of ${\hat\B}^\utm$, we have the solution $\B'^\ut=\bSigma_{m-1}^\utm\big(\V_{m-1}^\utm\big)^\top$.
In this view, the update of FD
\begin{align}
    \B^\ut = \sqrt{\big(\bSigma_{m-1}^\utm\big)^2-\big(\sigma_{m}^\utm\big)^2\I_{m-1}}\cdot \big(\V_{m-1}^\utm\big)^\top \label{FD:shrink}
\end{align}
looks imperfect, because it is not an optimal low-rank approximation.
However, the shrinkage operation in (\ref{FD:shrink}) is necessary.
If we take a greedy strategy \citep{brand2002incremental,hall1998incremental,levey2000sequential,ross2008incremental}
which directly replaces $\B^\ut$ with $\B'^\ut$ in FD,
it will perform worse in some specific cases\footnote{We provide an example in Appendix \ref{appendix:example}.} and also has no valid global error bound like (\ref{bound:FD}).

Hence, the question is: can we devise a method which enjoys the
optimality in each step and  maintains global tighter error bound in
the same time?
Fortunately, RFD is just such an algorithm holding both the properties.
We now explain the update rule of RFD formally, and provide the approximation error bound.
We first give the following theorem which plays an important role in our analysis.

\begin{thm}\label{thm:specMRA}
    Given a positive semi-definite matrix $\M\in\BR^{d\times d}$ and a positive integer $k<d$, let $\M=\U\bSigma\U^\top$ be the SVD of $\M$.
    Let $\U_k$ denote the matrix  of the first $k$ columns of $\U$ and $\sigma_k$ be the  $k$-th singular value of $\M$.
    Then the pair $({\hat\C}, {\hat\delta})$, defined as
    \begin{align*}
        \hat\C = \U_k(\bSigma_k-{\xi}\I_k)^{1/2}\Q \quad
        and \quad {\hat\delta} = (\sigma_{k+1} + \sigma_{d}) / 2
    \end{align*}
   where $\xi\in[\sigma_d,\sigma_{k+1}]$ and $\Q$ is an arbitrary $k\times k$ orthonormal matrix, is the global minimizer of
    \begin{align}
    \min_{\C\in\BR^{d\times k}, \delta\in\BR} \| \M-(\C\C^\top+\delta\I_d)  \|_2. \label{prob:thm1}
    \end{align}
    Additionally, we have
    \begin{align*}
         \| \M-({\hat\C}{\hat\C}^\top+{\hat\delta}\I_d) \|_2 \leq \| \M - \U_k\bSigma_k\U_k^\top \|_2,
    \end{align*}
    and the equality holds if and only if $\rk(\M)\leq k$.
\end{thm}
Theorem \ref{thm:specMRA} provides the optimal solution with the closed form for matrix approximation with a regularization term.
In the case of $\rk(\M)>k$, the approximation ${\hat\C}{\hat\C}^\top+{\hat\delta}\I_d$ is full rank and has strictly lower spectral norm error than the rank-$k$ truncated SVD.
Note that \citet{zhang2014matrix} has established the Frobenius norm based result about the optimal analysis\footnote{We also give a concise proof for the result of  \citet{zhang2014matrix}  in Appendix \ref{appendix:specMRA}.}.

Recall that in the streaming case, our goal is to approximate the concentration of historical approximation and current data at the $t$-th round.
The following theorem shows that the update of RFD is optimal with respect to the spectral norm for each step.
\begin{thm}\label{thm:optRFD}
Based on the updates in Algorithm \ref{alg:RFD},
we have
\begin{align}
       (\B^\ut, \alpha^\ut)
      =   \argmin_{\B\in\BR^{d\times (m-1)},  \alpha\in\BR}
          \big\| (\hB^\utm)^\top\hB^\utm + \alpha^\utm\I_d - (\B^\top\B + \alpha\I_d) \big\|_2. \label{prob:RFD}
\end{align}
\end{thm}
Theorem \ref{thm:optRFD} explains RFD from an optimization viewpoint. It shows that each step of RFD is optimal for current information.
Based on this theorem, the update of the standard FD corresponds $(\B, \alpha)= (\B^\ut, 0)$, which is not the optimal solution of (\ref{prob:RFD}). Intuitively, the regularization term of RFD compensates each direction for the over reduction from the shrinkage operation of FD.
Theorem \ref{thm:optRFD} also implies RFD is an online extension to the approximation of Theorem \ref{thm:specMRA}.
We can prove Theorem \ref{thm:optRFD} by using Theorem \ref{thm:specMRA} with $\M = (\hB^\utm)^\top\hB^\utm+\alpha^\utm\I_d$.
We defer the details to Appendix \ref{appendix:optRFD}.

RFD also enjoys a tighter approximation error than FD as the following theorem shows.
\begin{thm}\label{thm:RFDbound}
    For any $k<m$ and using the notation of Algorithm \ref{alg:RFD}, we have
    \begin{align}
        \big\|\A^\top\A-(\B^\top\B+\alpha\I_d) \big\|_2 \leq \frac{1}{2(m-k)}\| \A - [\A]_k \|_F^2, \label{bound:RFD}
    \end{align}
    where $[\A]_k$ is the best rank-k approximation to $\A$ in both the Frobenius and spectral norms.
\end{thm}
The right-hand side of inequality (\ref{bound:RFD}) is the half of the one in (\ref{bound:FD}),
which means RFD reduces the approximation error significantly with only one extra scalar.

The real applications usually consider the matrix with a regularization term.
Hence we also consider approximating the matrix $\M=\A^\top\A+\alpha_0\I_d$ where $\alpha_0>0$ and the rows of $\A$ are available in sequentially order.
Suppose that the standard FD approximates $\A^\top\A$ by $\B^\top\B$.
Then it estimates $\M$ as $\M_{\rm FD}=\B^\top\B+\alpha_0\I_d$. Meanwhile, RFD generates the approximation $\M_{\rm RFD}=\B^\top\B+\alpha\I_d$ by setting $\alpha^{(m-1)}=\alpha_0$.
Theorem \ref{thm:condition} shows that the condition number of $\M_{\rm RFD}$ is better than $\M_{\rm FD}$ and $\M$.
In general, the equality in Theorem \ref{thm:condition} usually can not be achieved for $t>m$ unless
$(\a^\ut)^\top$ lies in the row space of $\B^\utm$ exactly or the first $t$ rows of $\A$ have perfect low rank structure.
Hence RFD is more likely to generate a well-conditioned approximation than others.

\begin{thm}\label{thm:condition}
    With the notation of Algorithms \ref{alg:FD} and \ref{alg:RFD},
    let $\M=\A^\top\A+\alpha_0\I_d$, $\M_{\rm FD}=\B^\top\B+\alpha_0\I_d$, $\M_{\rm RFD}=\B^\top\B+\alpha\I_d$
    and $\alpha^{(m-1)}=\alpha_0$, where $\alpha_0 > 0$ is a fixed scalar.
    Then we have $\kappa(\M_{\rm RFD}) \leq \kappa(\M_{\rm FD})$
    and $\kappa(\M_{\rm RFD}) \leq \kappa(\M)$.
\end{thm}

\section{The Online Newton Step by RFD}
\label{sec:ons-rfd}

We now present the sketched online Newton step by robust frequent directions (RFD-SON).
The procedure is shown in Algorithm \ref{alg:RFD-ONS},
which is similar to sketched online Newton step (SON) algorithms \citep{luo2016efficient} but uses the new sketching method RFD.
The matrix $\bH^\ut$ in Line 10 will not be constructed explicitly in practice, which is only used to the ease of analysis.
The updates of $\bu^\ut$ and $\w^\ut$ can be finished in $\fO(md)$ time and space complexity by the Woodbury identity.
We demonstrate the details in Appendix \label{appendix:alg}.
When $d$ is large, RFD-SON is much efficient than the standard online Newton step with the full Hessian
that requires $\fO(d^2)$ both in time and space.

Note that we do not require the hyperparameter $\alpha_0$ to be strictly positive in RFD-SON.
In practice, RFD-SON always archives good performance by setting $\alpha_0=0$, which leads to a hyperparameter-free algorithm,
while the existing SON algorithm needs to select $\alpha_0$ carefully.
We consider the general case that $\alpha_0 \geq 0$ in this section for the ease of analysis.

\begin{algorithm}
    \caption{RFD for Online Newton Step}
	\label{alg:RFD-ONS}
	\begin{algorithmic}[1]
    \STATE {\textbf{Input:}} $\alpha^{(0)}=\alpha_0\geq0$, $m<d$, $\eta_t=\fO(1/t)$, $\w^{(0)}=\bz^d$ and
            $\B^{(0)}$ be empty. \\[0.15cm]
    \STATE {\textbf{for}} $t=0,\dots,T-1$ {\textbf{do}} \\[0.2cm]
    \STATE \quad Receive example $\x^\ut$, and loss function $f_t(\w)$ \\[0.2cm]
    \STATE \quad Predict the output of $\x^\ut$ by $\w^\ut$ and suffer loss $f_t(\w^\ut)$ \\[0.2cm]
    \STATE \quad $\g^\ut=\nabla f_t(\w^\ut)$ \\[0.2cm]
    \STATE \quad ${\hat\B}^\utm=\begin{bmatrix}\B^\utm \\ (\sqrt{\mu_t+\eta_t} \g^\ut)^\top \end{bmatrix}$ \\[0.2cm]
    \STATE \quad Compute SVD: $\hB^\utm = \U^\utm\bSigma^\utm(\V^\utm)^\top$ \\[0.2cm]
    \STATE \quad $\B^\ut = \sqrt{\big(\bSigma_{m-1}^\utm\big)^2-\big(\sigma^\utm_{m}\big)^2\I_{m-1}}\cdot (\V^\utm_{m-1})^\top$ \\[0.2cm]
    \STATE \quad $\alpha^\ut = \alpha^\utm + \big(\sigma^\utm_{m}\big)^2 / 2$ \\[0.2cm]
    \STATE \quad $\bH^\ut = (\B^\ut)^\top\B^\ut + \alpha^\ut\I_d$ \\[0.2cm]
    \STATE \quad $\bu^\utp =  \w^\ut - (\bH^\ut)^\dag \g^\ut$  \\[0.2cm]
    \STATE \quad $\w^\utp = \argmin_{\w\in\fK_{t+1}} \| \w - \bu^\utp \|_{\bH^\ut}$ \\[0.2cm]
    \STATE {\textbf{end for}} \\[0.2cm]
	\end{algorithmic}
\end{algorithm}

\begin{thm}\label{thm:regret}
    Let $\mu=\min_{t=1}^T\{\mu_t\}$ and $\fK={\mathop{\bigcap}}_{t=1}^T\fK_t$.
    Then under Assumptions \ref{asm:bound} and \ref{asm:curv} for any $\w\in\fK$,
    Algorithm \ref{alg:RFD-ONS} has the following regret for $\alpha_0>0$
    \begin{align}
            R_T(\w)
        \leq \frac{\alpha_0}{2}\|\w\|^2 + 2(CL)^2 \sum_{t=1}^T \eta_t
             + \frac{m}{2(\mu+\eta_T)}\ln\Big(\frac{\tr\big((\B^{(T)})^\top\B^{(T)}\big)}{m\alpha_0}
                +\frac{\alpha^{(T)}}{\alpha_0}\Big) + \Omega_{\rm RFD}               \label{bound:regretRFD}
    \end{align}
    where
    \begin{align*}
        \Omega_{\rm RFD}=  \frac{d-m}{2(\mu+\eta_T)}\ln\frac{\alpha^{(T)}}{{\alpha_0}} +
            \frac{m}{4(\mu+\eta_T)}\sum_{t=1}^T\frac{(\sigma_m^\ut)^2}{\alpha^\ut} + C^2 \sum_{t=1}^T(\sigma_m^\utm)^2.
    \end{align*}
\end{thm}

We present the regret bound of RFD-SON for positive $\alpha_0$ in Theorem \ref{thm:regret}.
The term $\Omega_{\rm RFD}$ in  (\ref{bound:regretRFD}) is the main gap between RFD-SON and the standard online Newton step without sketching.
 $\Omega_{\rm RFD}$ is dominated by  the last term which can be bounded as (\ref{bound:FD}).
If we exploit the standard FD to sketched online Newton step \citep{luo2016efficient} (FD-SON),
the regret bound is similar to (\ref{bound:regretRFD}) but the gap will be
\begin{align*}
    \Omega_{{\rm FD}} = \frac{m\Omega_k}{2(m-k)(\mu+\eta_T)\alpha_0},
\end{align*}
where $\Omega_k$ plays the similar role to term $\sum_{t=1}^T(\sigma_m^\utm)^2$ in RFD-SON
and the detailed definition can be found in \cite{luo2016efficient}.
This result is heavily dependent on the hyperparameter $\alpha_0$.
If we increase the value of $\alpha_0$, the gap $\Omega_{{\rm FD}}$ can be reduced
but the term $\frac{\alpha_0}{2}\|\w\|^2$ in the bound will increase, and vice versa.
In other words, we need to trade off $\frac{\alpha_0}{2}\|\w\|^2$ and $\Omega_{{\rm FD}}$ by tuning $\alpha_0$ carefully.
For RFD-SON, Theorem \ref{thm:regret} implies that we can set $\alpha_0$ be sufficiently small to reduce $\frac{\alpha_0}{2}\|\w\|^2$
and it has limited effect on the term $\Omega_{\rm RFD}$.
The reason is that the first term of $\Omega_{\rm RFD}$ contains $\frac{1}{\alpha_0}$ in the logarithmic function
and the second term contains $\alpha^{(t)}=\alpha_0+\frac{1}{2}\sum_{i=1}^{t-1}{(\sigma_m^\ui)^2}$ in the denominator.
For large $t$, $\alpha^{(t)}$ is mainly dependent on $\sum_{i=1}^{t-1}{(\sigma_m^\ui)^2}$, rather than $\alpha_0$.
Hence the regret bound of RFD-SON is much less sensitive to the hyperparameter $\alpha_0$ than FD-SON.
We have $\sum_{t=1}^T(\sigma_m^\utm)^2 \leq \|\A-[\A]_k\|_F^2/(m-k)$ for $k<m$ by using (\ref{bound:sv}) with $\A=\sum_{t=1}^T(\mu_t+\eta_t)\g^\ut(\g^\ut)^\top$.

Consider RFD-SON with $\alpha_0=0$.
Typically, the parameter $\alpha^\ut$ is zero at very few first iterations and increase to be strictly positive later.
Hence the learning algorithm can be divided into two phases based on whether $\alpha^\ut$ is zero.
Suppose that $T'$ satisfies
\begin{align}
    \alpha^\ut
    \begin{cases}
        =0, & t< T' \\
        >0,& t \geq T'.
    \end{cases}\label{cond:phase}
\end{align}
Then the regret can be written as
\begin{align*}
    R_{T}(\w) = R_{1:T'}(\w) + R_{(T'+1):T}(\w),
\end{align*}
where
\begin{align*}
    R_{1:T'}(\w)=\sum_{t=1}^{T'} f_t(\w^\ut) - \sum_{t=1}^{T'} f_t(\w)
\end{align*}
and
\begin{align*}
R_{(T'+1):T}(\w)=\sum_{t=T'+1}^{T} f_t(\w^\ut) - \sum_{t=T'+1}^{T'} f_t(\w).
\end{align*}
The regret from the first $T'$ iterations can be bounded by Theorem \ref{thm:regret1}
and the bound of $R_{(T'+1):T}(\w)$ can be derived by the similar proof of Theorem \ref{thm:regret}.
\begin{thm}\label{thm:regret1}
    By the condition of Theorem \ref{thm:regret}, $\mu'=\max_{t=1}^{T'}\{\mu_t\}$ and letting $\alpha_0=0$,
    $\sigma^*$ be the smallest nonzero singular values of $\sum_{t=1}^{T'}\g^\ut(\g^\ut)^\top$
    and $T'$ satisfy (\ref{cond:phase}),
    we have that  the first $T'$ iterations of Algorithm \ref{alg:RFD-ONS} has the following regret
    \begin{align}
            R_{1:T'}(\w)
        \leq 2(CL)^2 \sum_{t=1}^{T'} \eta_t + \frac{m-1}{2(\eta_1+\mu')}
              + \frac{m(m-1)}{2(\eta_1+\mu')}\ln\Big(1+\frac{2\sum_{t=1}^{T'}\|\g^\ut\|_2^2}{(1+r)r\sigma^*}\Big).
    \end{align}
\end{thm}

Combining above results, we can conclude the regret bound for the hyperparameter-free algorithm in Theorem \ref{thm:regret2}.
In practice, we often set $m$ to be much smaller than $d$ and $T$, which implies $T'$ is much smaller than $T$.
Hence, the regret bound of Theorem \ref{thm:regret2} is similar to the one of Theorem \ref{thm:regret} when $\alpha_0$ is close to 0.
We can use RFD-SON with $\alpha_0=0$ and $\eta_t=1/t$ to obtain a hyperparameter-free sketched online Newton algorithm.
\cite{luo2016efficient} have proposed a hyperparameter-free online Newton algorithm without sketching
and their regret contains a term with coefficient $d$.
\begin{thm}\label{thm:regret2}
    Consider Algorithm \ref{alg:RFD-ONS} with $\alpha_0=0$,
    let $T'$ satisfy (\ref{cond:phase})$, \mu=\min_{t=1}^T\{\mu_t\}$, $\mu'=\max_{t=1}^{T'}\{\mu_t\}$,
    $\fK={\mathop{\bigcap}}_{t=1}^T\fK_t$, $\sigma^*$ has the same definition of
    Theorem \ref{thm:regret1} and $\alpha'_0=\det(\bH^{(T')})$.
    Then under Assumptions \ref{asm:bound} and \ref{asm:curv} for any $\w\in\fK$, we have that
    \begin{align}
            R_{T}(\w)
        \leq 2(CL)^2 \sum_{t=1}^{T} \eta_t + \frac{m-1}{2(\eta_1+\mu')}
              + \frac{m(m-1)}{2(\eta_1+\mu')}\ln\Big(1+\frac{2\sum_{t=1}^{T'}\|\g^\ut\|_2^2}{(1+r)r\sigma^*}\Big)  \nonumber\\
             + \frac{1}{2}\|\w^{(T')}\|_{\bH^{(T')}}^2
             + \frac{m}{2(\mu+\eta_T)}\ln\Big(\tr\frac{\big((\B^{(T)})^\top\B^{(T)}\big)}{m\alpha^{(T')}}
             +\frac{\alpha^{(T)}}{\alpha'_0}\Big) + \Omega'_{\rm RFD} \label{bound:RFDSON0}
    \end{align}
    where
    \begin{align*}
        \Omega'_{\rm RFD}=  \frac{d-m}{2(\mu+\eta_T)}\ln\frac{\alpha^{(T)}}{{\alpha'_0}} +
            \frac{m}{4(\mu+\eta_T)}\sum_{t=T'+1}^T\frac{(\sigma_m^\ut)^2}{\alpha^\ut} + C^2 \sum_{t=T'+1}^T(\sigma_m^\utm)^2.
    \end{align*}
\end{thm}

\section{Experiments}
\label{sec:exp}

In this section, we evaluate the performance of robust frequent directions
(RFD) and online Newton step by RFD (RFD-SON) on six real-world data sets ``a9a,''  ``gisette,''  ``sido0,'' ``farm-ads,'' ``rcv1'' and ``real-sim,'' whose details are listed in Table \ref{table:dataset}.
The data sets ``sido0'' and ``farm-ads''can be found on Causality Workbench\footnote{https://www.causality.inf.ethz.ch/data/SIDO.html},
and UCI Machine Learning Repository\footnote{https://archive.ics.uci.edu/ml/datasets/Farm+Ads}.
The others can be downloaded from LIBSVM repository\footnote{https://www.csie.ntu.edu.tw/~cjlin/libsvmtools/datasets/}.
The experiments are conducted in Matlab and run on a server with Intel (R) Core (TM) i7-3770 CPU 3.40GHz$\times$2, 8GB RAM and 64-bit Windows Server 2012  system.
\begin{table}[ht]
    \centering
	\begin{tabular}{|c|c|c|c|c|}
        \hline
        data sets & $n$ & $d$ & source  \\ \hline
        a9a       & 32,561   &  ~~~~123 & \citep{platt199912} \\ \hline
        gisette   & ~~6,000  &  ~~5,000 & \citep{guyon2004result} \\ \hline
        sido0     & 12,678   &  ~~4,932 & \citep{guyon2008sido} \\ \hline
        farm-ads  & ~~4,143  & 54,877   & \citep{mesterharm2011active} \\ \hline
        rcv1      & 20,242   & 47,236 & \citep{lewis2004rcv1} \\ \hline
        real-sim  & 72,309   &  20,958 & \citep{real-sim} \\ \hline
	\end{tabular}
    \caption{Summary of data sets used in our experiments}\label{table:dataset}\vspace{0.05cm}
\end{table}

\subsection{Matrix Approximation}
We evaluate the approximation errors of the deterministic sketching algorithms including
frequent directions (FD) \citep{liberty2013simple,DBLP:journals/siamcomp/GhashamiLPW16}, parameterized frequent directions (PFD),
compensative frequent directions (CFD) \citep{DBLP:journals/tkde/DesaiGP16} and robust frequent directions (RFD).
For a given data set $\A\in\BR^{n\times d}$ of $n$ samples with $d$ features,
we use the accelerated algorithms (see details in Appendix \ref{appendix:alg}) to
approximate the covariance matrix $\A^\top\A$ by $\B^\top\B$ for FD, PFD, CFD and by $\B^\top\B+\alpha\I_d$ for RFD, respectively.
We measure the performance according to the relative spectral norm error.
We report the relative spectral norm error by varying the sketch size $m$.

Figure \ref{figure:RFDerror1} shows the performance of FD, CFD and RFD.
These three methods have no extra hyperparameter and their outputs only rely on the sketch size.
The relative error of RFD is always smaller than that of FD and CFD.
The error of RFD is nearly half of the error of FD in most cases,
which matches our theoretical analysis in Theorem \ref{thm:RFDbound} very well.

\begin{figure*}[ht]
    \centering\vskip-0.4cm
	\begin{tabular}{ccc}
        \includegraphics[scale=0.33]{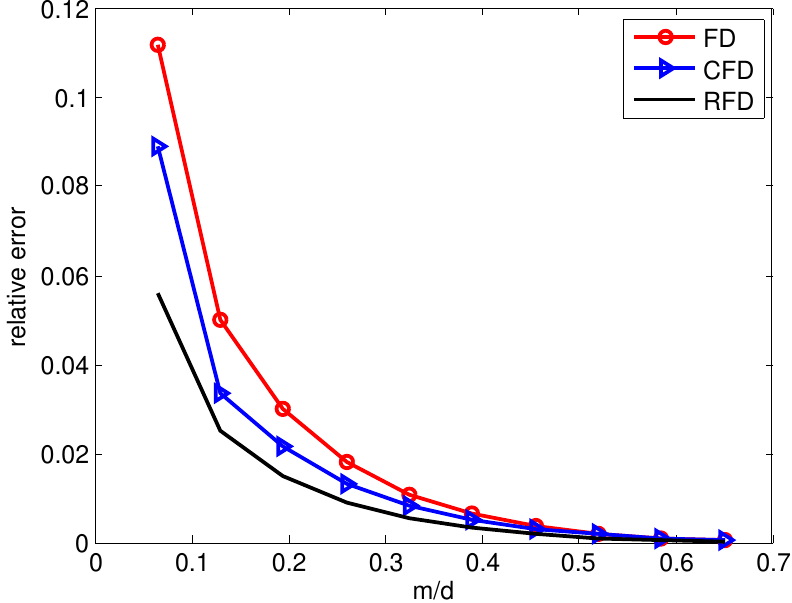}  &
        \includegraphics[scale=0.33]{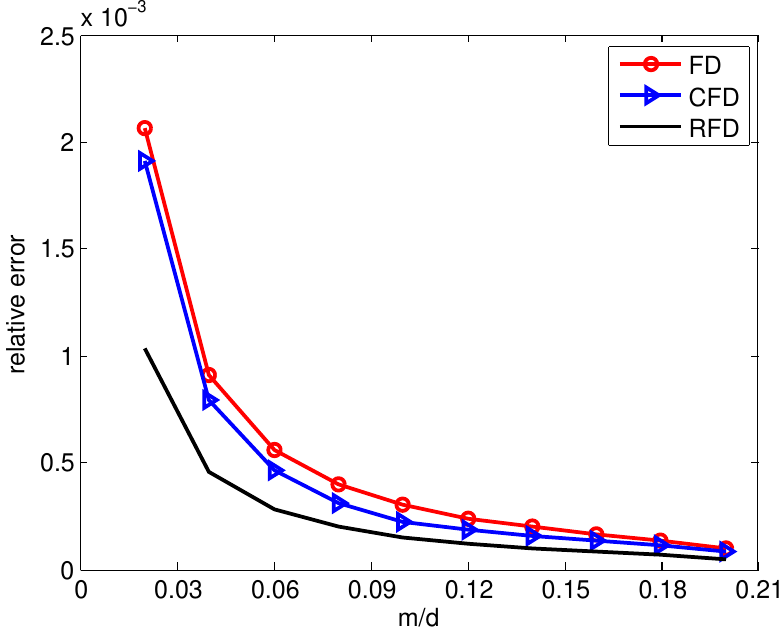}  &
        \includegraphics[scale=0.33]{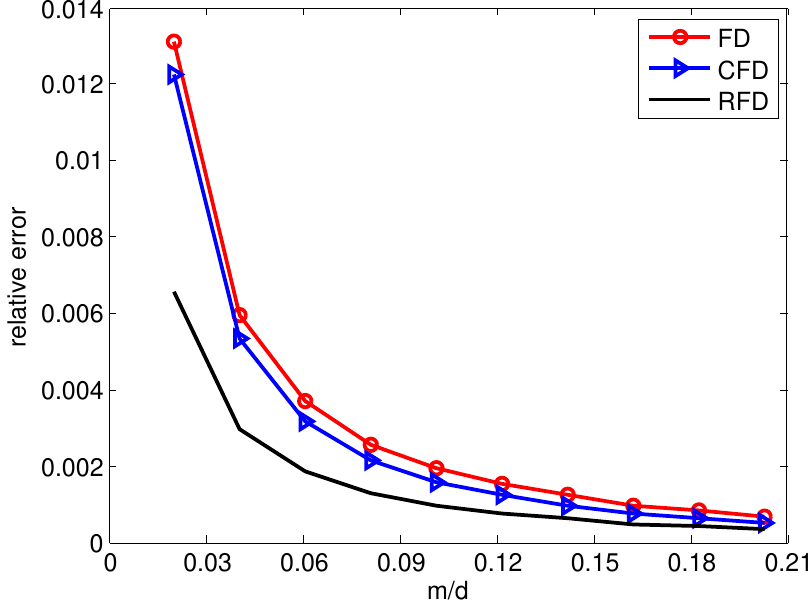}  \\
        (a) a9a & (b) gisette & (c) sido0 \\[0.1cm]
        \includegraphics[scale=0.33]{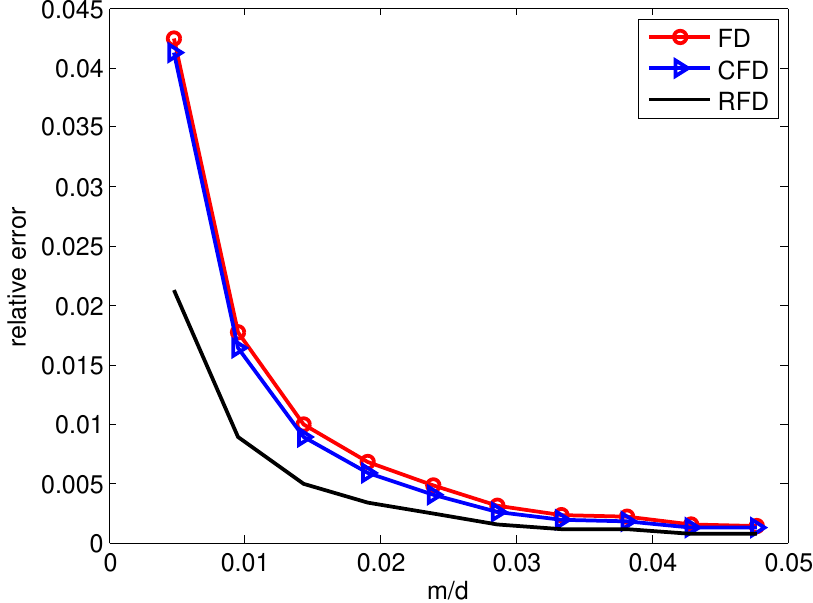}  &
        \includegraphics[scale=0.33]{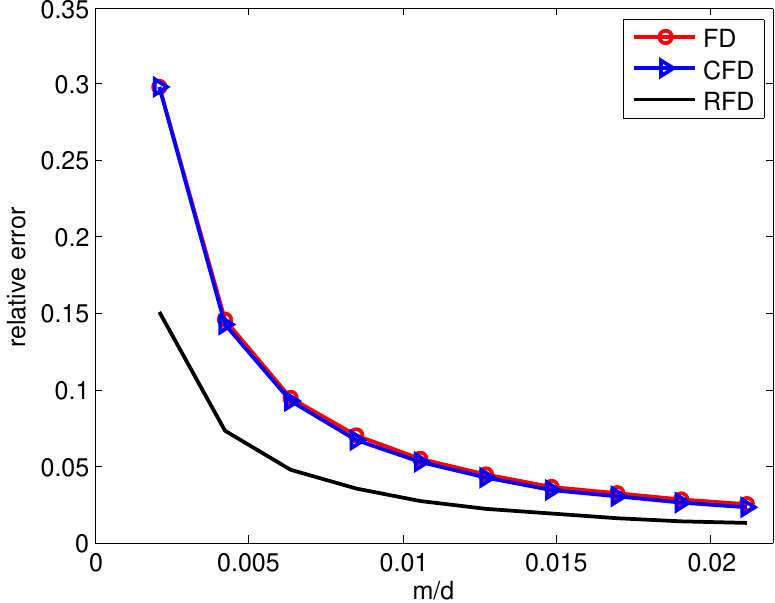}  &
        \includegraphics[scale=0.33]{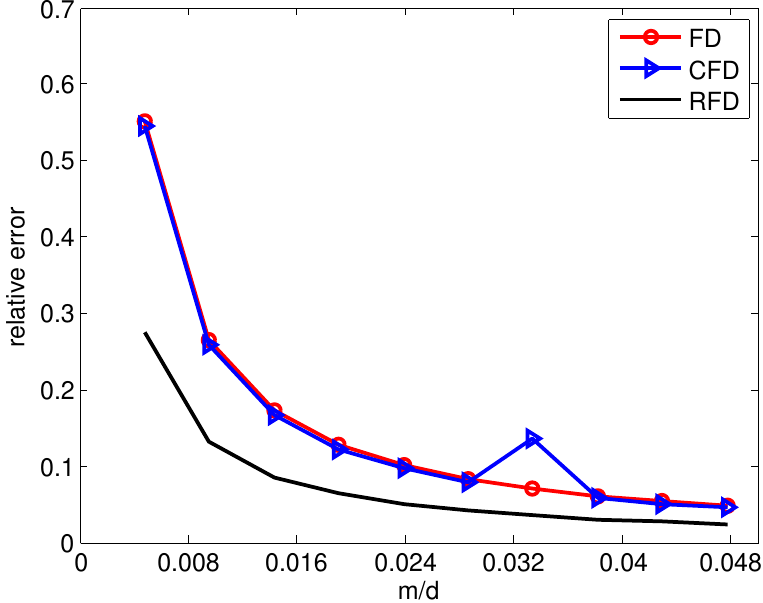}  \\
        (d) farm-ads & (e) rcv1 & (f) real-sim
	\end{tabular}
    \caption{Comparison of relative spectral error of FD, CFD and RFD with proportion of sketching}\vskip-0.4cm
    \label{figure:RFDerror1}\vskip 0.1cm
\end{figure*}

Figure \ref{figure:RFDerror2} compares the performance of RFD and PFD with different choices of the hyperparameter.
We use PFD-$\beta$ to refer the PFD algorithm where $\lfloor \beta m\rfloor$ singular values will get affected by the shrinkage steps.
The extra hyperparameter $\beta$ is tuned from $\{0.2, 0.4, 0.6, 0.8\}$.
The result shows that RFD is better than PFD in most cases.
PFD sometimes can achieve lower approximation error with a good choice of $\beta$.
However, selecting the hyperparameter requires additional computation.

\begin{figure*}[ht]
    \centering
	\begin{tabular}{ccc}
        \includegraphics[scale=0.33]{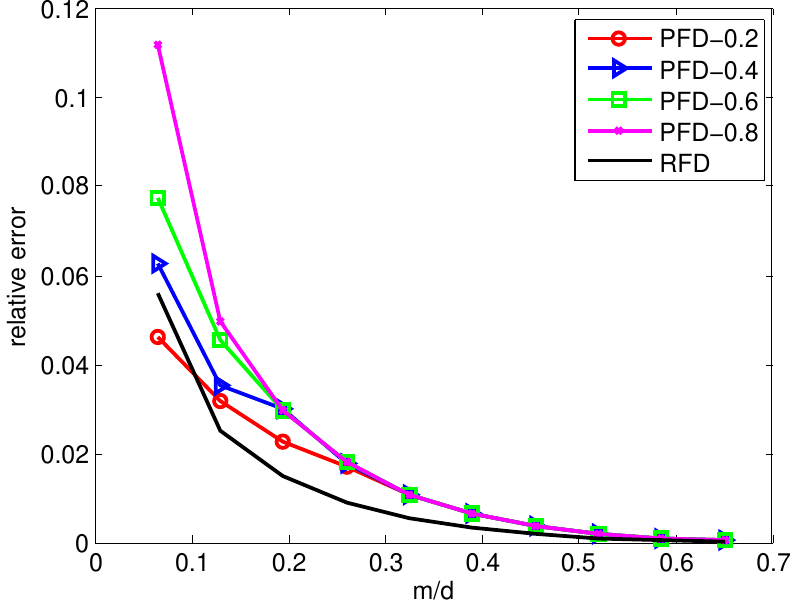}  &
        \includegraphics[scale=0.33]{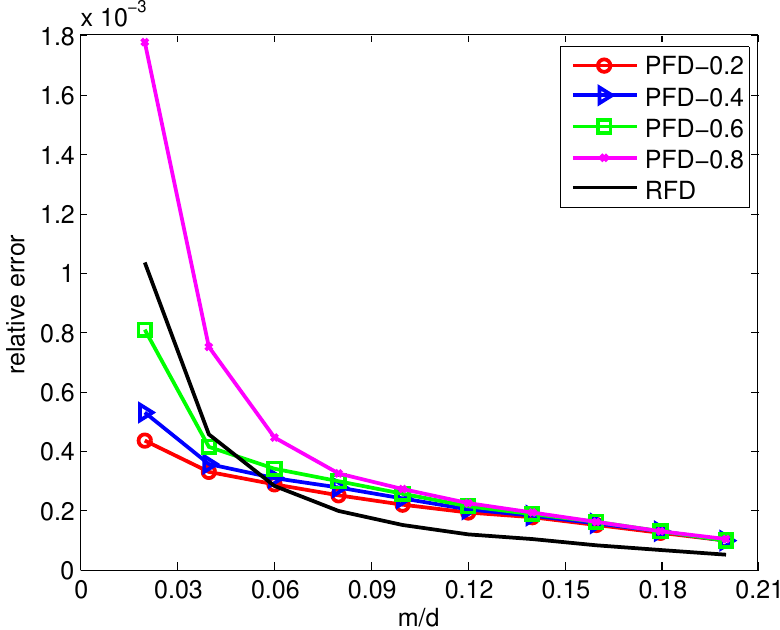}  &
        \includegraphics[scale=0.33]{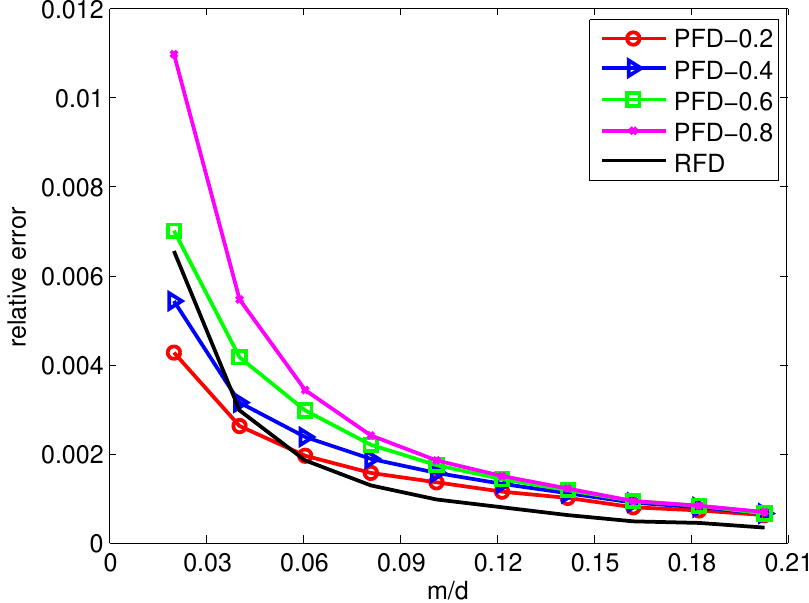}  \\
        (a) a9a & (b) gisette & (c) sido0 \\[0.1cm]
        \includegraphics[scale=0.33]{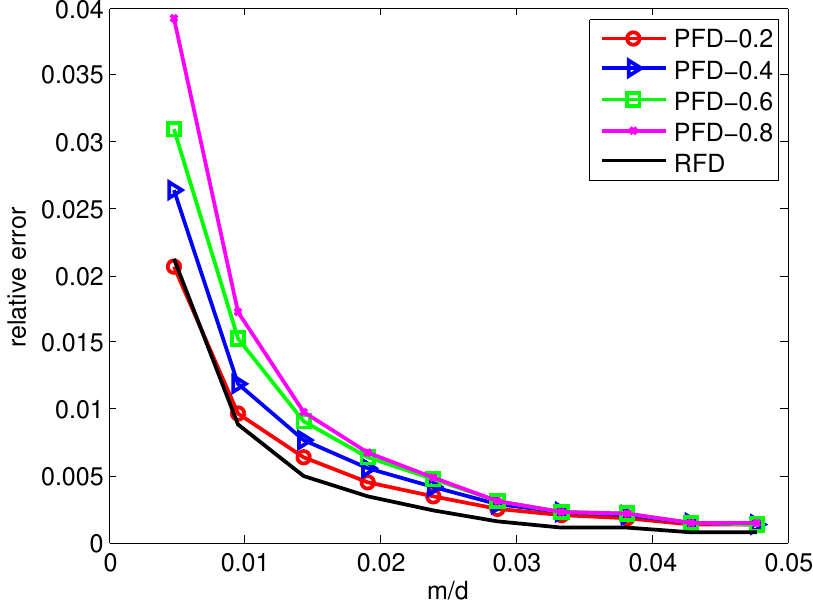}  &
        \includegraphics[scale=0.33]{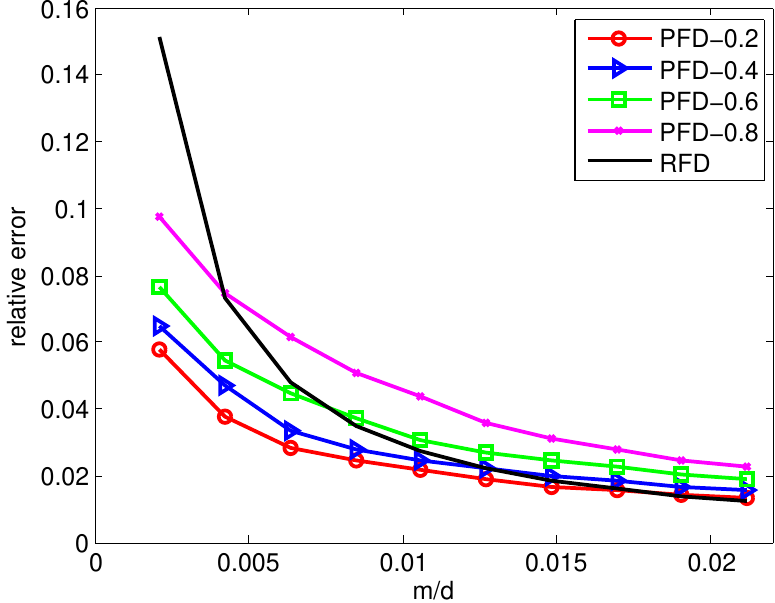}  &
        \includegraphics[scale=0.33]{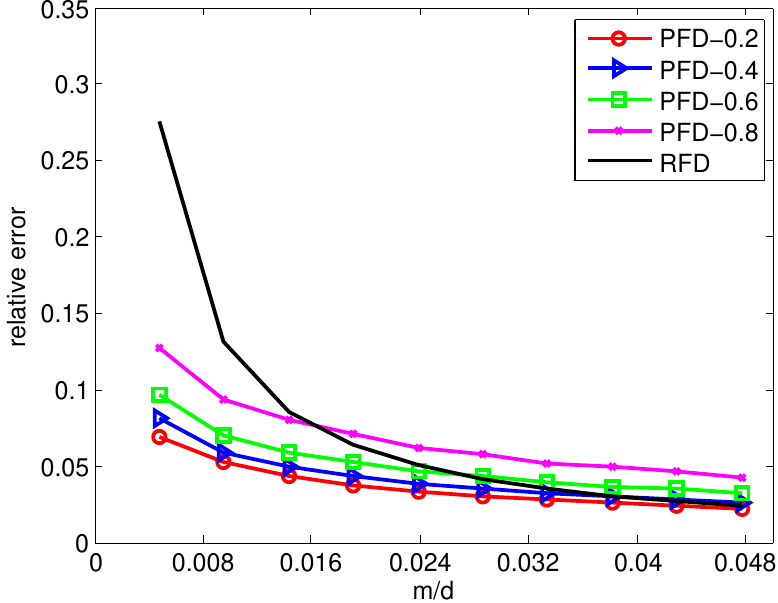}  \\
        (d) farm-ads & (e) rcv1 & (f) real-sim
	\end{tabular}
    \caption{Comparison of relative spectral error of PFD and RFD with proportion of sketching}\vskip-0.6cm
    \label{figure:RFDerror2}
\end{figure*}

\subsection{Online Learning}
We now evaluate the performance of RFD-SON.
We use the least squares loss $f_t(\w)=(\w^\top\x^\ut-y^\ut)^2$, and set $\fK_t=\{\w:|\w^\top\x^\ut|\leq 1\}$.
In the experiments, we use the doubling space strategy (Algorithm \ref{alg:FRFD} in Appendix \ref{appendix:alg}).
We use 70\% of the data set for training and the rest for test.
The algorithms in the experiments include ADAGRAD, the standard online Newton step with the full Hessian \citep{duchi2011adaptive} (FULL-ON),
the sketched online Newton step with frequent directions (FD-SON), the parameterized frequent directions (PFD-SON), the random projections (RP-SON),
Oja's algorithms (Oja-SON) \citep{luo2016efficient,DBLP:journals/tkde/DesaiGP16}, and our proposed sketched online Newton step with RFD (RFD-SON).

The hyperparameter $\alpha_0$ is tuned from $\{10^{-3}, 10^{-2} \dots$ $10^{5}$, $10^{6}\}$ for all methods and let $\eta=1/t$ for SON algorithms.
FULL-ON is too expensive and impractical for large $d$, so we exclude it from experiments on ``farm-ads,''  ``rcv1'' and ``real-sim.''
For PFD-SON, we let $\beta=0.2$ heuristically because it usually achieves good performance on approximating the covariance matrix.
Additionally, RFD-SON includes the result with $\alpha_0=0$ (RFD$_0$-SON).
The sketch size of sketched online Newton methods is chosen from $\{5, 10,  20\}$ for ``a9a,'' ``gisette,'' ``sido0,''
and  $\{20, 30,  50\}$ for ``farm-ads,'' ``rcv1'' and ``real-sim.''
We measure performance according to two metrics \citep{duchi2011adaptive}: the online error rate and the test set performance of the predictor at the end of one pass through the training data.

We are interested in how the hyperparameter $\alpha_0$ affects the performance of the algorithms.
We display the test set performance in Figures \ref{figure:testA} and \ref{figure:testB}.
We compare the online error rate of RFD$_0$-SON with the one of FULL-ON in Figure \ref{figure:train_full}
and show the comparison between RFD$_0$-SON and other SON methods with different choices of $\alpha_0$
in Figures \ref{figure:train_a9a} - \ref{figure:train_real-sim}.

We also report the accuracy on the test sets for all algorithms at one pass with the best $\alpha_0$ in Table \ref{table:accuracy}
and the corresponding running times in Table \ref{table:time}.
All SON algorithms can perform well with the best choice of $\alpha_0$.
However, only RFD$_0$-SON can perform well without tuning the hyperparameter
while all baseline methods ADAGRAD, FD-SON, PFD-SON, RP-SON and Oja-SON are very sensitive to the value of $\alpha_0$.
The sub-figure (j)-(l) in Figures \ref{figure:train_full}-\ref{figure:train_real-sim}
shows RFD-SON usually has good performance with small $\alpha_0$,
which validates our theoretical analysis in Theorem \ref{thm:regret}.
The choice of  the hyperparameter almost has no effect of RFD-SON on data set ``a9a',' ``gisette,'' ``sido0'' and ``farm-ads.''
These results verify that RFD-SON is a very stable algorithm in practice.

\begin{figure}[ht]
\centering
    \begin{tabular}{ccc}
        \includegraphics[scale=0.35]{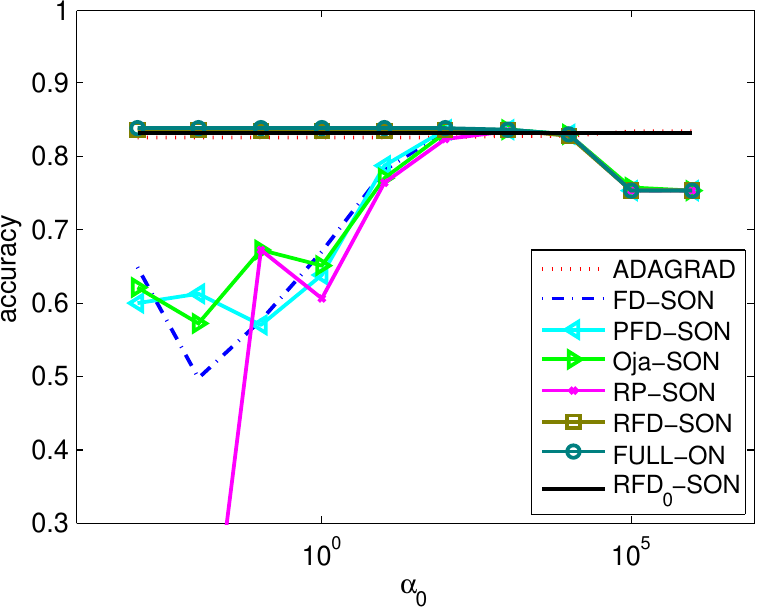} &
        \includegraphics[scale=0.35]{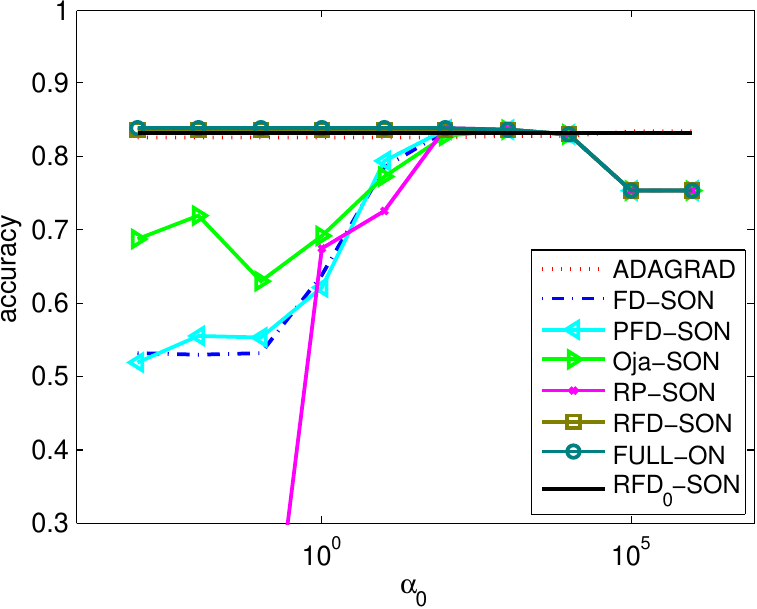} &
        \includegraphics[scale=0.35]{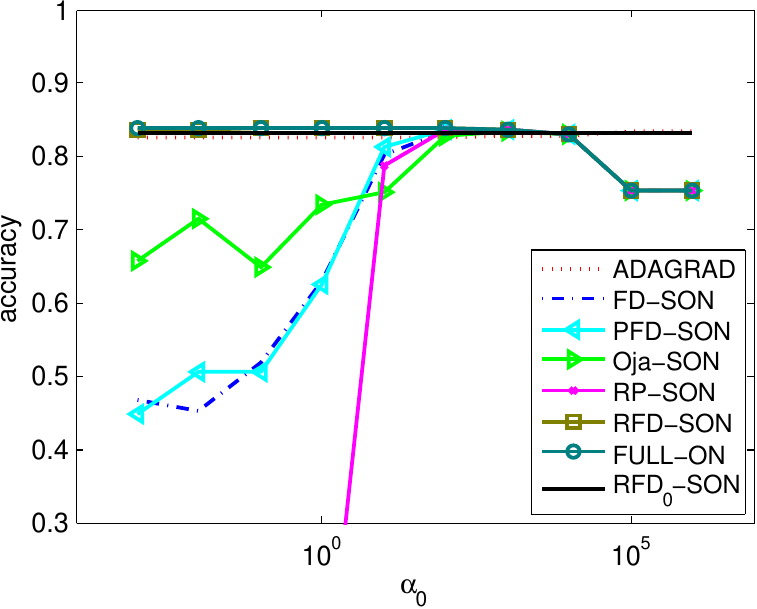} \\
         (a) a9a, $m=5$ &  (b) a9a, $m=10$ &  (c) a9a, $m=20$ \\[0.2cm]
       \includegraphics[scale=0.35]{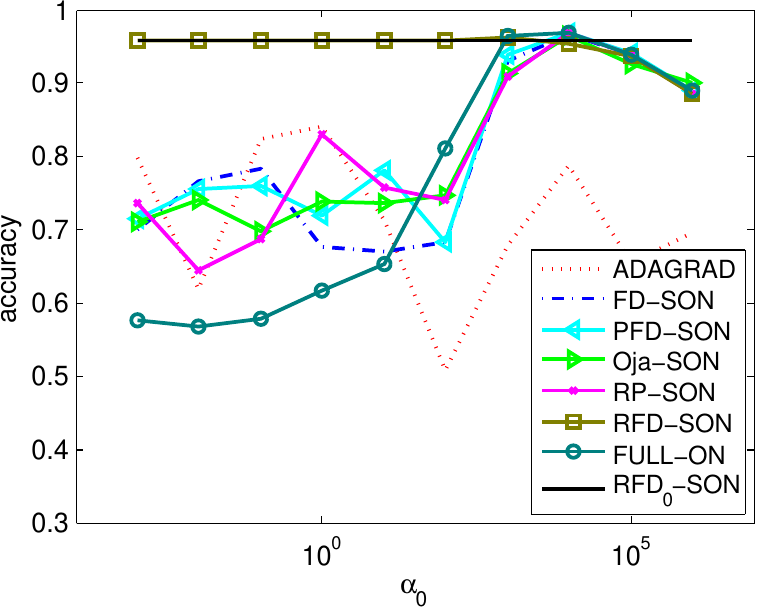} &
        \includegraphics[scale=0.35]{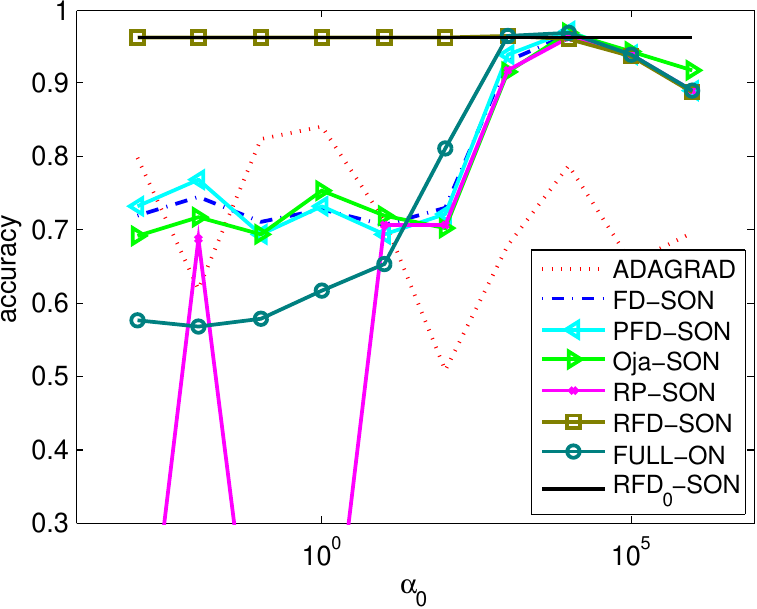} &
        \includegraphics[scale=0.35]{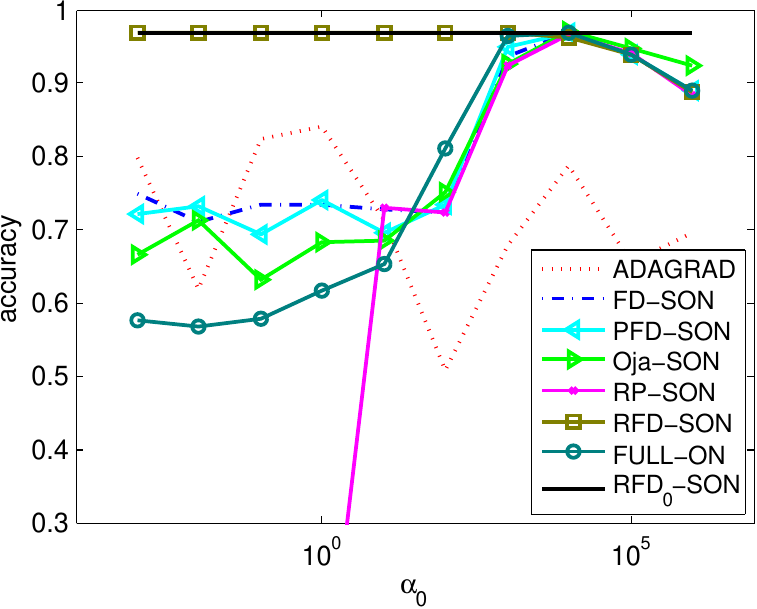} \\
         (d) gisette, $m=5$ &  (e) gisette, $m=10$ &  (f) gisette, $m=20$ \\[0.2cm]
       \includegraphics[scale=0.35]{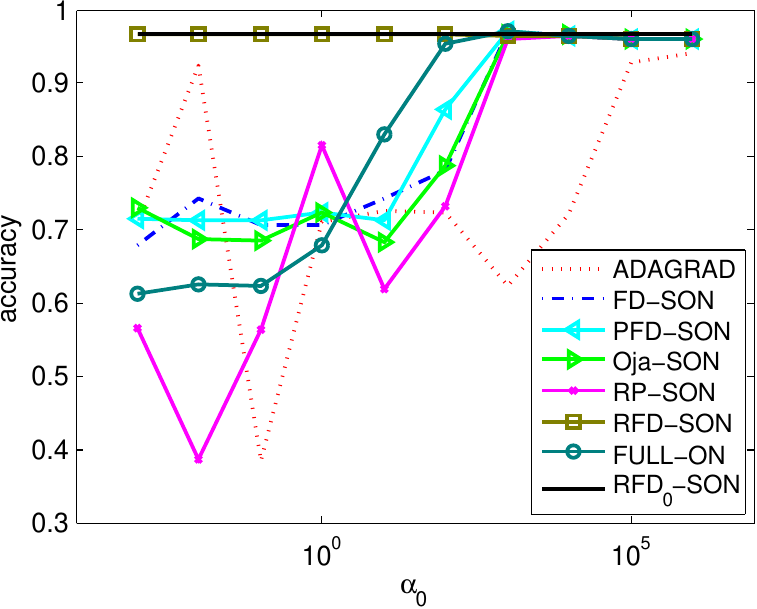} &
        \includegraphics[scale=0.35]{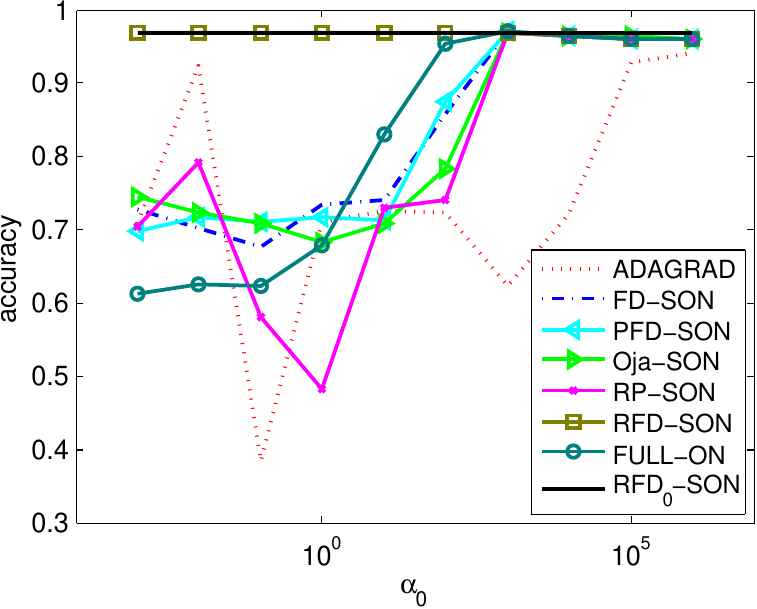} &
        \includegraphics[scale=0.35]{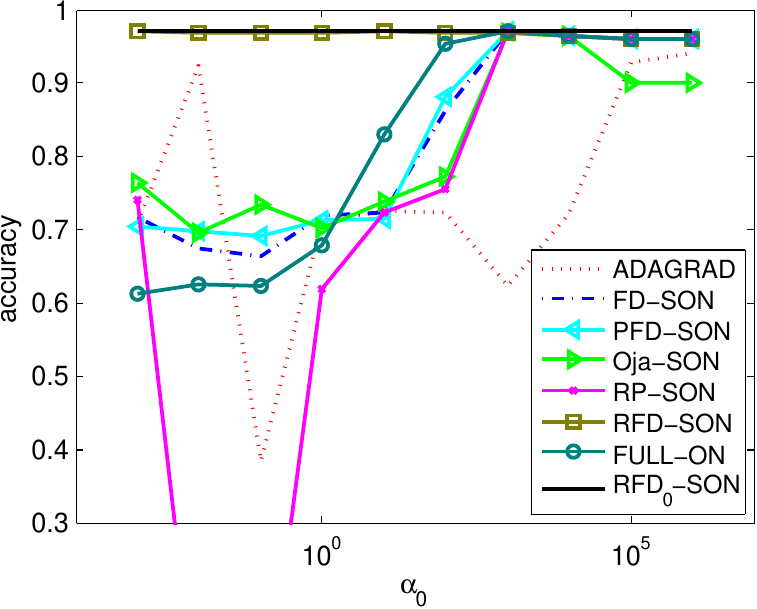} \\
         (g) sido0, $m=5$ &  (h) sido0, $m=10$ &  (i) sido0, $m=20$ \\[0.2cm]
    \end{tabular}
    \caption{\small Comparison of the test error at the end of one pass on ``a9a'', ``gisette'', ``sido0''}
    \label{figure:testA}
\end{figure}

\begin{figure}[ht]
\centering
    \begin{tabular}{ccc}
        \includegraphics[scale=0.35]{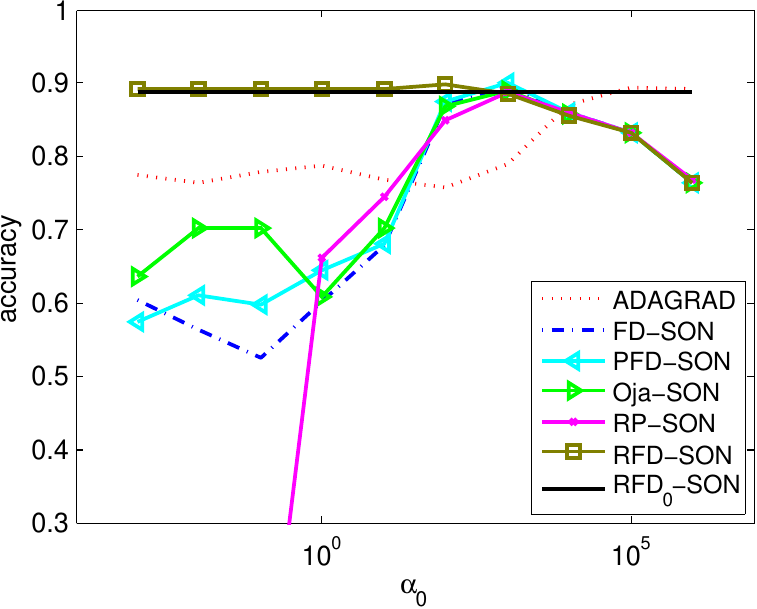} &
        \includegraphics[scale=0.35]{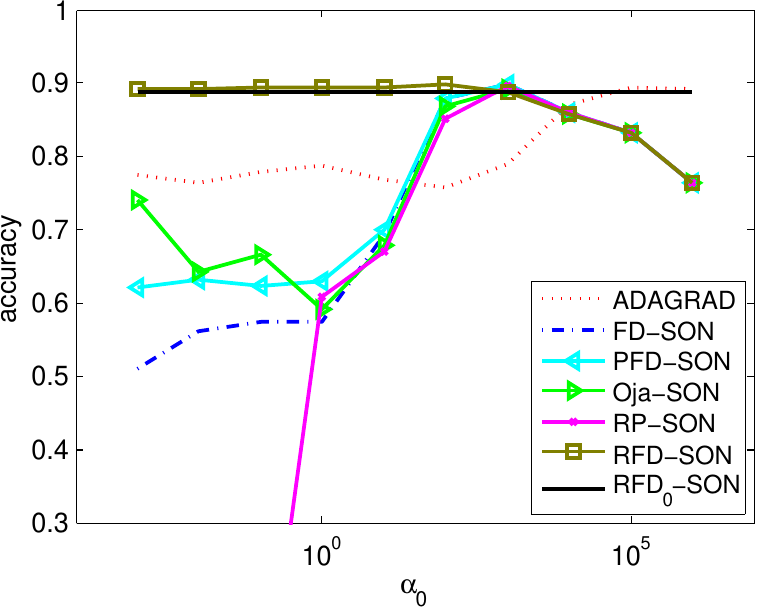} &
        \includegraphics[scale=0.35]{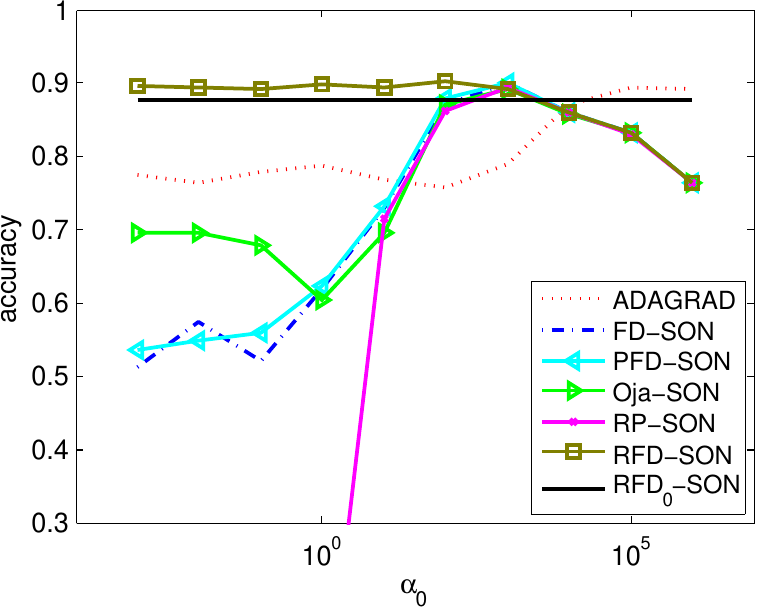} \\
        (a) farm-ads, $m=20$ & (b) farm-ads, $m=30$ & (c) farm-ads, $m=50$ \\[0.4cm]
       \includegraphics[scale=0.35]{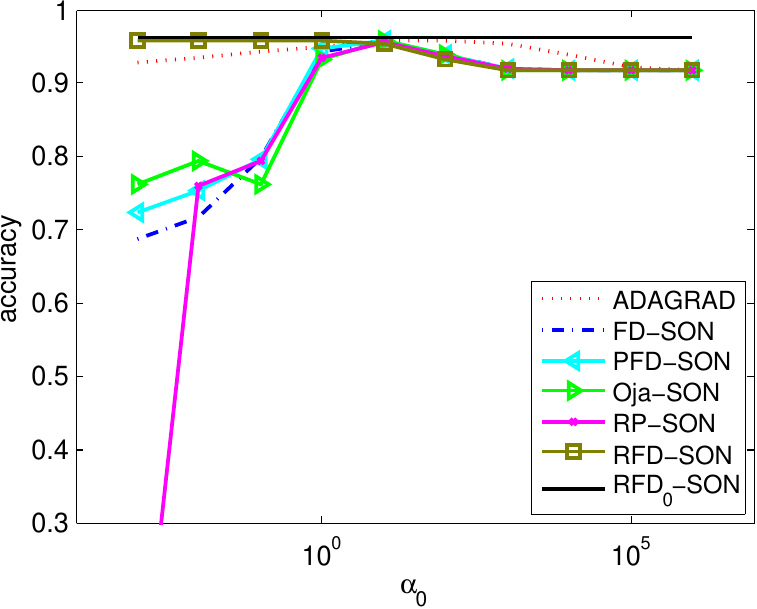} &
        \includegraphics[scale=0.35]{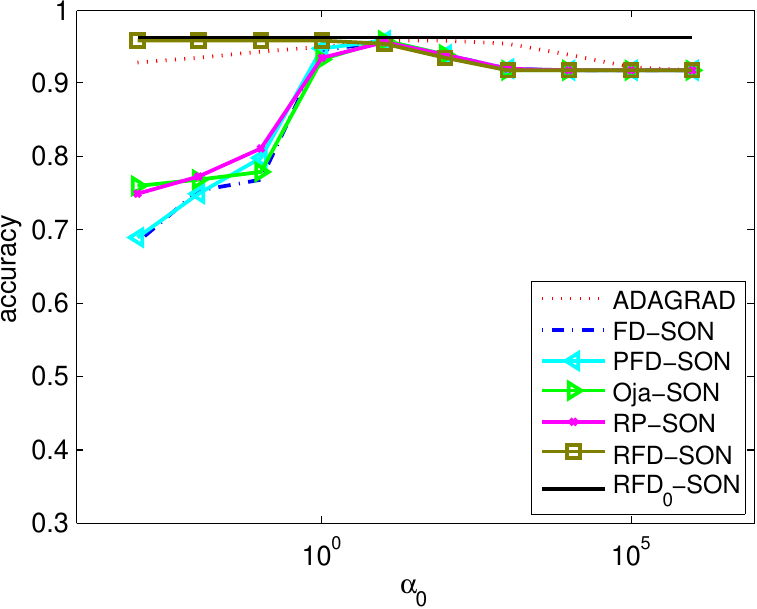} &
        \includegraphics[scale=0.35]{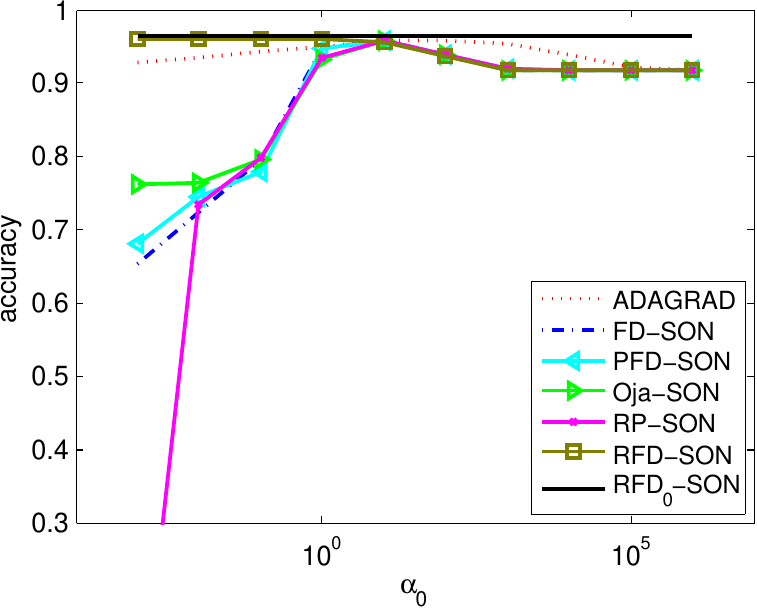} \\
        (d) rcv1, $m=20$ & (e) rcv1, $m=30$ & (f) rcv1, $m=50$ \\[0.4cm]
       \includegraphics[scale=0.35]{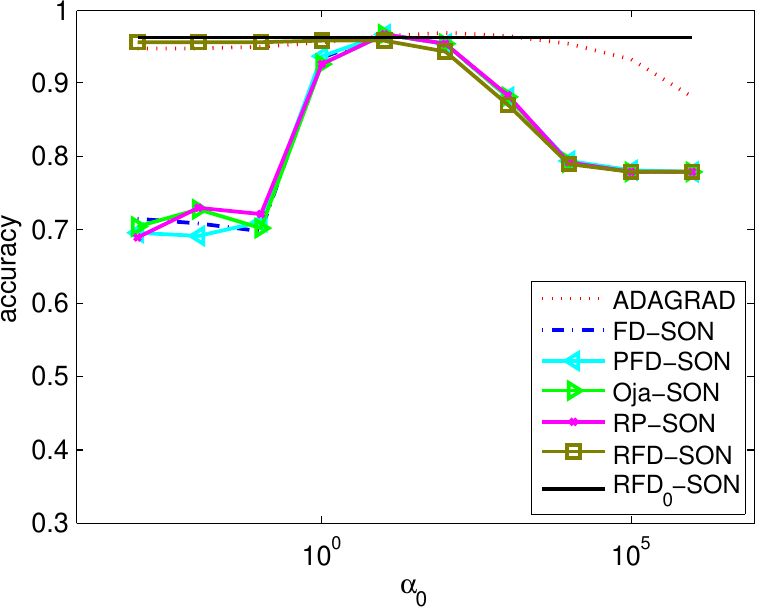} &
        \includegraphics[scale=0.35]{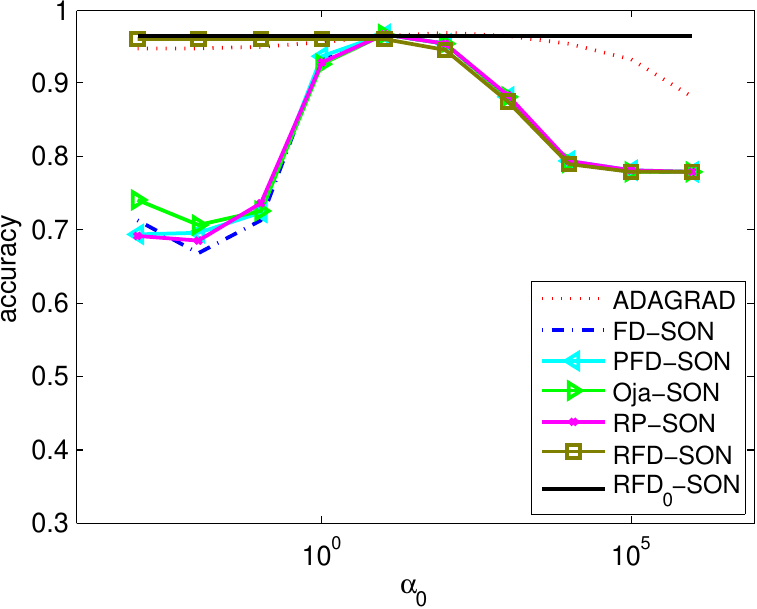} &
        \includegraphics[scale=0.35]{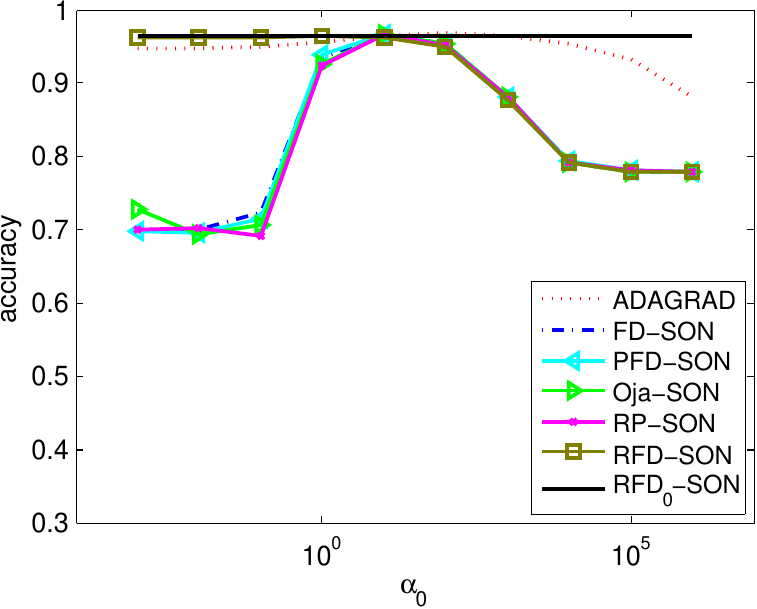} \\
        (g) real-sim, $m=20$ & (h) real-sim, $m=30$ & (i) real-sim, $m=50$ \\[0.4cm]
    \end{tabular}
    \caption{Comparison of the test error at the end of one pass on ``farm-ads'', ``rcv1'', ``real-sim''}
    \label{figure:testB}
\end{figure}

\begin{figure}[H]
\centering
    \begin{tabular}{ccc}
        \includegraphics[scale=0.34]{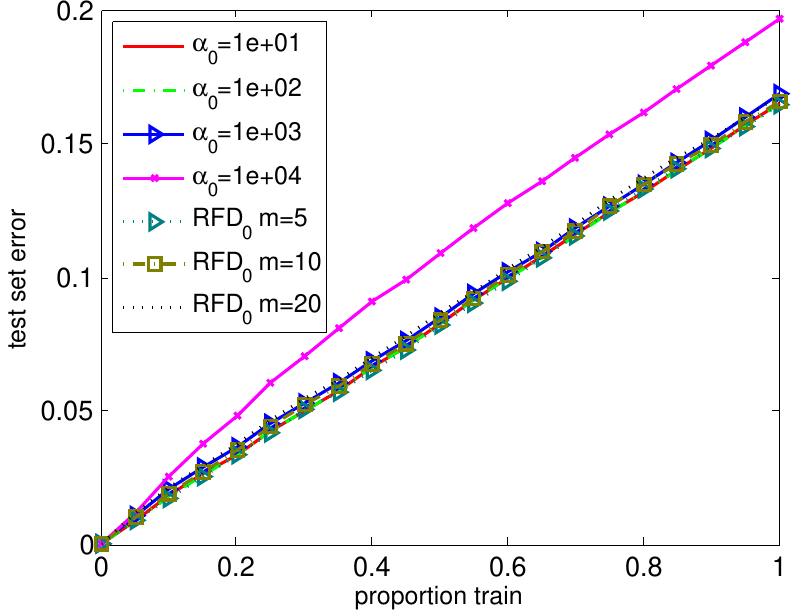} &
        \includegraphics[scale=0.34]{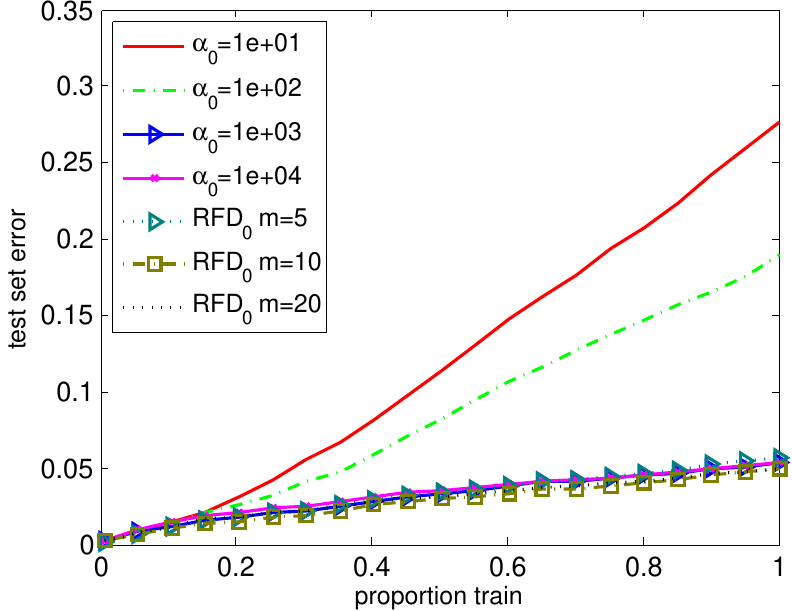} &
        \includegraphics[scale=0.34]{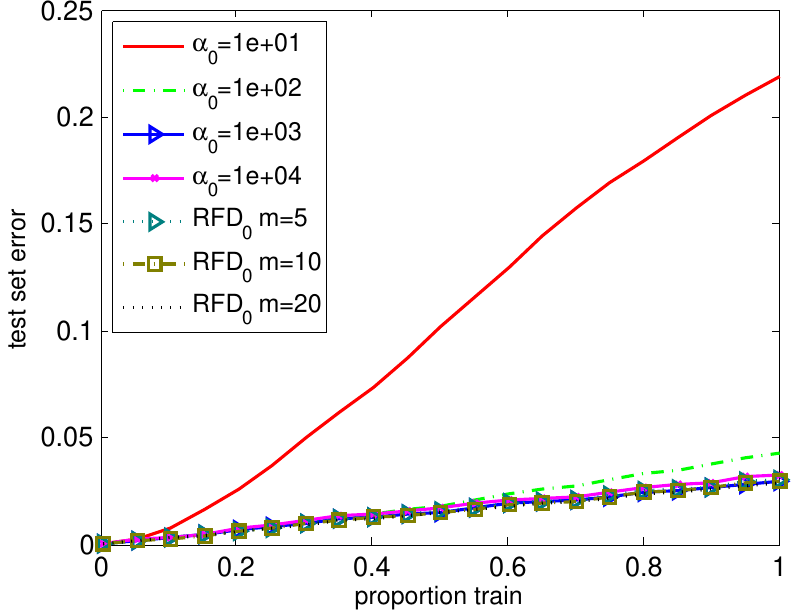} \\
         (a) a9a  &  (a) gisette  &  (c) sido0 \\[0.4cm]
    \end{tabular}
    \caption{Comparison of the online error rate between algorithm FULL-ON and RFD$_0$}
    \label{figure:train_full}
\end{figure}

\begin{figure}[H]
\centering
    \begin{tabular}{ccc}
        \includegraphics[scale=0.32]{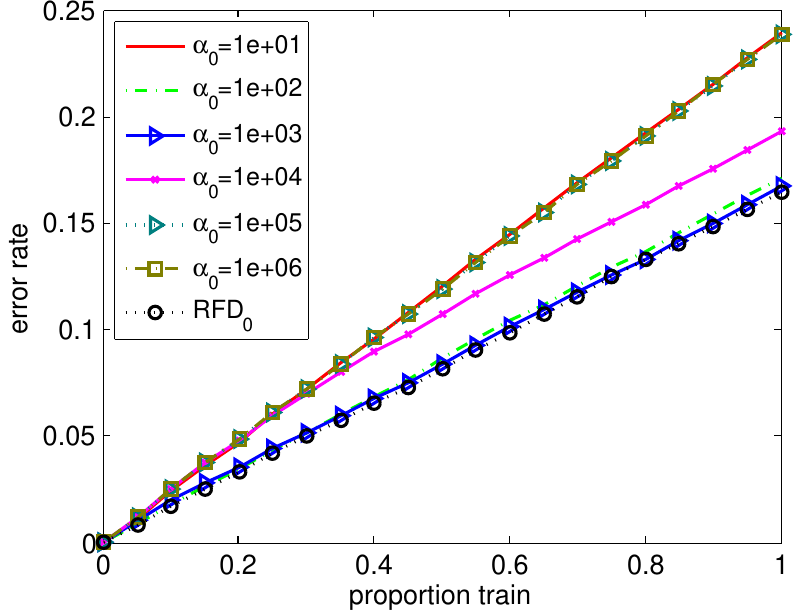} &
        \includegraphics[scale=0.32]{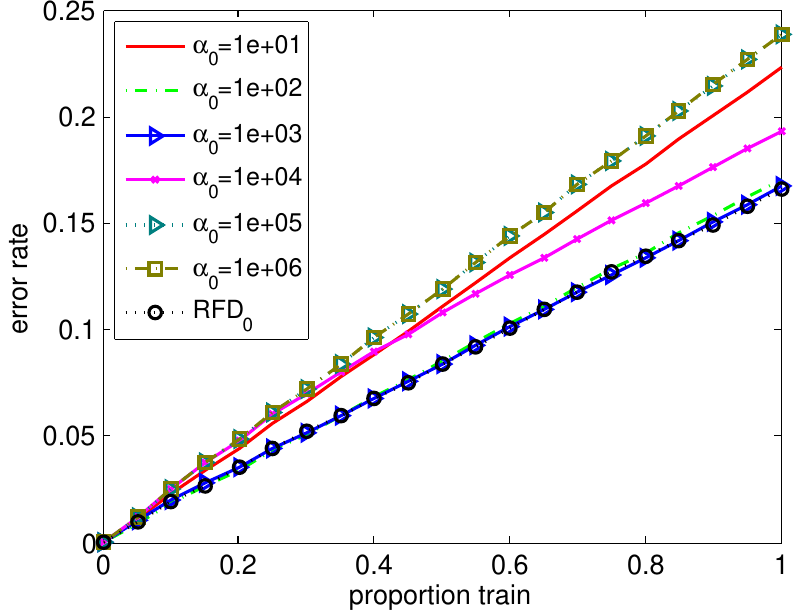} &
        \includegraphics[scale=0.32]{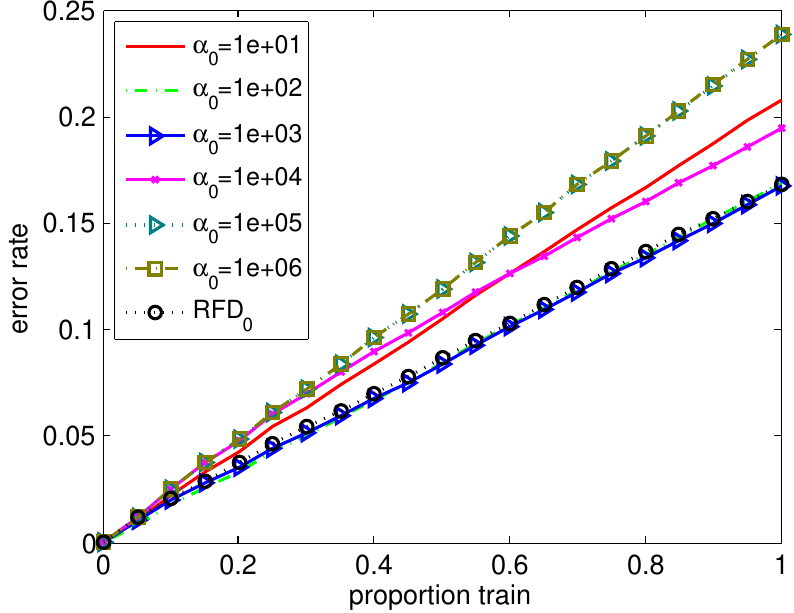} \\
        {\small (a) FD vs RFD$_0$, $m=5$}  & {\small (b) FD vs RFD$_0$}, $m=10$ & {\small (c) FD vs RFD$_0$, $m=20$} \\[0.1cm]
        \includegraphics[scale=0.32]{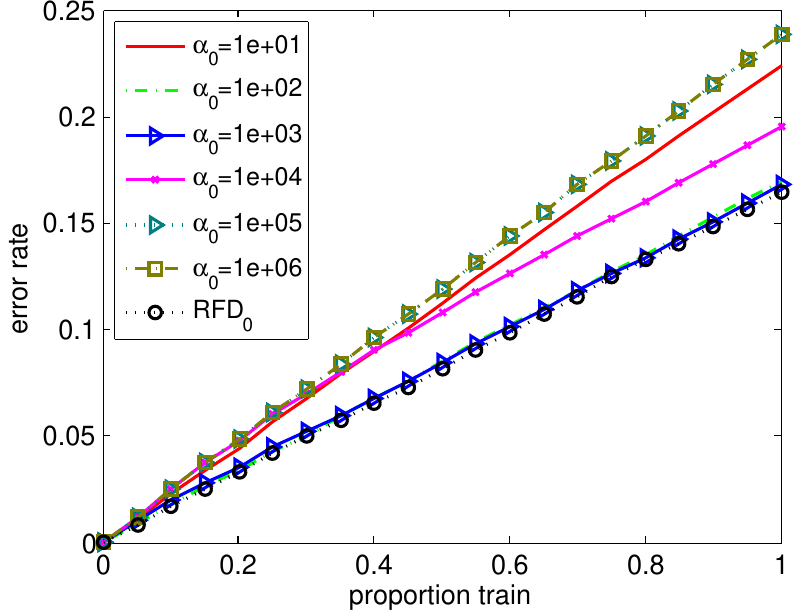} &
        \includegraphics[scale=0.32]{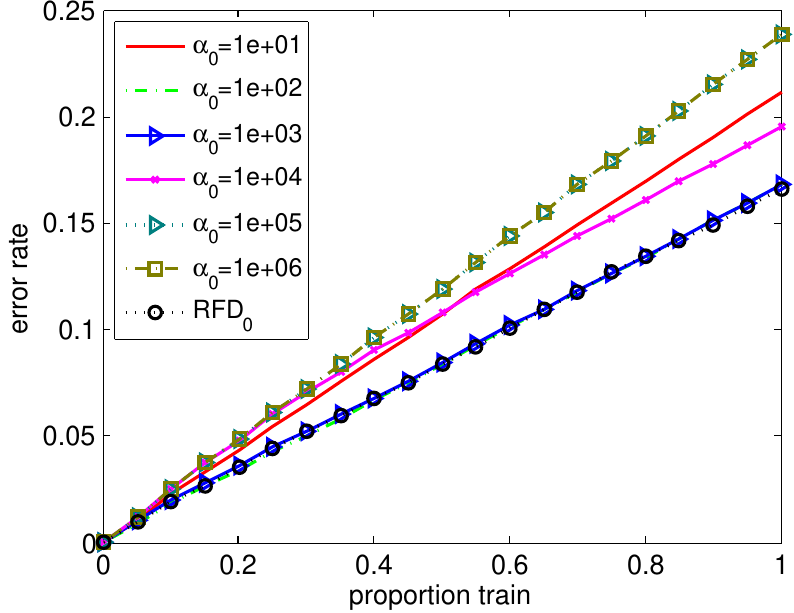} &
        \includegraphics[scale=0.32]{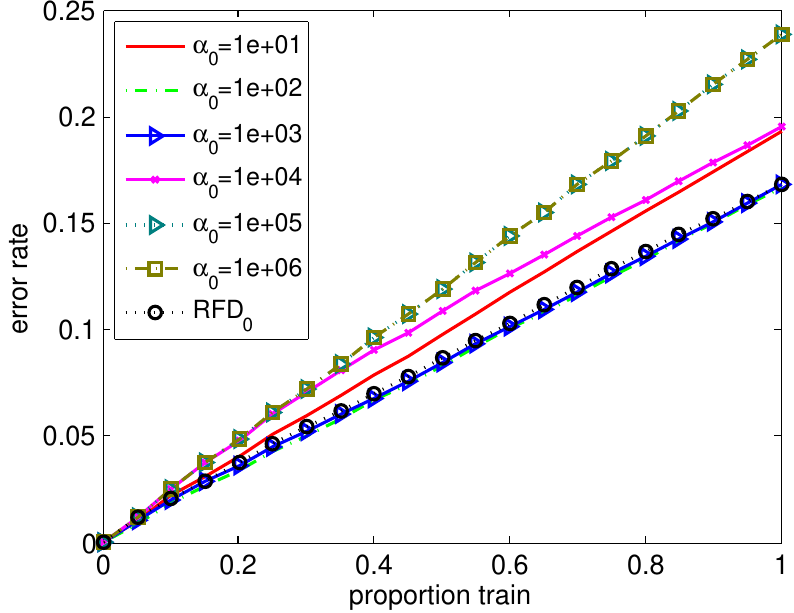} \\
        {\small (a) PFD vs RFD$_0$, $m=5$}  & {\small (b) PFD vs RFD$_0$}, $m=10$ & {\small (c) PFD vs RFD$_0$, $m=20$} \\[0.1cm]
        \includegraphics[scale=0.32]{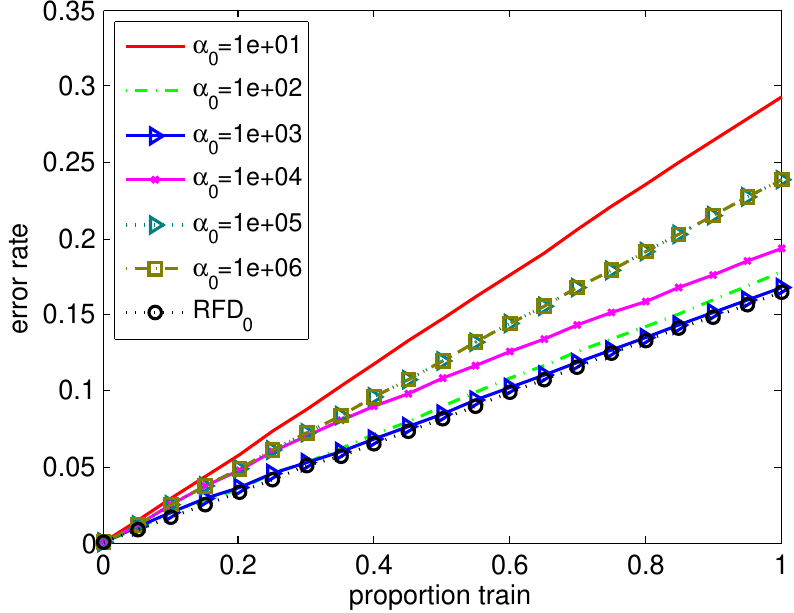} &
        \includegraphics[scale=0.32]{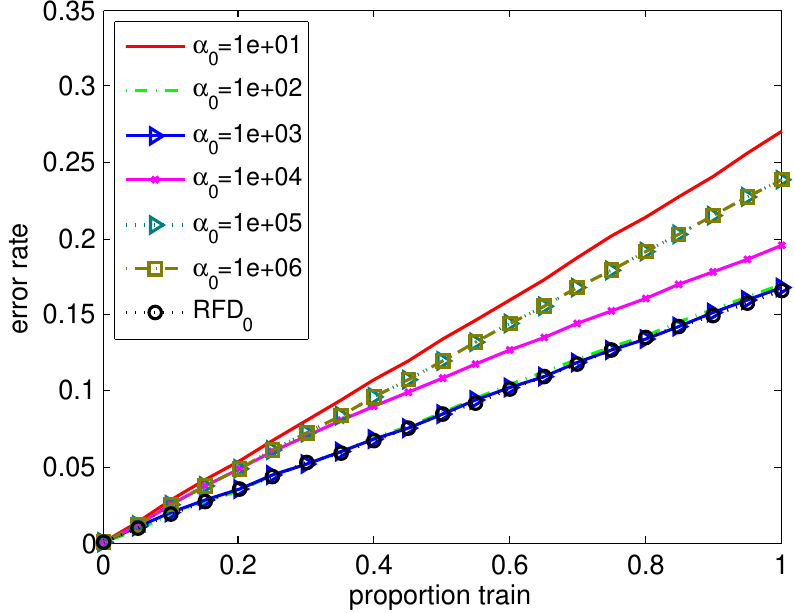} &
        \includegraphics[scale=0.32]{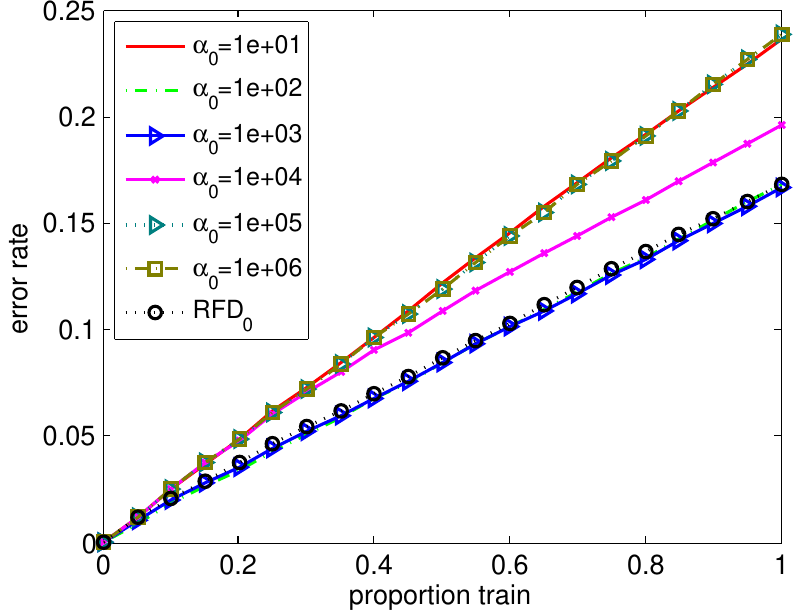} \\
        {\small (d) RP vs RFD$_0$, $m=5$}  & {\small (e) RP vs RFD$_0$}, $m=10$ & {\small (f) RP vs RFD$_0$, $m=20$} \\[0.1cm]
        \includegraphics[scale=0.34]{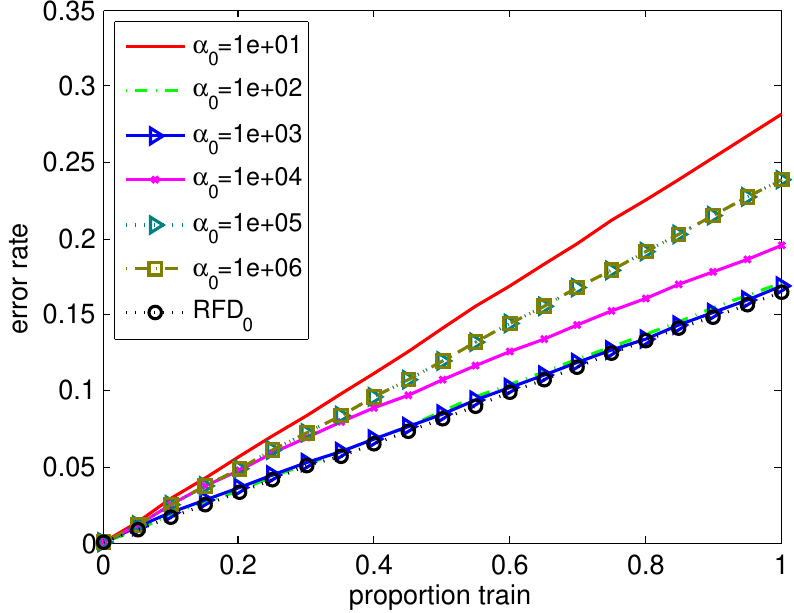} &
        \includegraphics[scale=0.34]{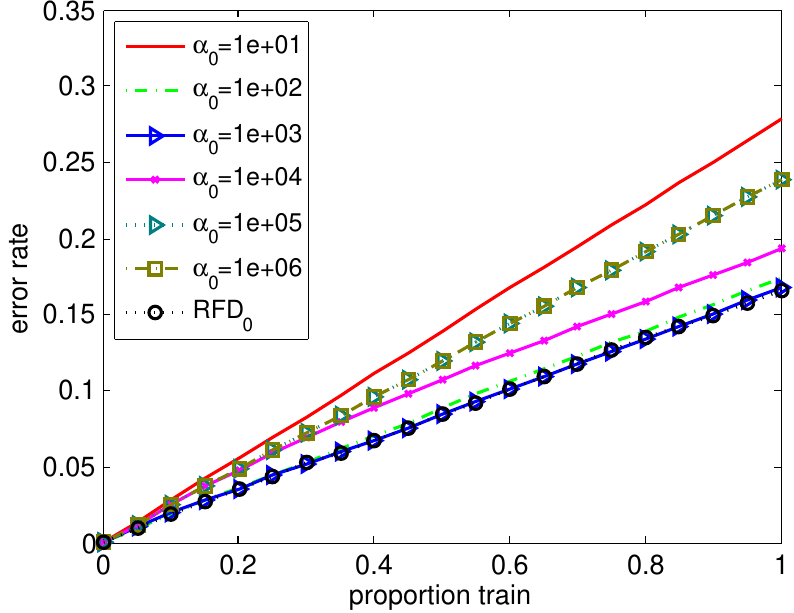} &
        \includegraphics[scale=0.34]{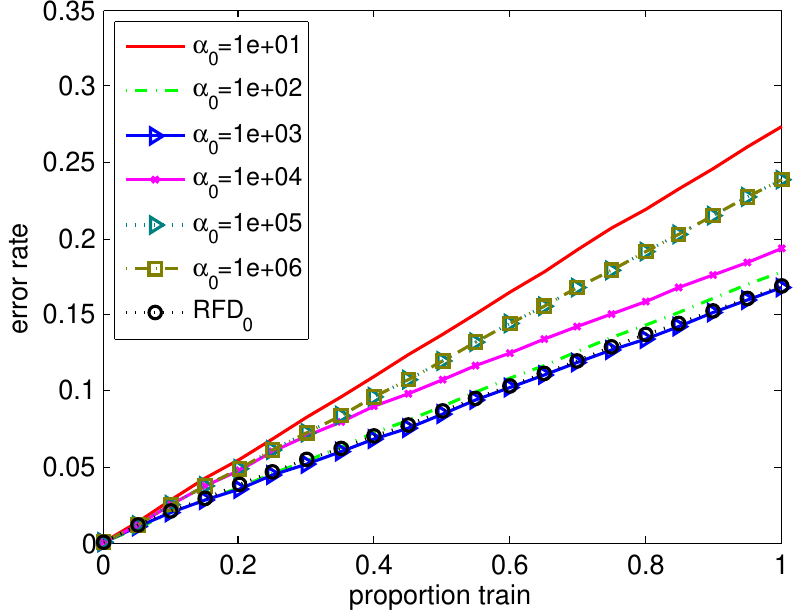} \\
        {\small (g) Oja vs RFD$_0$, $m=5$}  & {\small (h) Oja vs RFD$_0$}, $m=10$ & {\small (i) Oja vs RFD$_0$, $m=20$} \\[0.1cm]
        \includegraphics[scale=0.34]{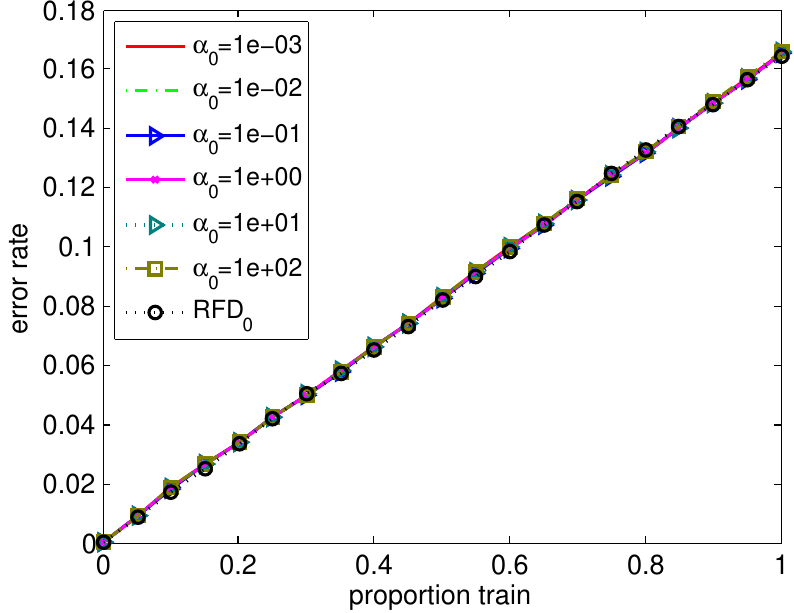} &
        \includegraphics[scale=0.34]{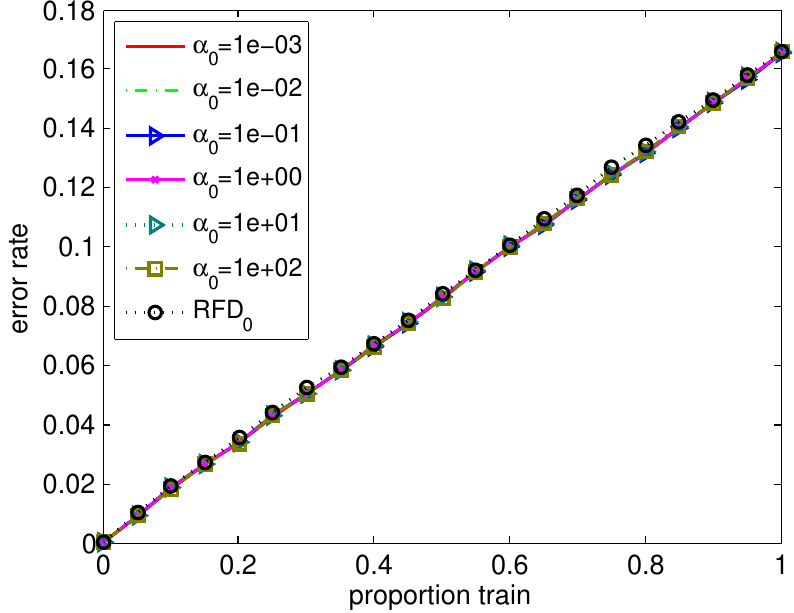} &
        \includegraphics[scale=0.34]{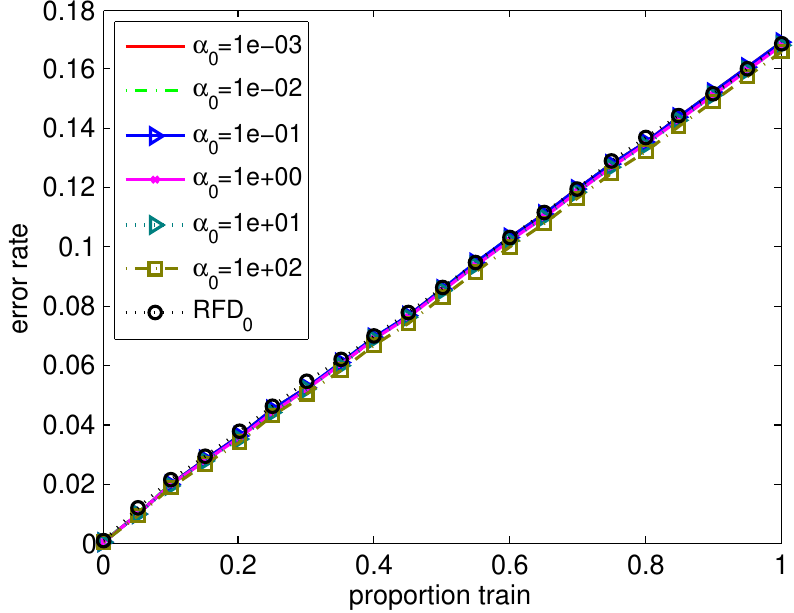} \\
        {\small (j) RFD vs RFD$_0$, $m=5$}  & {\small (k) RFD vs RFD$_0$}, $m=10$ & {\small (l) RFD vs RFD$_0$, $m=20$}
    \end{tabular}
    \caption{Comparison of the online error rate on ``a9a'' }
    \label{figure:train_a9a}
\end{figure}

\begin{figure}[H]
\centering
    \begin{tabular}{ccc}
        \includegraphics[scale=0.34]{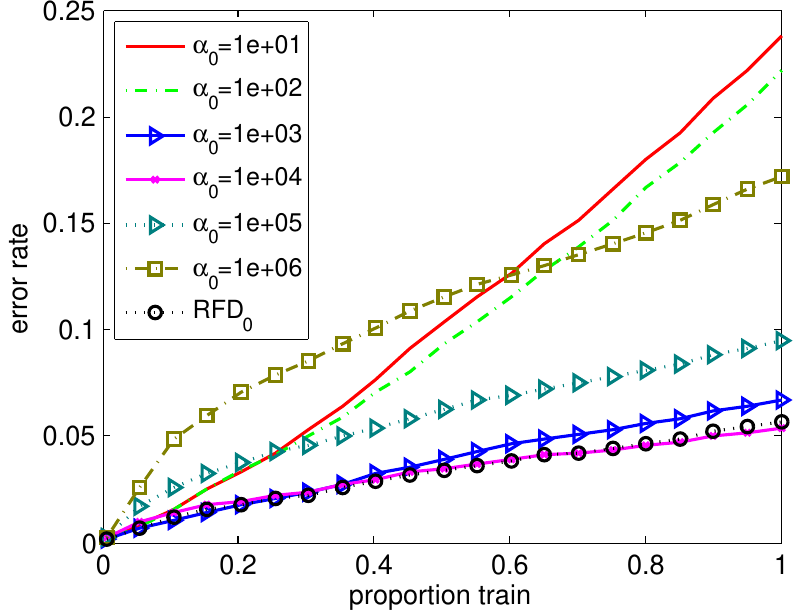} &
        \includegraphics[scale=0.34]{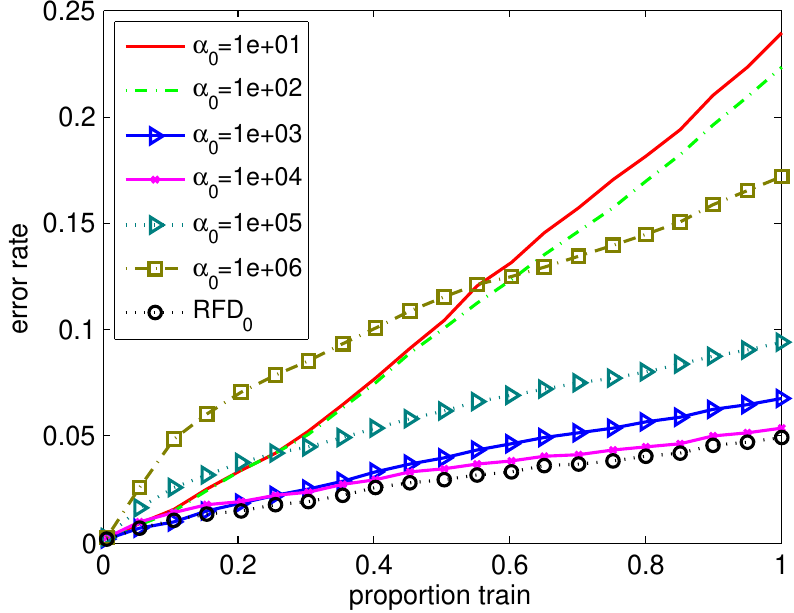} &
        \includegraphics[scale=0.34]{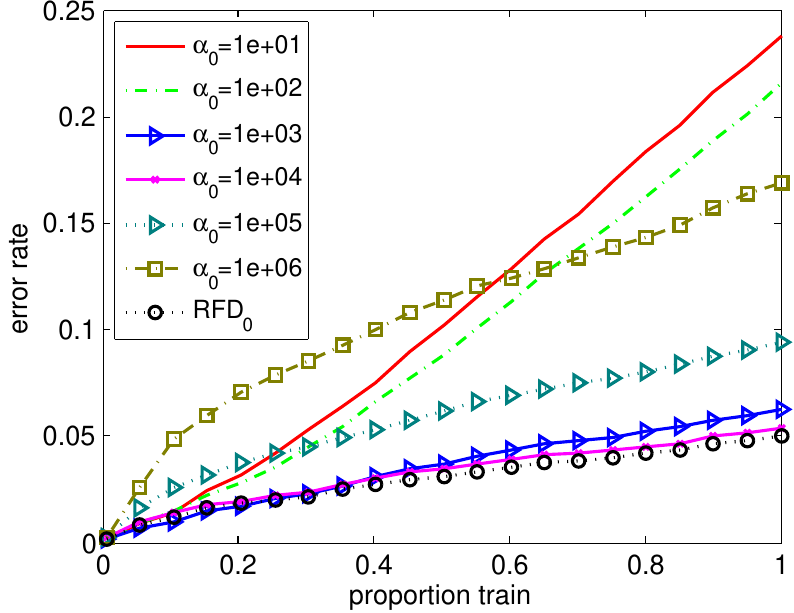} \\
        {\small (a) FD vs RFD$_0$, $m=5$}  & {\small (b) FD vs RFD$_0$}, $m=10$ & {\small (c) FD vs RFD$_0$, $m=20$} \\[0.1cm]
        \includegraphics[scale=0.34]{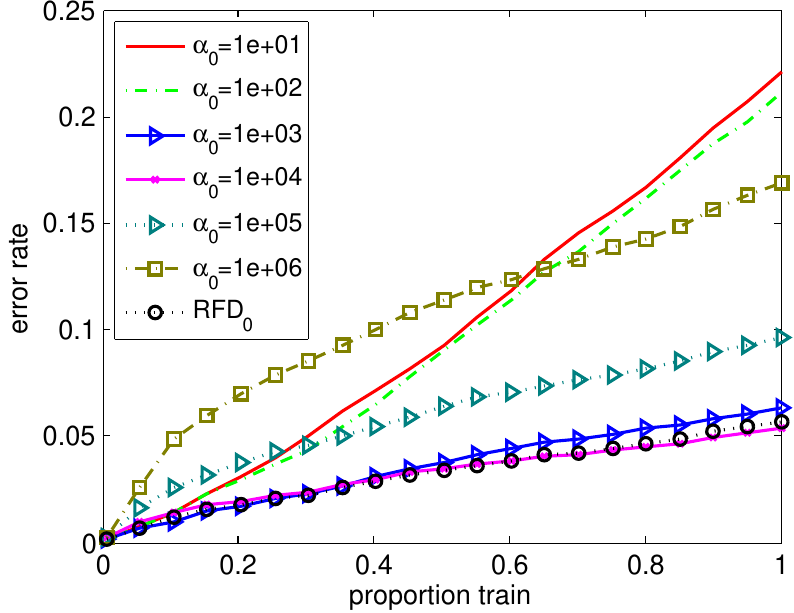} &
        \includegraphics[scale=0.34]{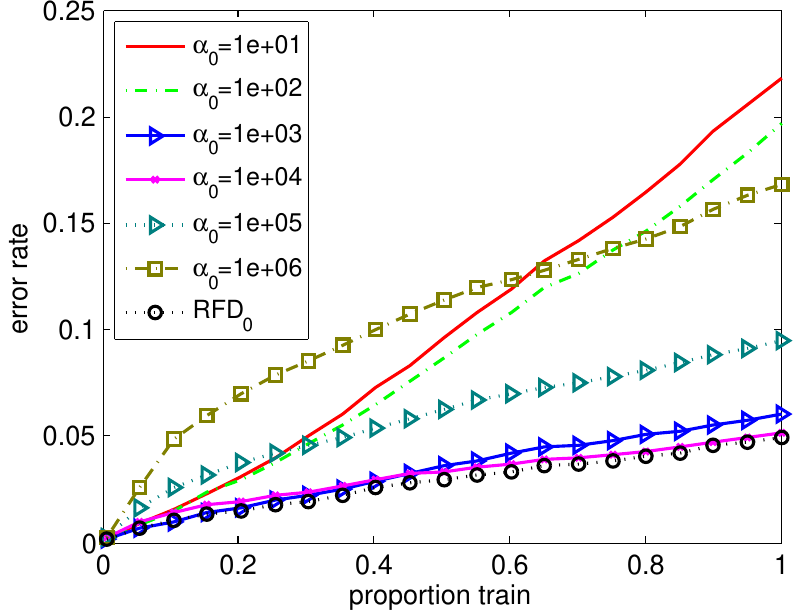} &
        \includegraphics[scale=0.34]{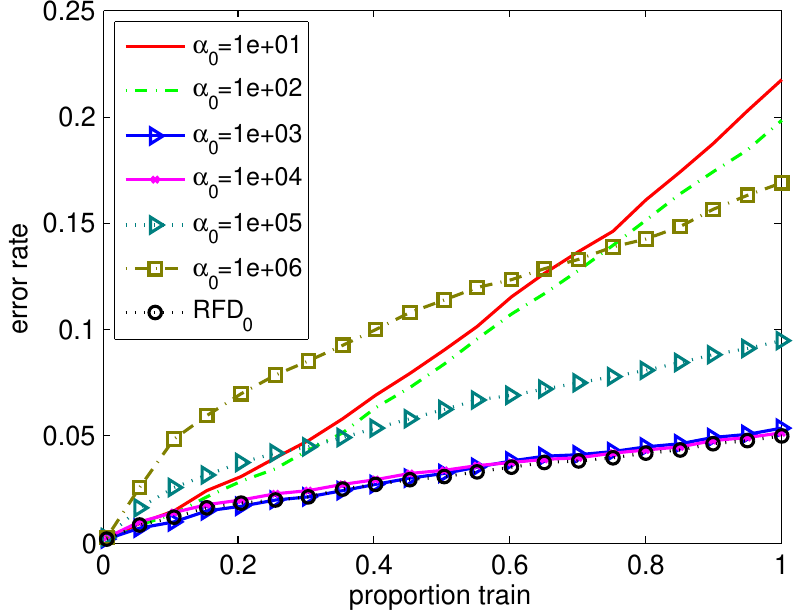} \\
        {\small (a) PFD vs RFD$_0$, $m=5$}  & {\small (b) PFD vs RFD$_0$}, $m=10$ & {\small (c) PFD vs RFD$_0$, $m=20$} \\[0.1cm]
        \includegraphics[scale=0.34]{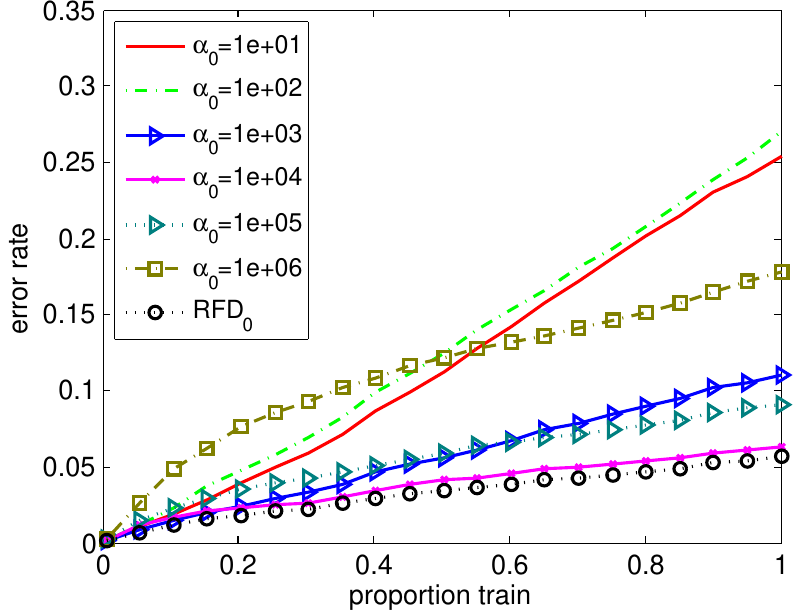} &
        \includegraphics[scale=0.34]{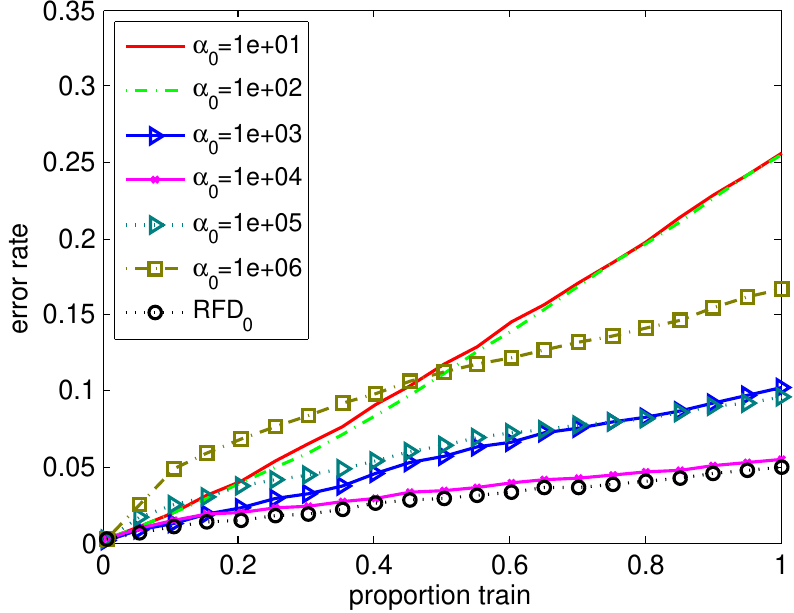} &
        \includegraphics[scale=0.34]{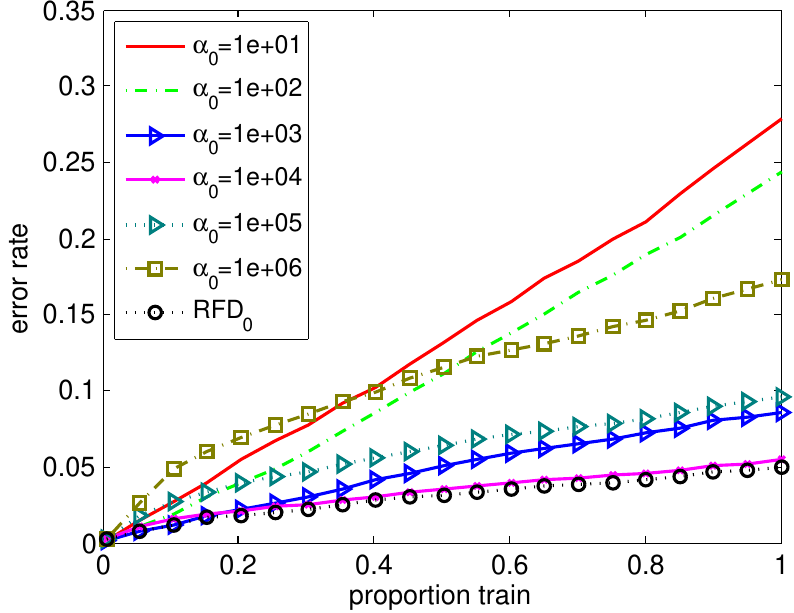} \\
        {\small (d) RP vs RFD$_0$, $m=5$}  & {\small (e) RP vs RFD$_0$}, $m=10$ & {\small (f) RP vs RFD$_0$, $m=20$} \\[0.1cm]
        \includegraphics[scale=0.34]{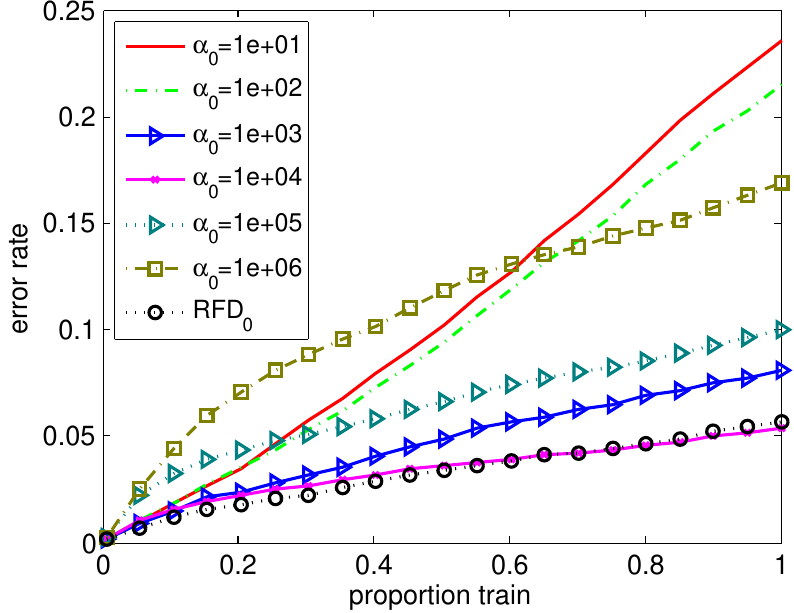} &
        \includegraphics[scale=0.34]{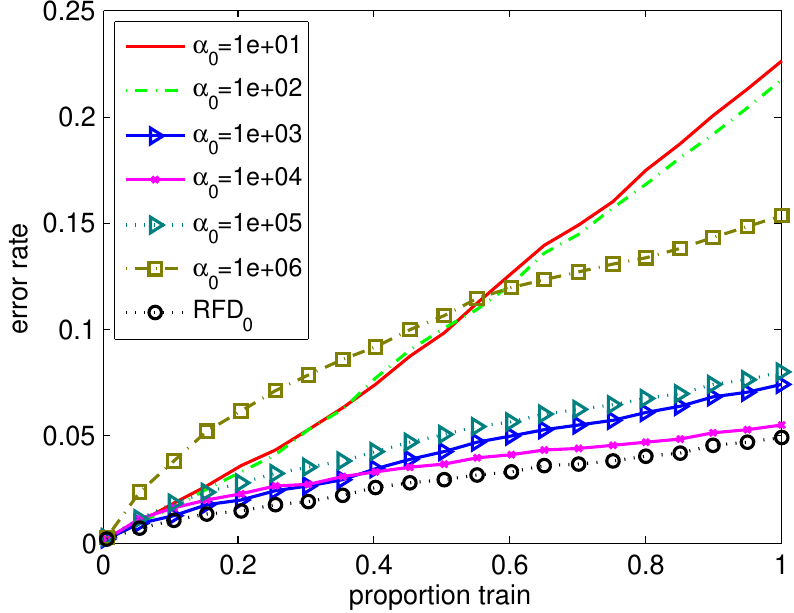} &
        \includegraphics[scale=0.34]{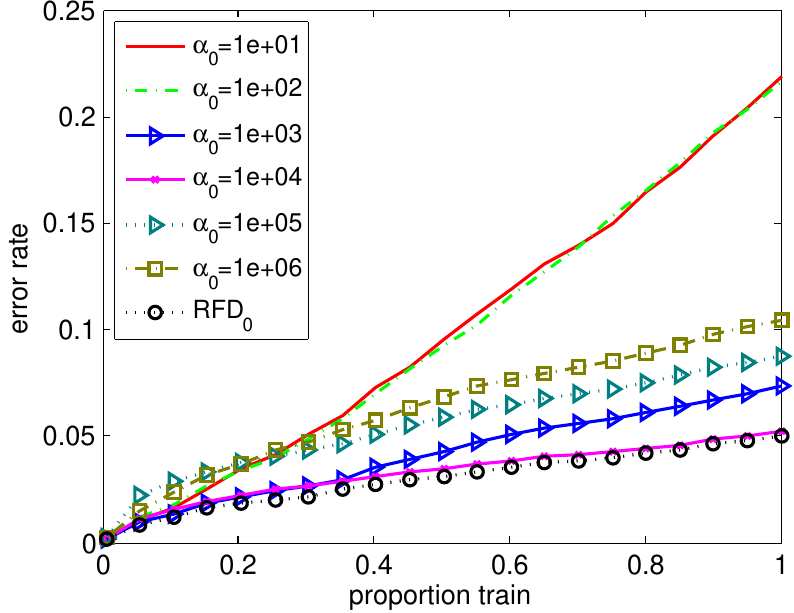} \\
        {\small (g) Oja vs RFD$_0$, $m=5$}  & {\small (h) Oja vs RFD$_0$}, $m=10$ & {\small (i) Oja vs RFD$_0$, $m=20$} \\[0.1cm]
        \includegraphics[scale=0.34]{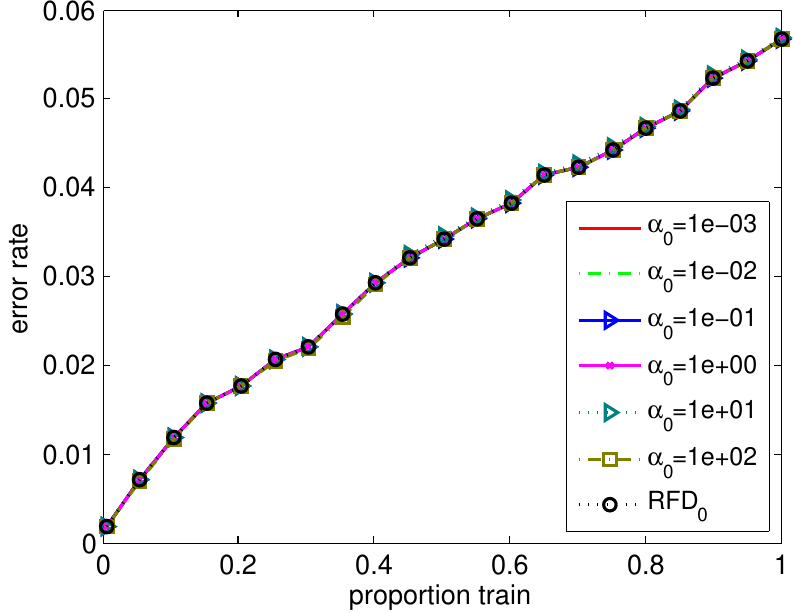} &
        \includegraphics[scale=0.34]{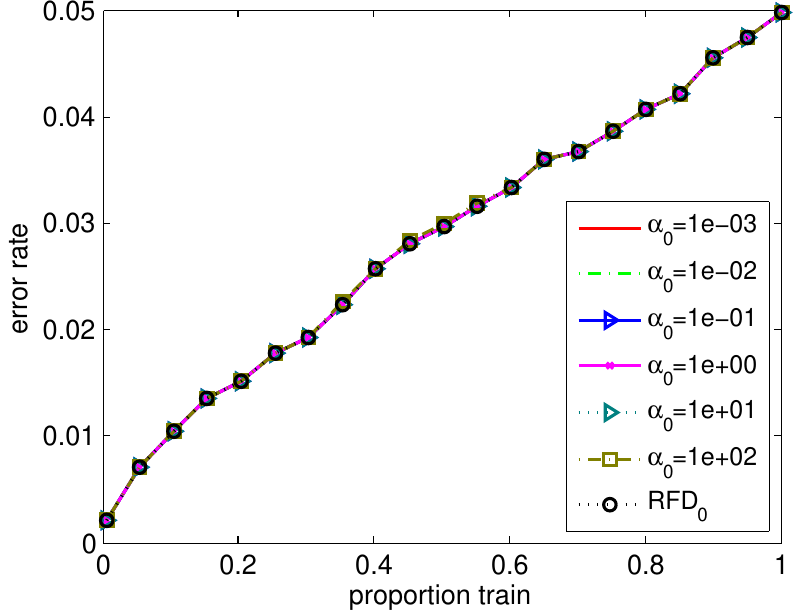} &
        \includegraphics[scale=0.34]{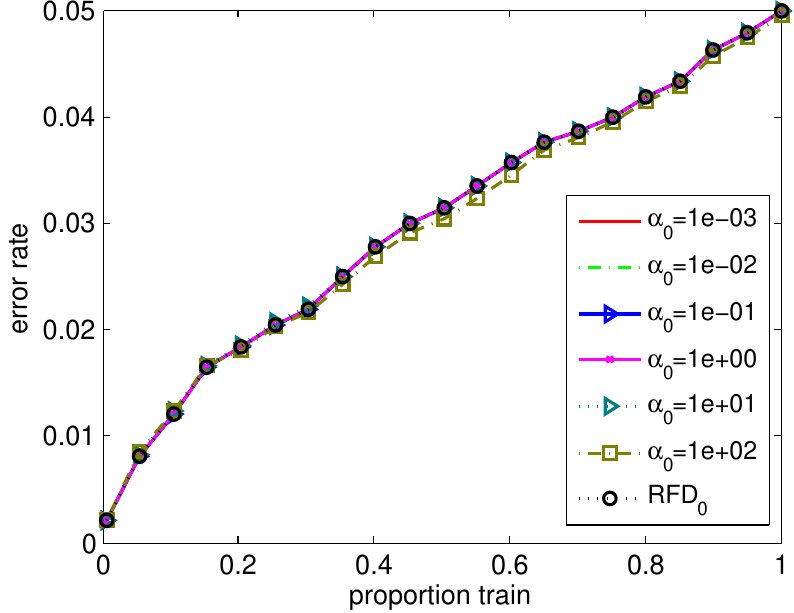} \\
        {\small (j) RFD vs RFD$_0$, $m=5$}  & {\small (k) RFD vs RFD$_0$}, $m=10$ & {\small (l) RFD vs RFD$_0$, $m=20$}
    \end{tabular}
    \caption{Comparison of the online error rate on ``gisette'' }
    \label{figure:train_gisette}
\end{figure}

\begin{figure}[H]
\centering
    \begin{tabular}{ccc}
        \includegraphics[scale=0.34]{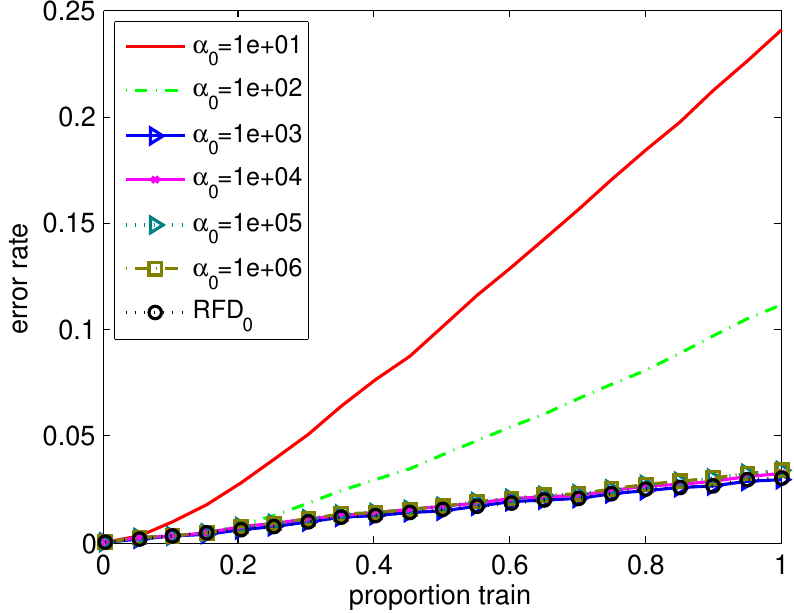} &
        \includegraphics[scale=0.34]{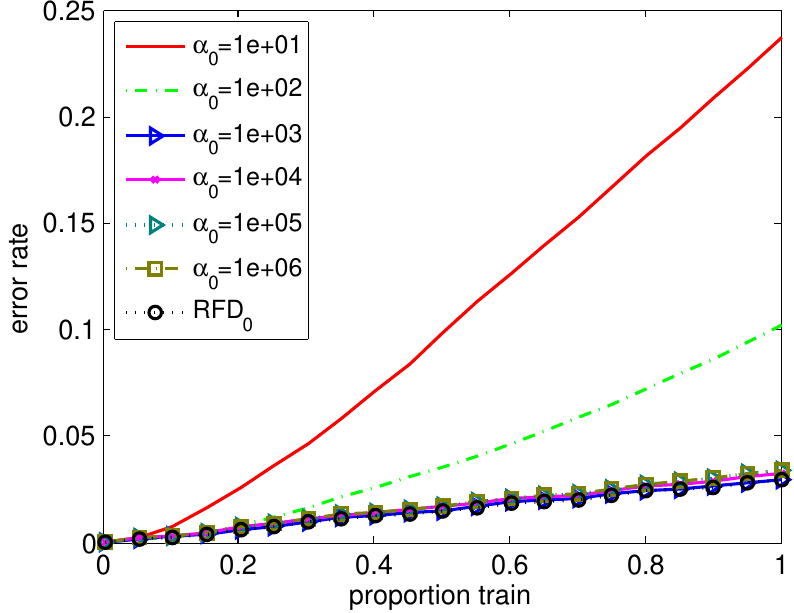} &
        \includegraphics[scale=0.34]{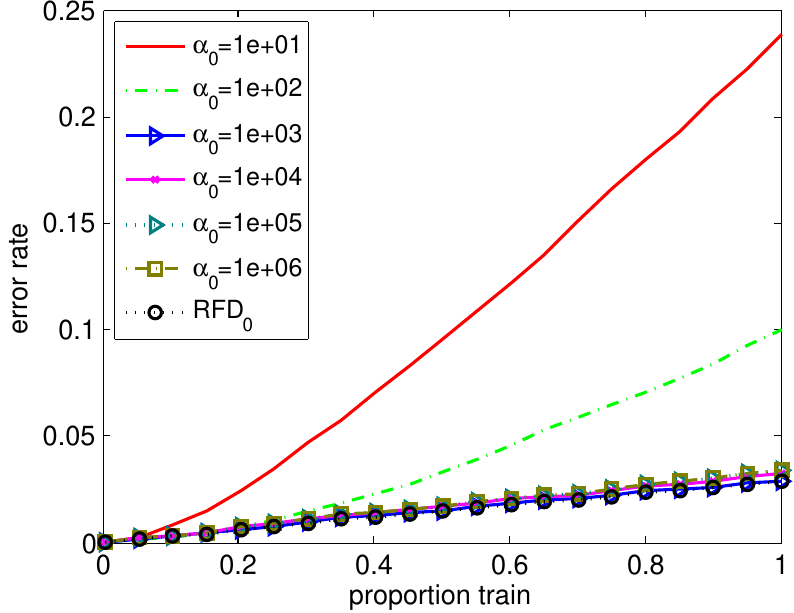} \\
        {\small (a) FD vs RFD$_0$, $m=5$}  & {\small (b) FD vs RFD$_0$}, $m=10$ & {\small (c) FD vs RFD$_0$, $m=20$} \\[0.1cm]
        \includegraphics[scale=0.34]{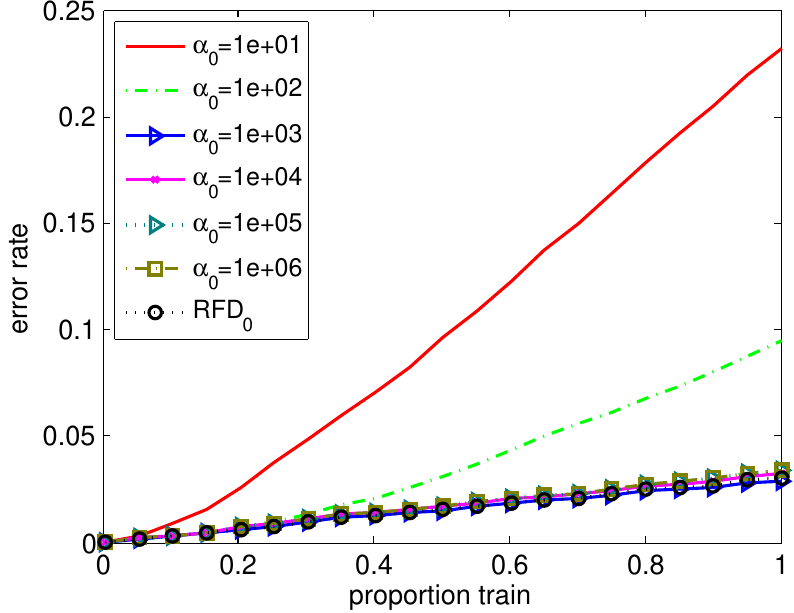} &
        \includegraphics[scale=0.34]{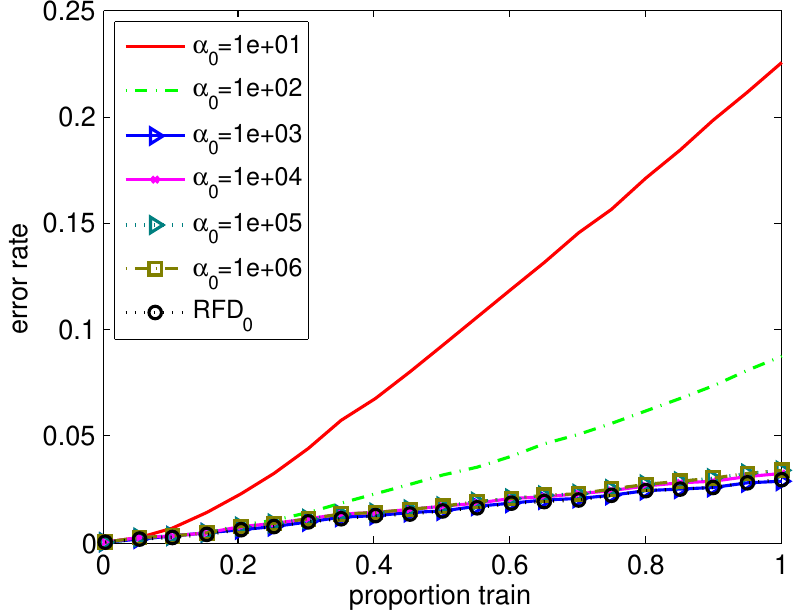} &
        \includegraphics[scale=0.34]{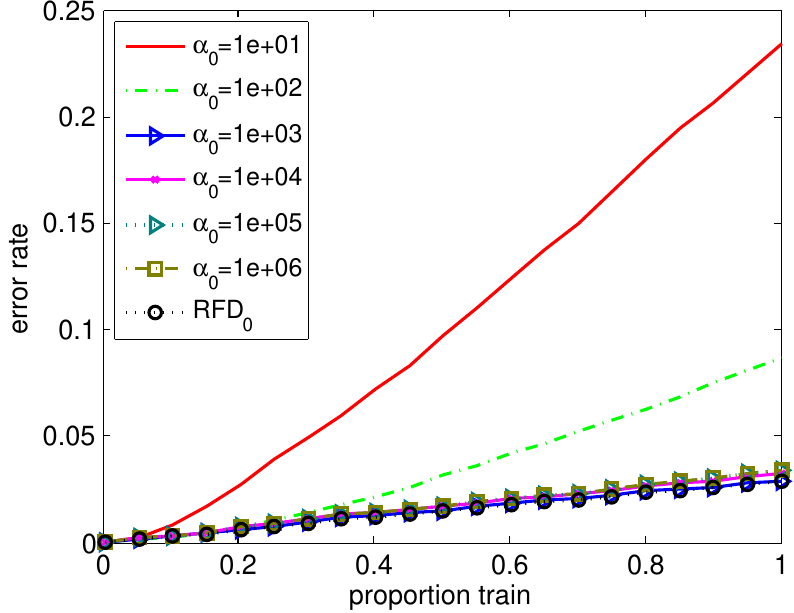} \\
        {\small (a) PFD vs RFD$_0$, $m=5$}  & {\small (b) PFD vs RFD$_0$}, $m=10$ & {\small (c) PFD vs RFD$_0$, $m=20$} \\[0.1cm]
        \includegraphics[scale=0.34]{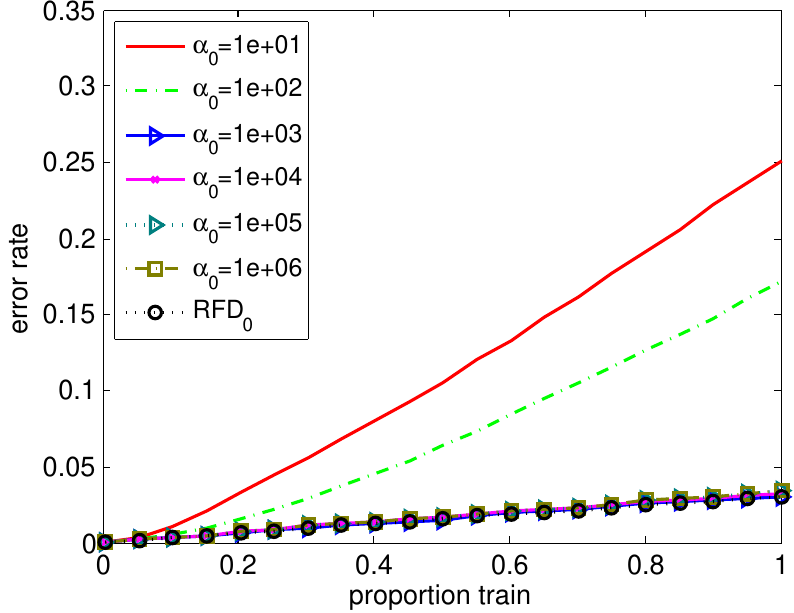} &
        \includegraphics[scale=0.34]{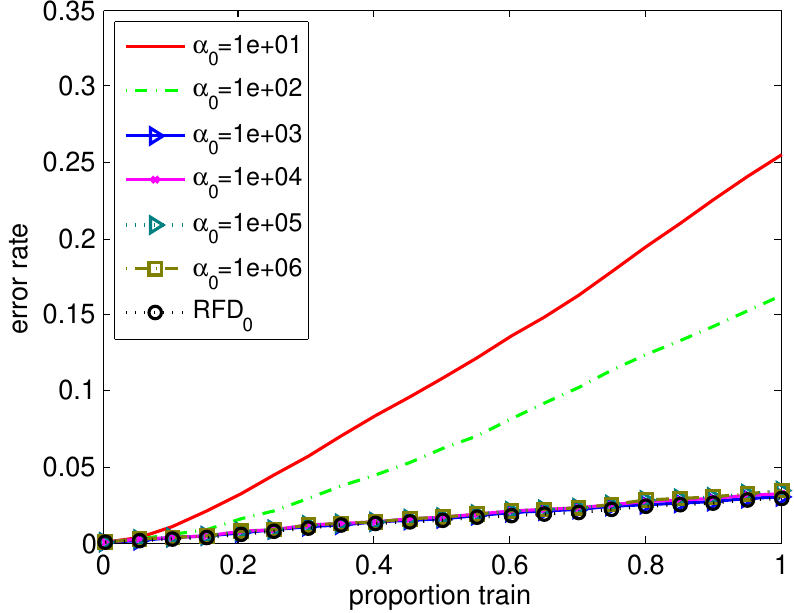} &
        \includegraphics[scale=0.34]{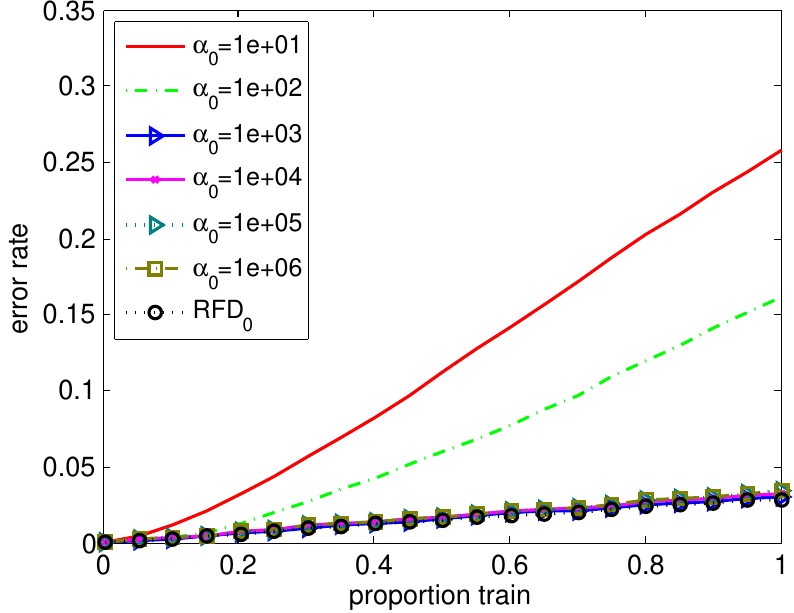} \\
        {\small (d) RP vs RFD$_0$, $m=5$}  & {\small (e) RP vs RFD$_0$}, $m=10$ & {\small (f) RP vs RFD$_0$, $m=20$} \\[0.1cm]
        \includegraphics[scale=0.34]{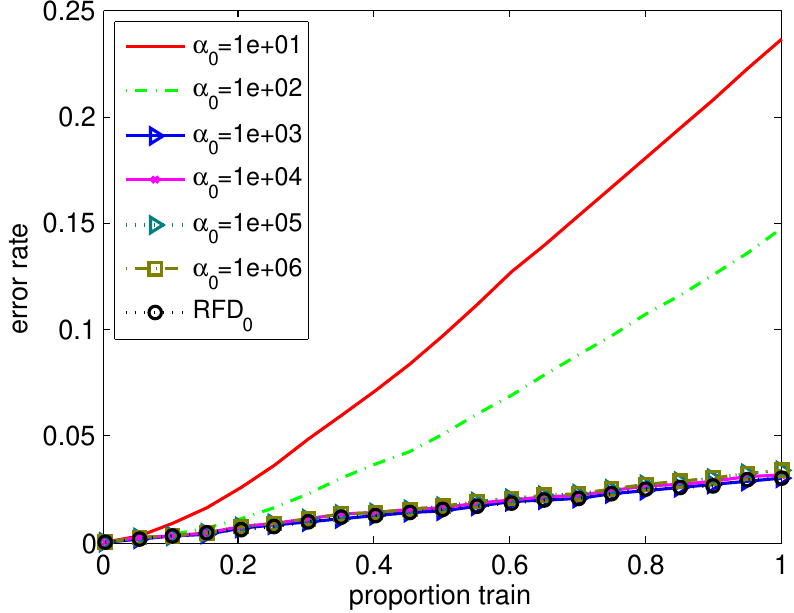} &
        \includegraphics[scale=0.34]{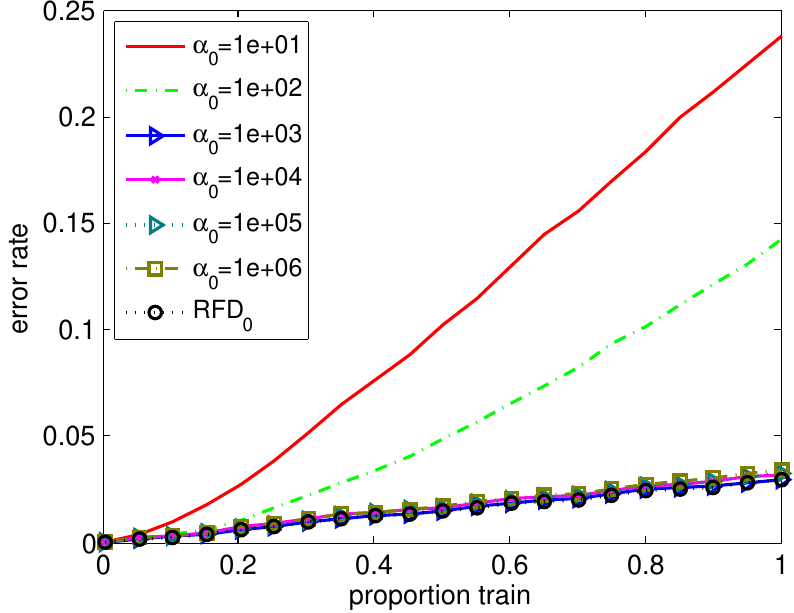} &
        \includegraphics[scale=0.34]{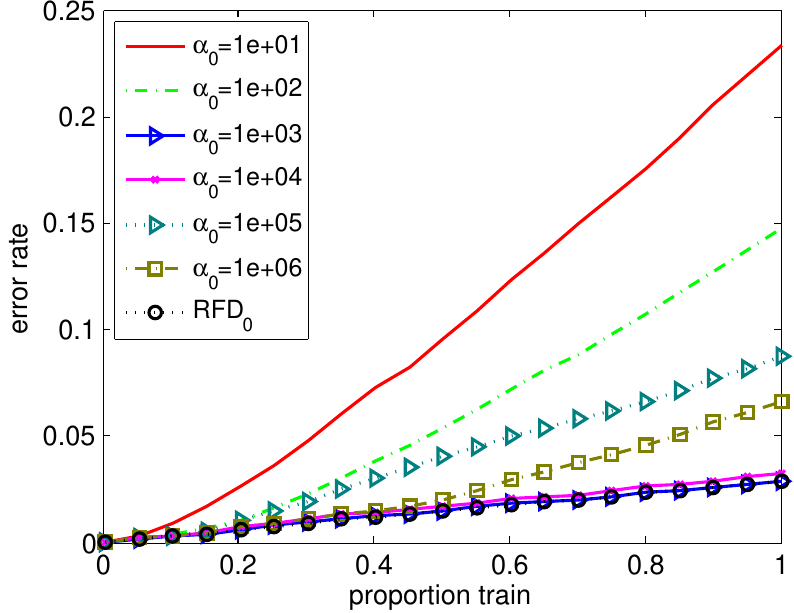} \\
        {\small (g) Oja vs RFD$_0$, $m=5$}  & {\small (h) Oja vs RFD$_0$}, $m=10$ & {\small (i) Oja vs RFD$_0$, $m=20$} \\[0.1cm]
        \includegraphics[scale=0.34]{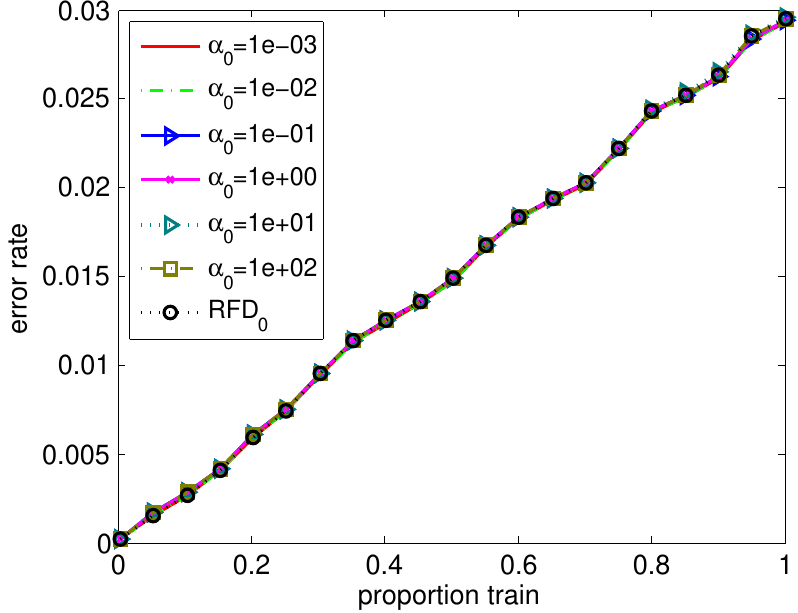} &
        \includegraphics[scale=0.34]{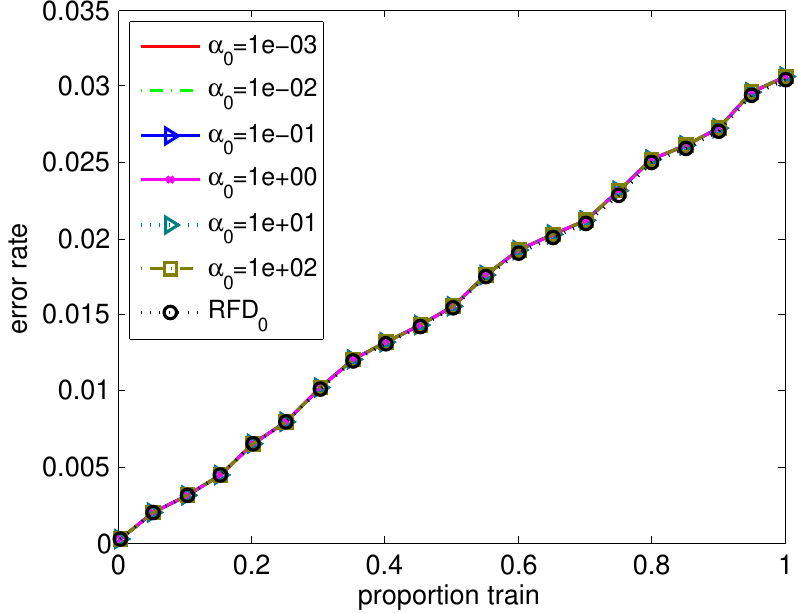} &
        \includegraphics[scale=0.34]{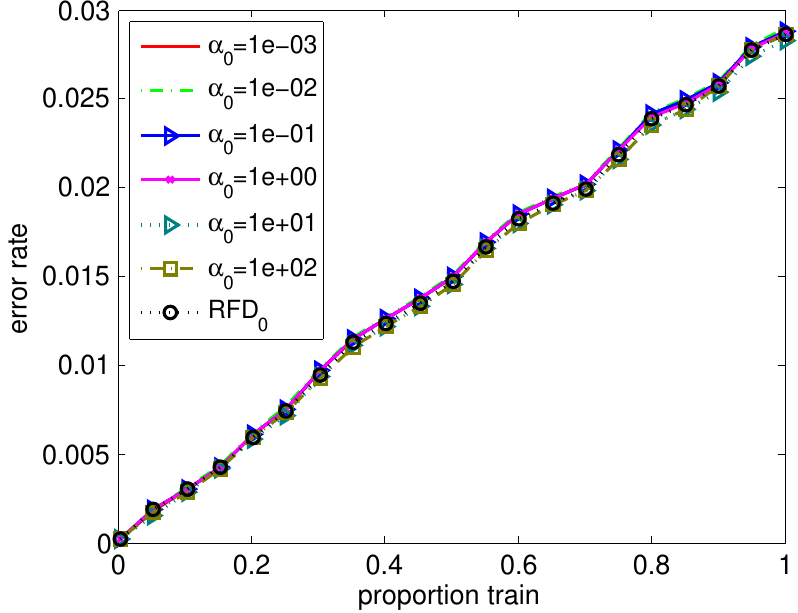} \\
        {\small (j) RFD vs RFD$_0$, $m=5$}  & {\small (k) RFD vs RFD$_0$}, $m=10$ & {\small (l) RFD vs RFD$_0$, $m=20$}
    \end{tabular}
    \caption{Comparison of the online error rate on  ``sido0'' }
    \label{figure:train_sido0}
\end{figure}

\begin{figure}[H]
\centering
    \begin{tabular}{ccc}
        \includegraphics[scale=0.34]{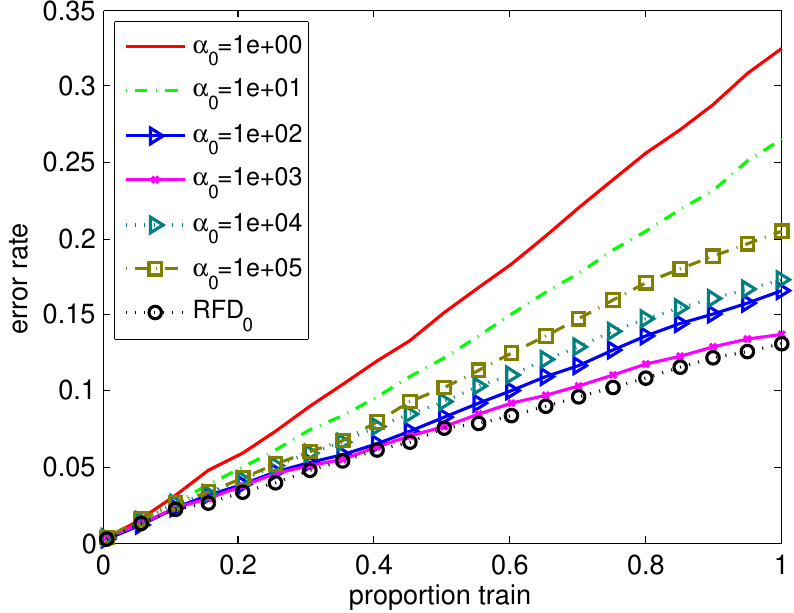} &
        \includegraphics[scale=0.34]{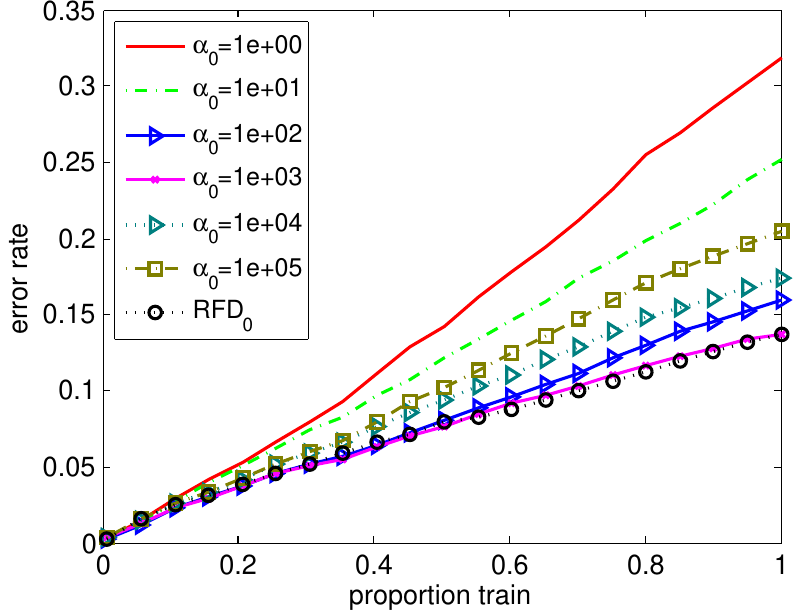} &
        \includegraphics[scale=0.34]{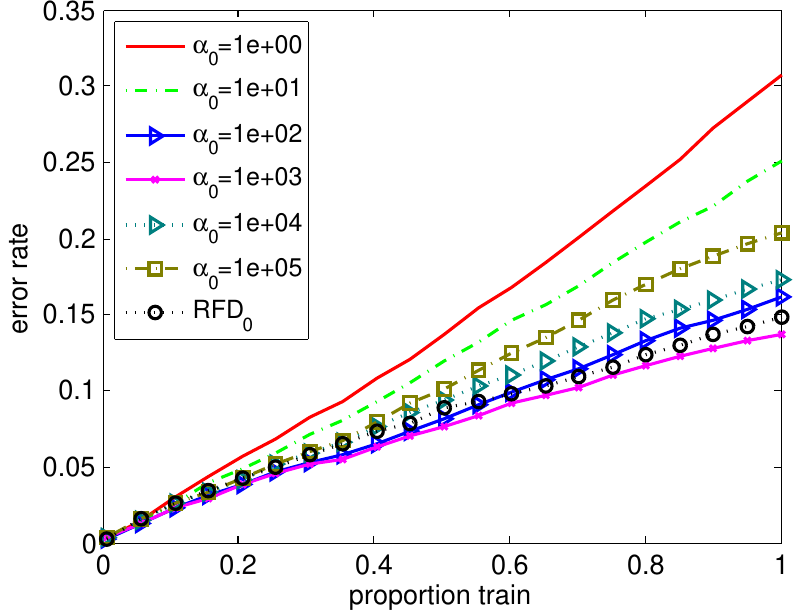} \\
        {\small (a) FD vs RFD$_0$, $m=20$}  & {\small (b) FD vs RFD$_0$}, $m=30$ & {\small (c) FD vs RFD$_0$, $m=50$} \\[0.1cm]
        \includegraphics[scale=0.34]{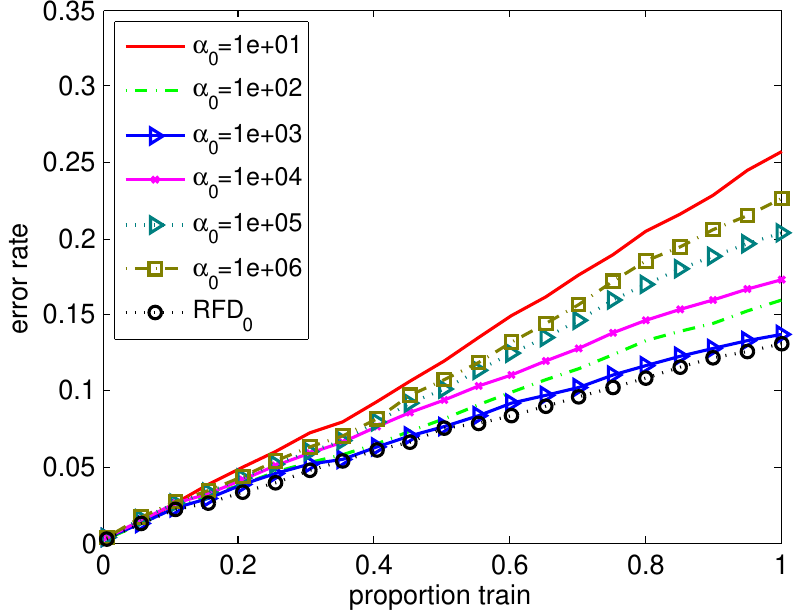} &
        \includegraphics[scale=0.34]{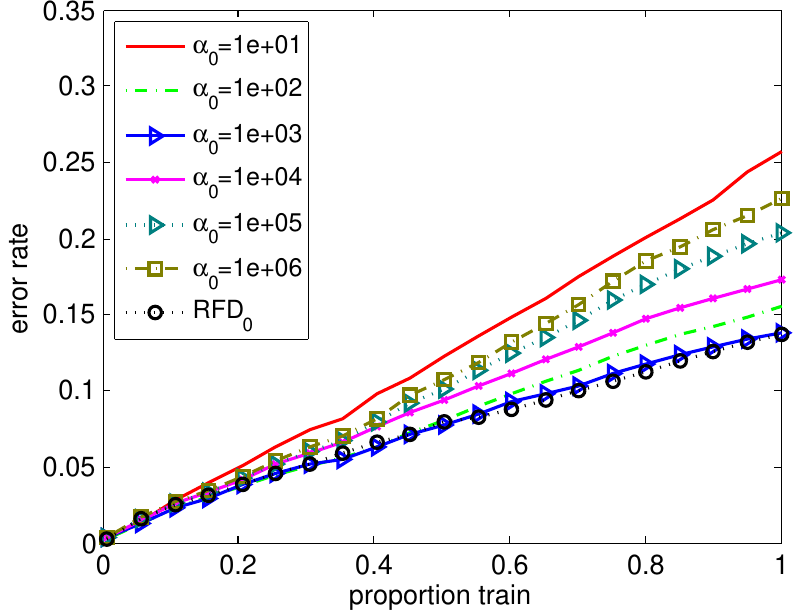} &
        \includegraphics[scale=0.34]{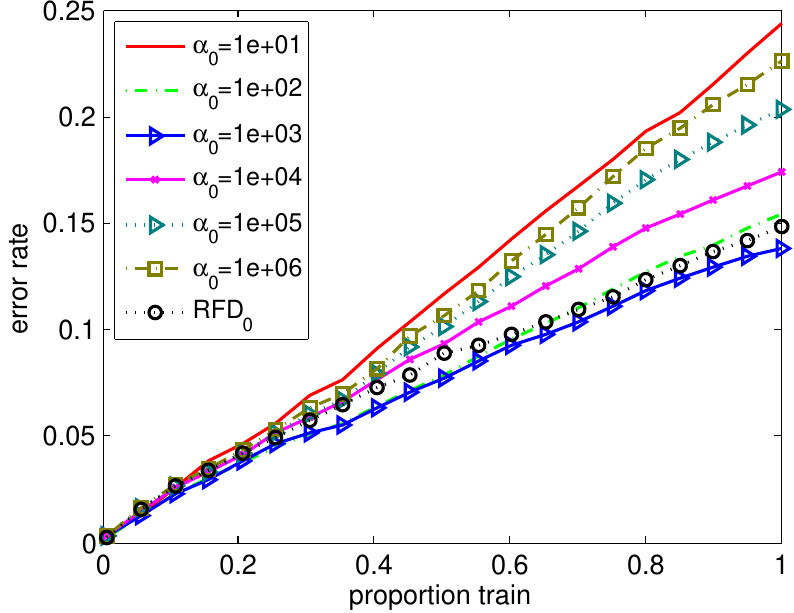} \\
        {\small (a) PFD vs RFD$_0$, $m=20$}  & {\small (b) PFD vs RFD$_0$}, $m=30$ & {\small (c) PFD vs RFD$_0$, $m=50$} \\[0.1cm]
        \includegraphics[scale=0.34]{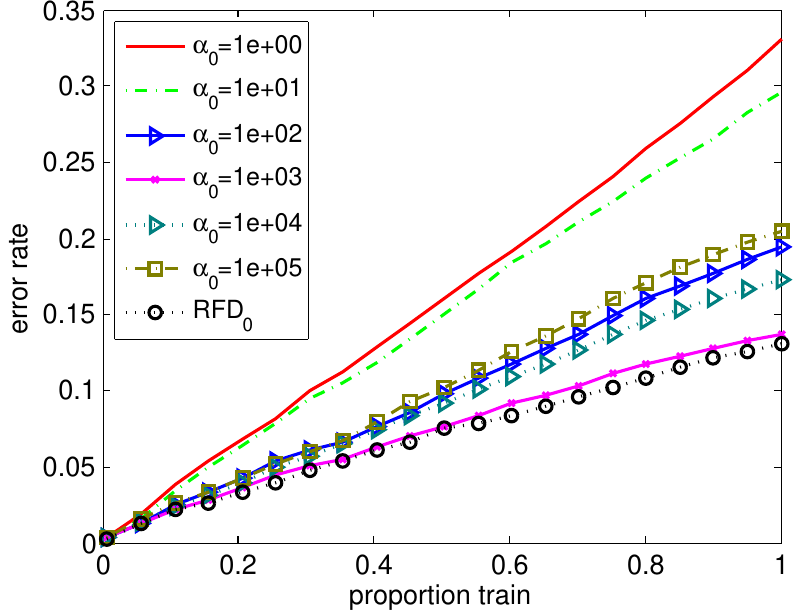} &
        \includegraphics[scale=0.34]{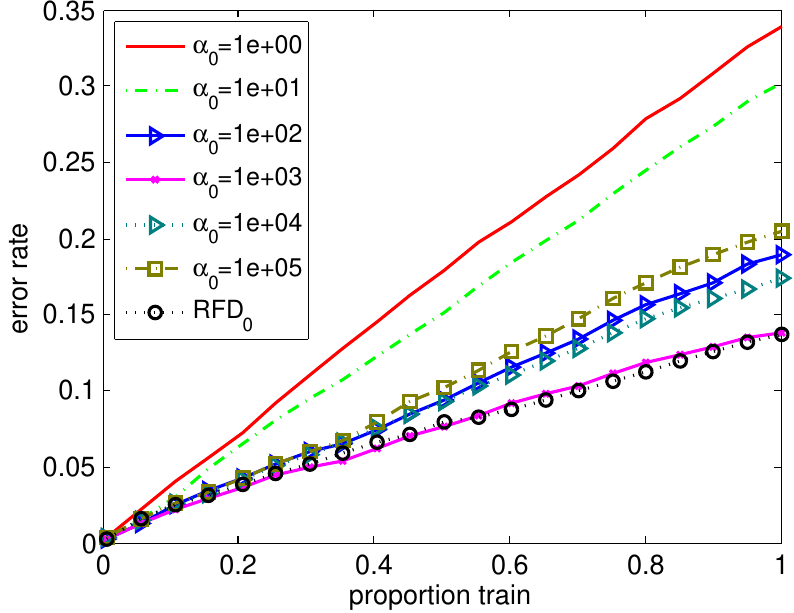} &
        \includegraphics[scale=0.34]{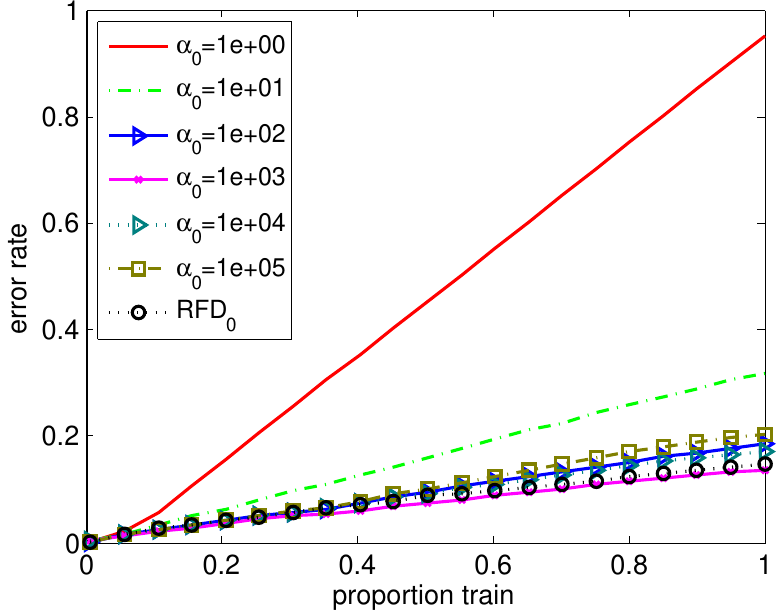} \\
        {\small (d) RP vs RFD$_0$, $m=20$}  & {\small (e) RP vs RFD$_0$}, $m=30$ & {\small (f) RP vs RFD$_0$, $m=50$} \\[0.1cm]
        \includegraphics[scale=0.34]{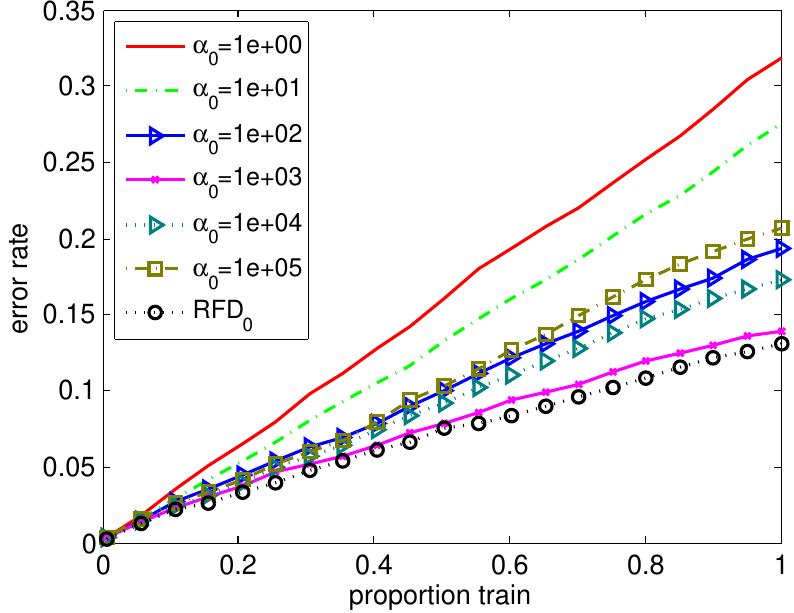} &
        \includegraphics[scale=0.34]{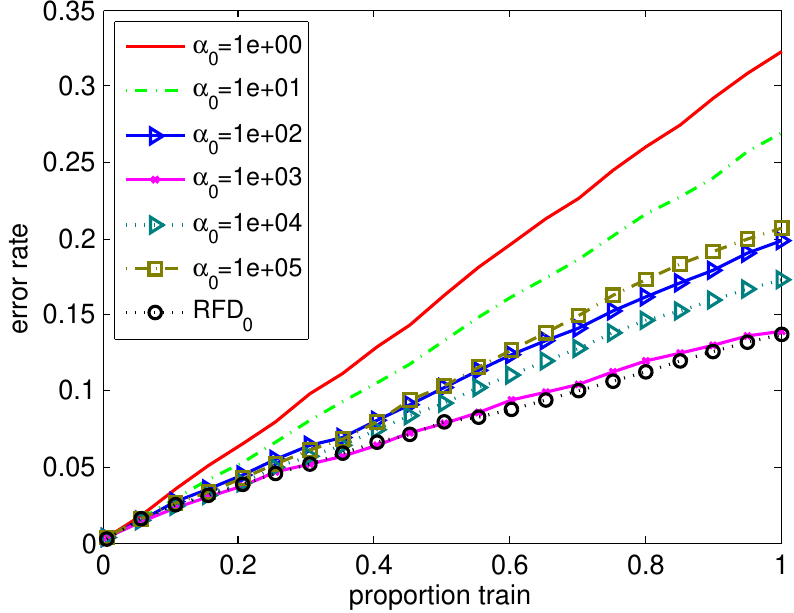} &
        \includegraphics[scale=0.34]{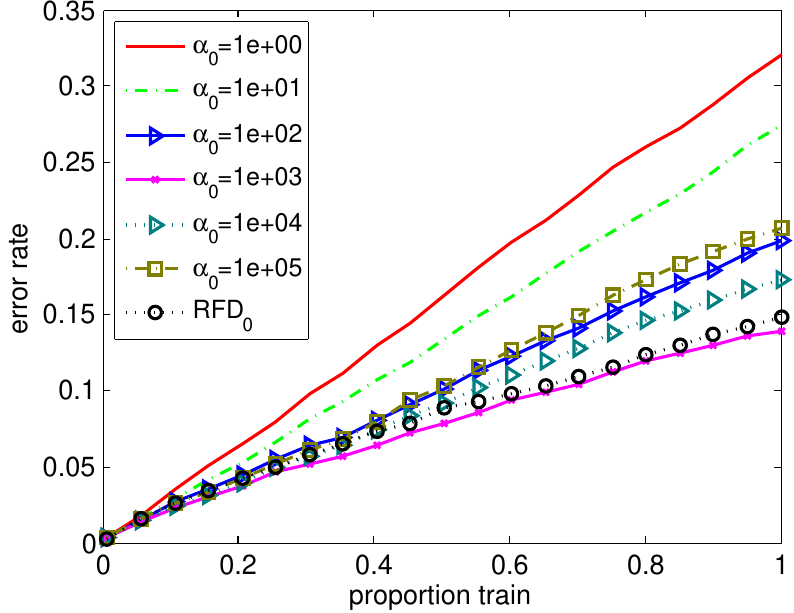} \\
        {\small (g) Oja vs RFD$_0$, $m=20$}  & {\small (h) Oja vs RFD$_0$}, $m=30$ & {\small (i) Oja vs RFD$_0$, $m=50$} \\[0.1cm]
        \includegraphics[scale=0.34]{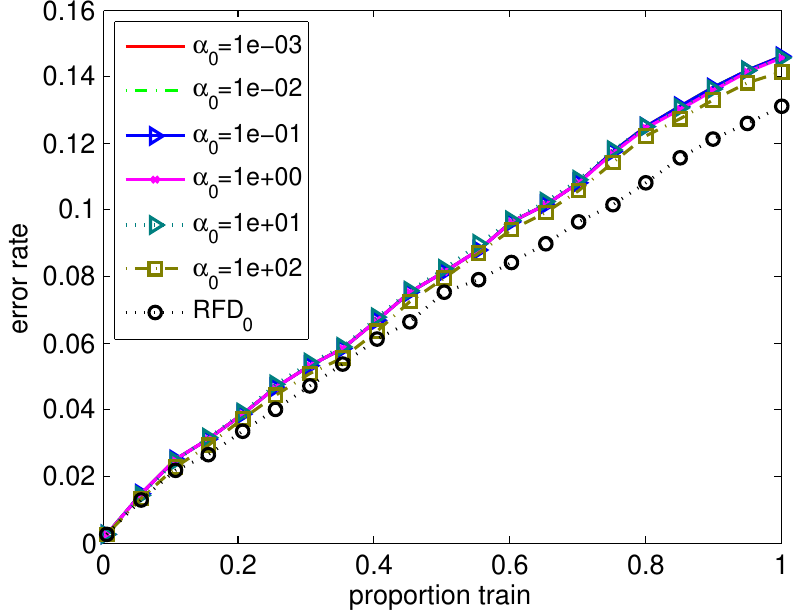} &
        \includegraphics[scale=0.34]{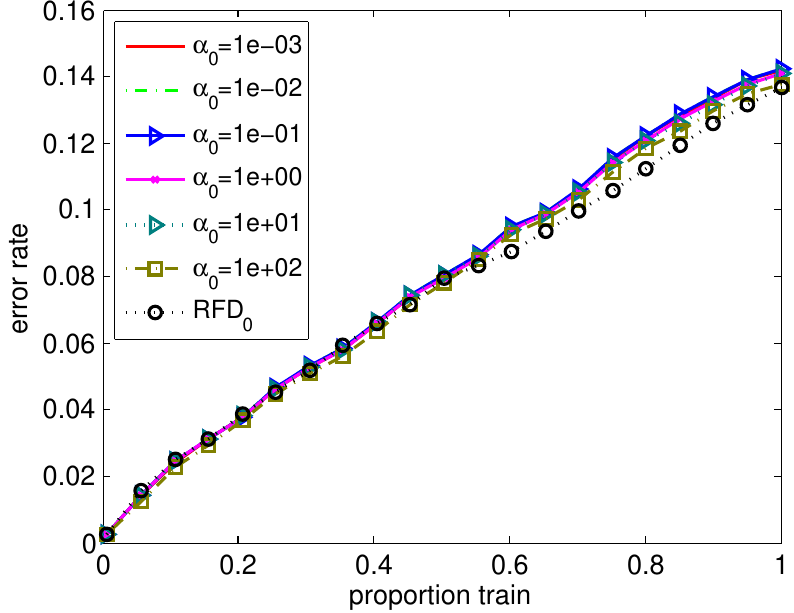} &
        \includegraphics[scale=0.34]{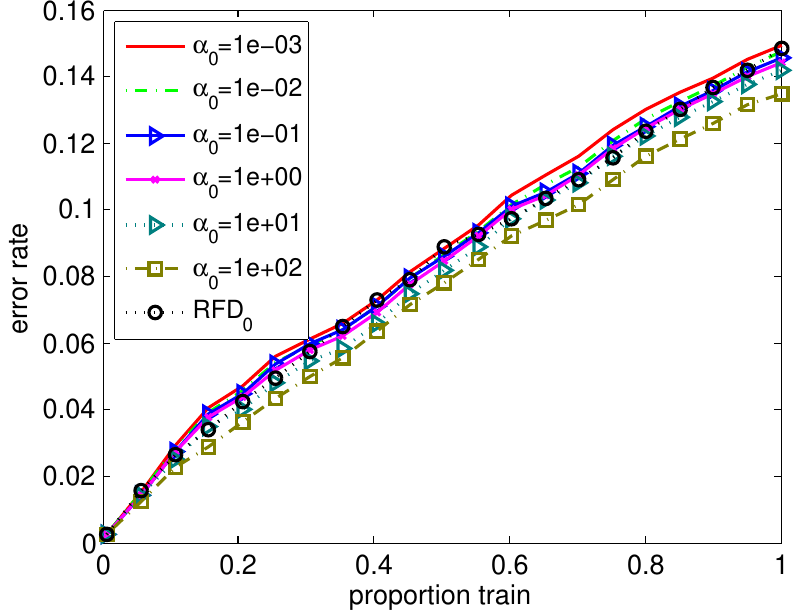} \\
        {\small (j) RFD vs RFD$_0$, $m=20$}  & {\small (k) RFD vs RFD$_0$}, $m=30$ & {\small (l) RFD vs RFD$_0$, $m=50$} \\[0.1cm]
    \end{tabular}
    \caption{Comparison of the online error rate on  ``farm-ads'' }
    \label{figure:train_farm-ads}
\end{figure}

\begin{figure}[H]
\centering
    \begin{tabular}{ccc}
        \includegraphics[scale=0.34]{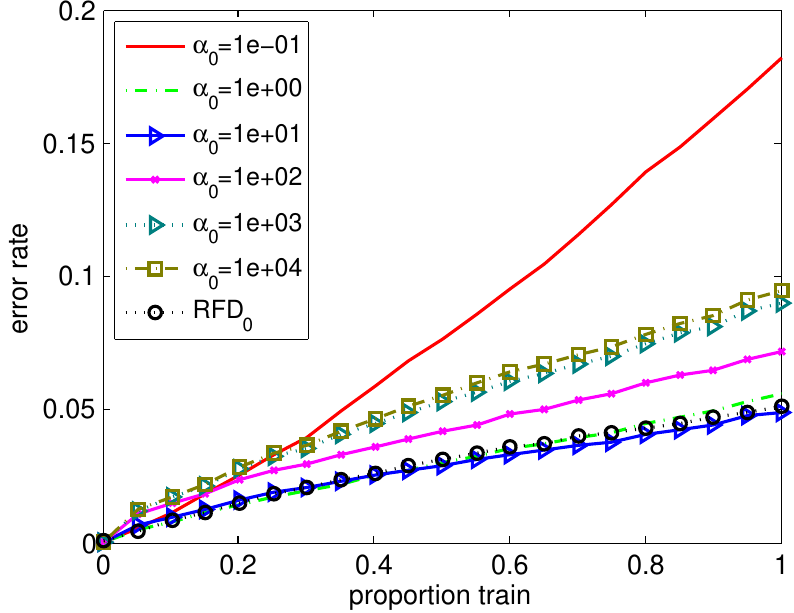} &
        \includegraphics[scale=0.34]{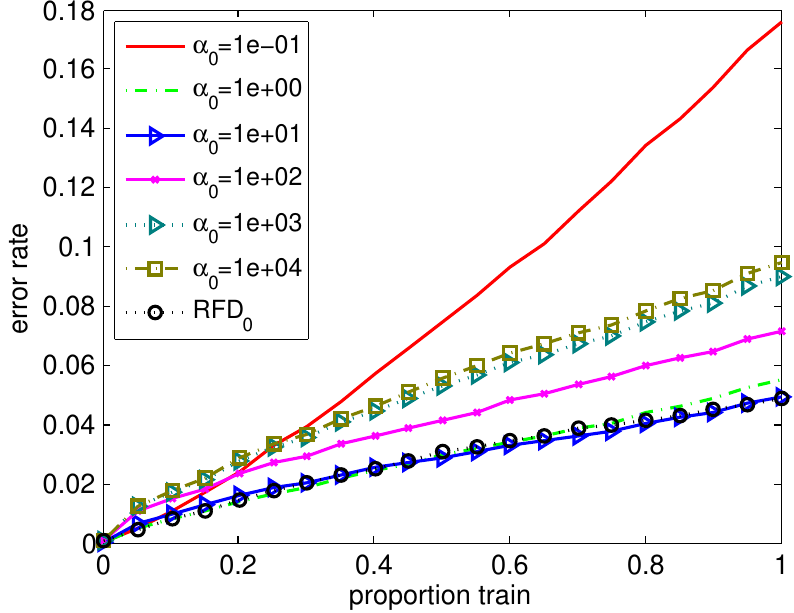} &
        \includegraphics[scale=0.34]{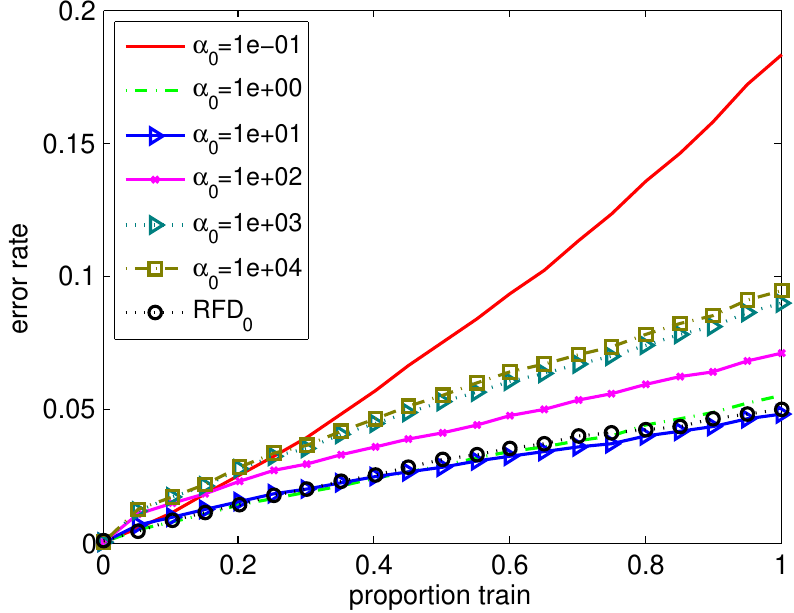} \\
        {\small (a) FD vs RFD$_0$, $m=20$}  & {\small (b) FD vs RFD$_0$}, $m=30$ & {\small (c) FD vs RFD$_0$, $m=50$} \\[0.1cm]
        \includegraphics[scale=0.34]{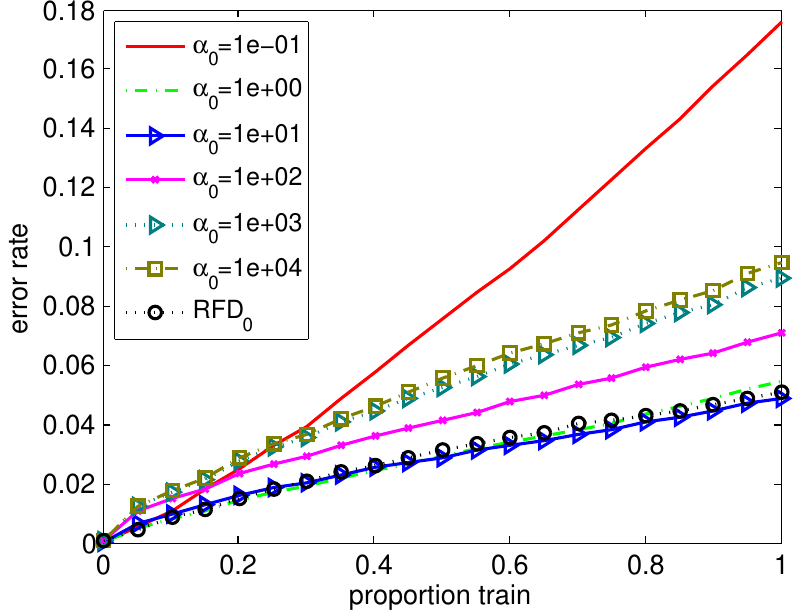} &
        \includegraphics[scale=0.34]{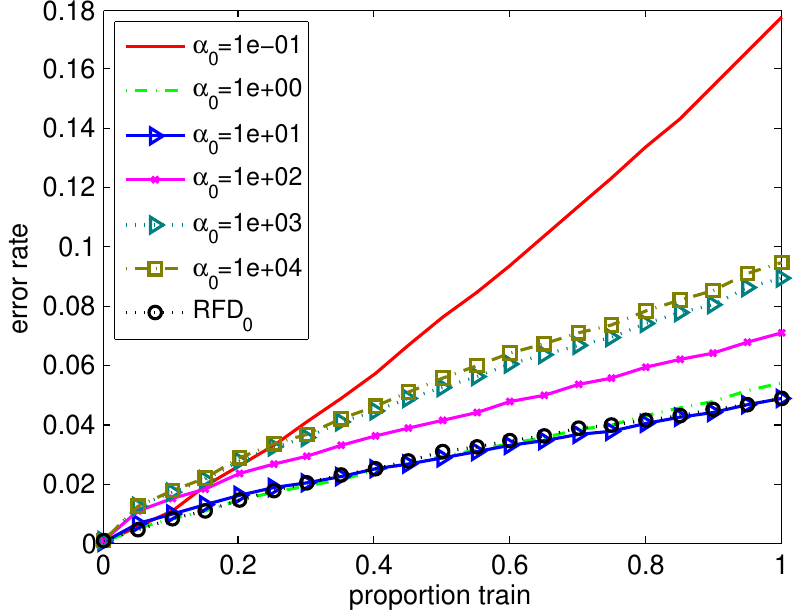} &
        \includegraphics[scale=0.34]{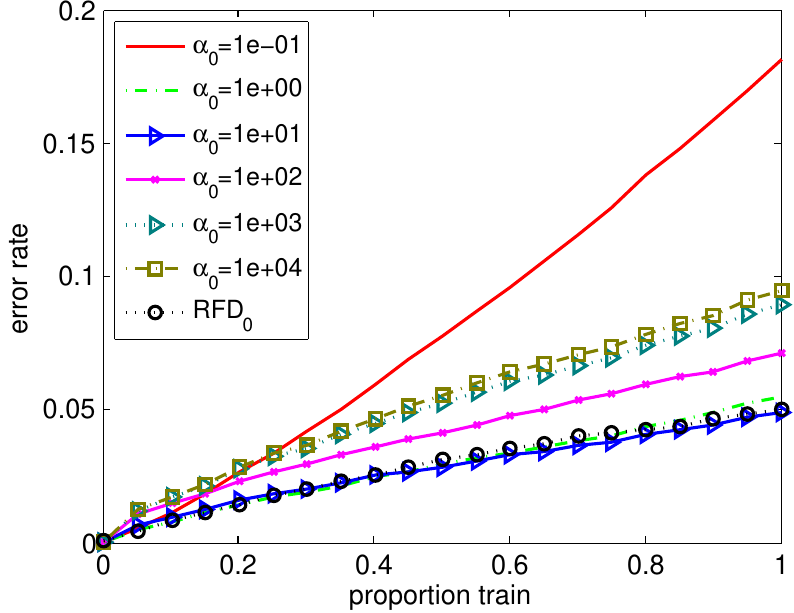} \\
        {\small (a) PFD vs RFD$_0$, $m=20$}  & {\small (b) FD vs PFD$_0$}, $m=30$ & {\small (c) PFD vs RFD$_0$, $m=50$} \\[0.1cm]
        \includegraphics[scale=0.34]{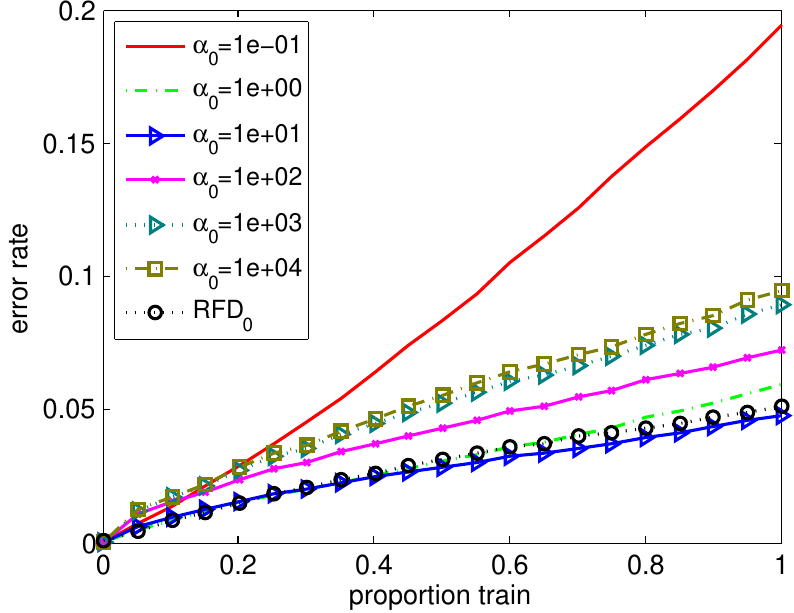} &
        \includegraphics[scale=0.34]{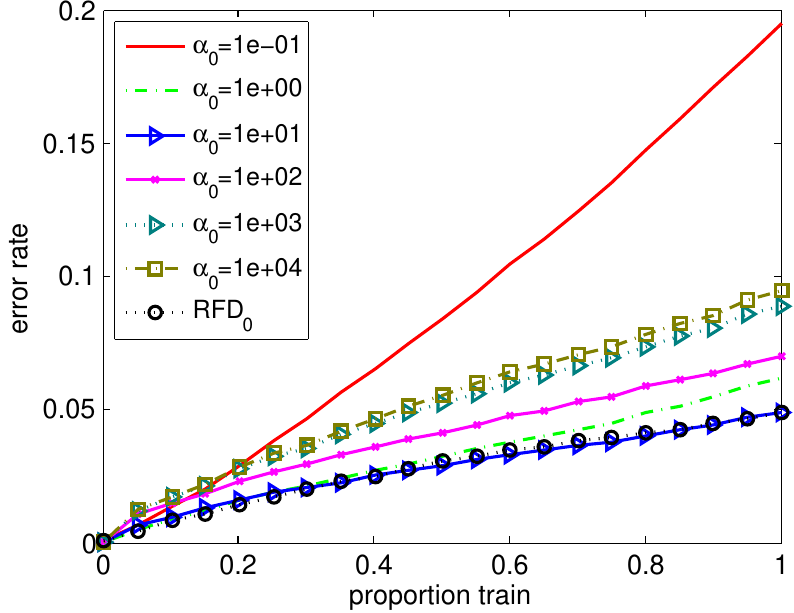} &
        \includegraphics[scale=0.34]{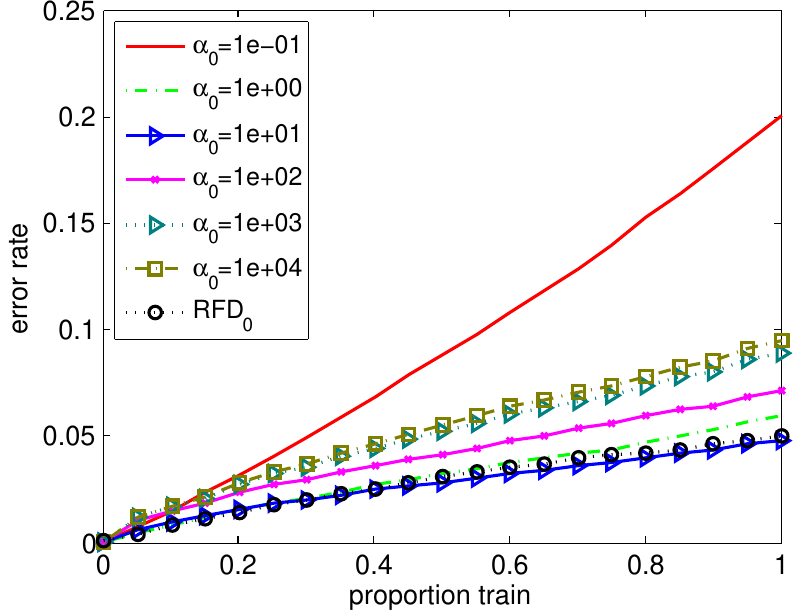} \\
        {\small (d) RP vs RFD$_0$, $m=20$}  & {\small (e) RP vs RFD$_0$}, $m=30$ & {\small (f) RP vs RFD$_0$, $m=50$} \\[0.1cm]
        \includegraphics[scale=0.34]{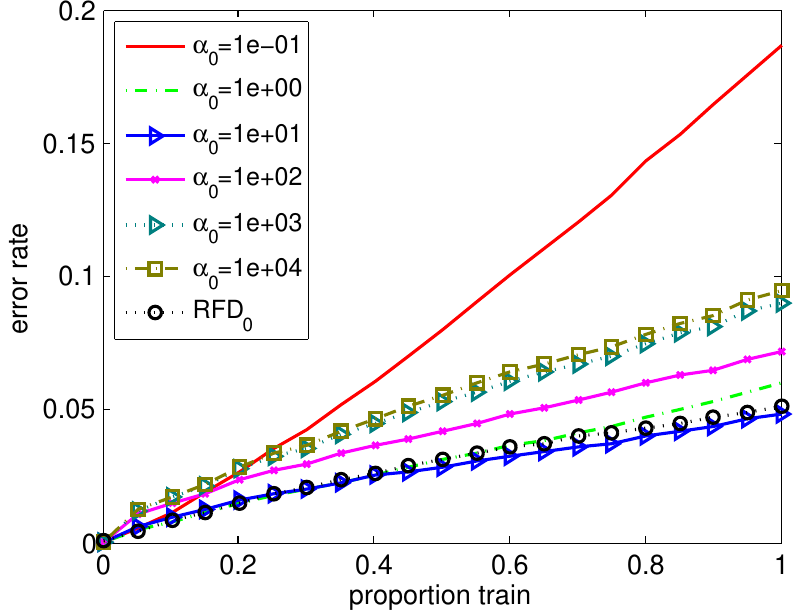} &
        \includegraphics[scale=0.34]{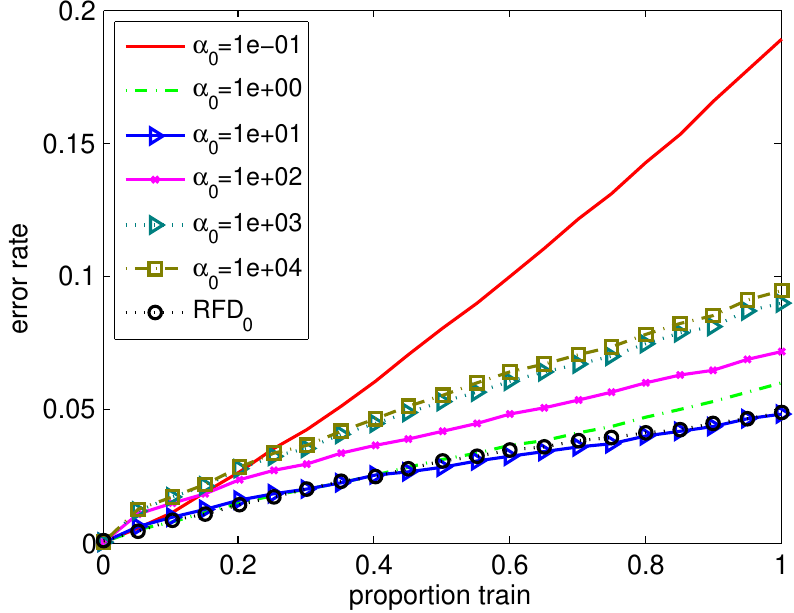} &
        \includegraphics[scale=0.34]{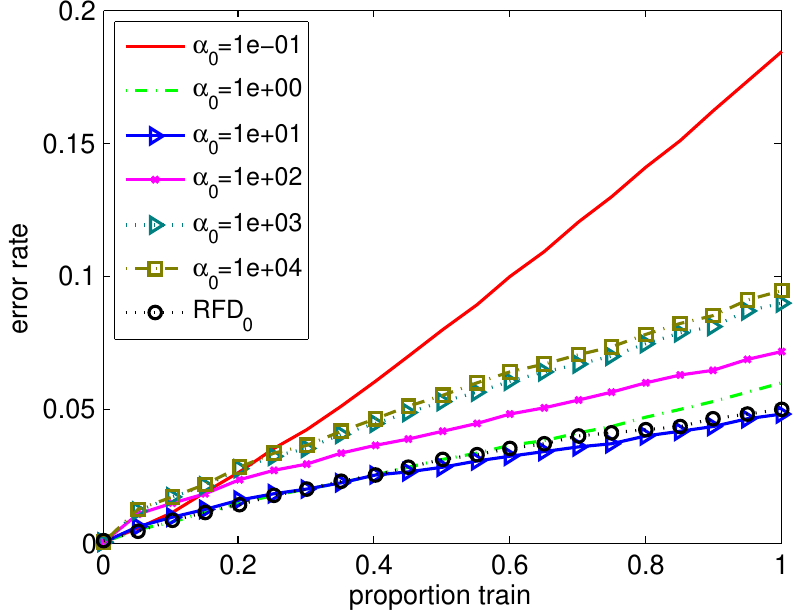} \\
        {\small (g) Oja vs RFD$_0$, $m=20$}  & {\small (h) Oja vs RFD$_0$}, $m=30$ & {\small (i) Oja vs RFD$_0$, $m=50$} \\[0.1cm]
        \includegraphics[scale=0.34]{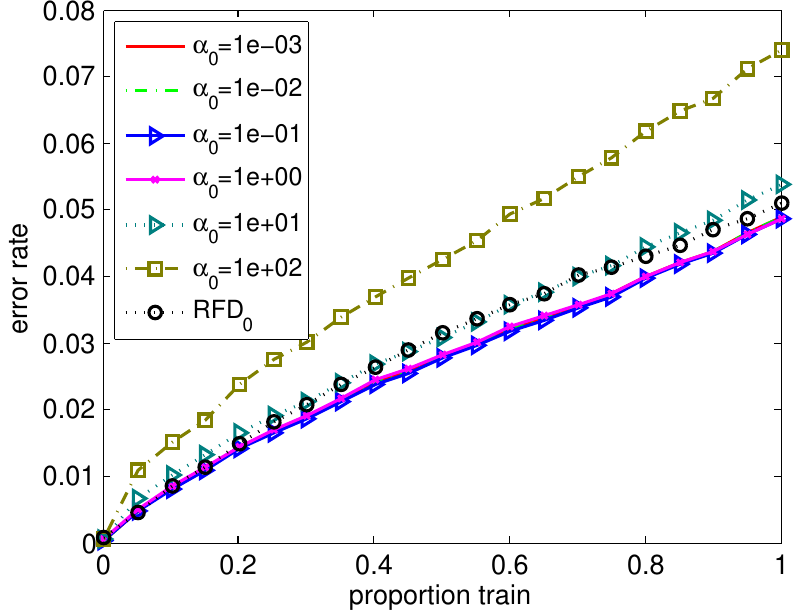} &
        \includegraphics[scale=0.34]{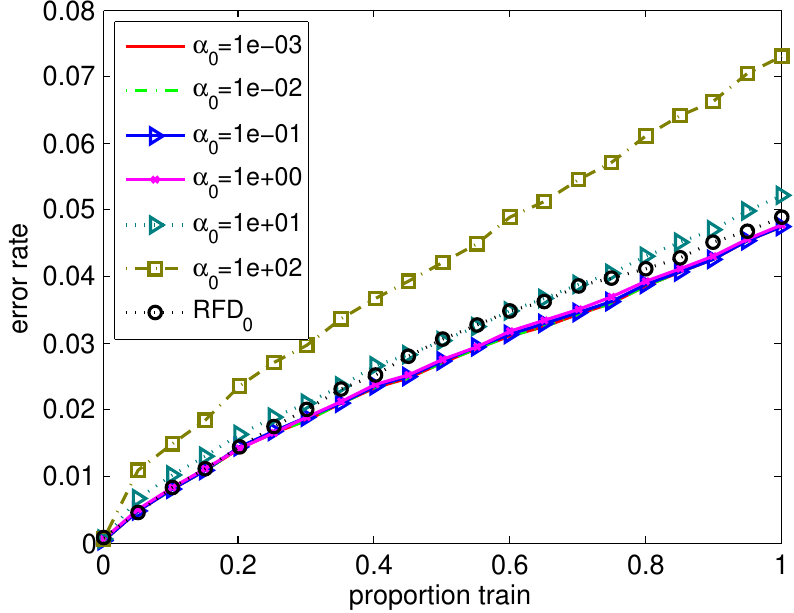} &
        \includegraphics[scale=0.34]{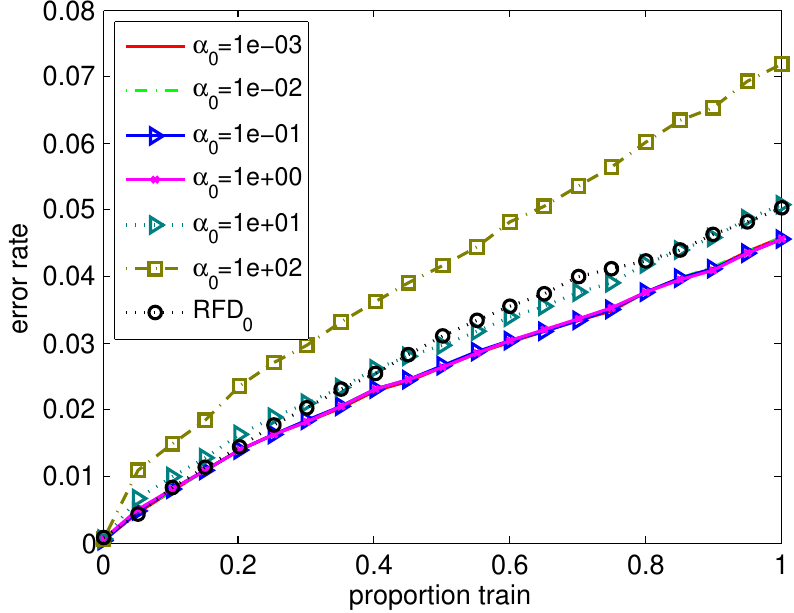} \\
        {\small (j) RFD vs RFD$_0$, $m=20$}  & {\small (k) RFD vs RFD$_0$}, $m=30$ & {\small (l) RFD vs RFD$_0$, $m=50$} \\[0.1cm]
    \end{tabular}
    \caption{Comparison of the online error rate on ``rcv1'' }
    \label{figure:train_rcv1}
\end{figure}

\begin{figure}[H]
\centering
    \begin{tabular}{ccc}
        \includegraphics[scale=0.34]{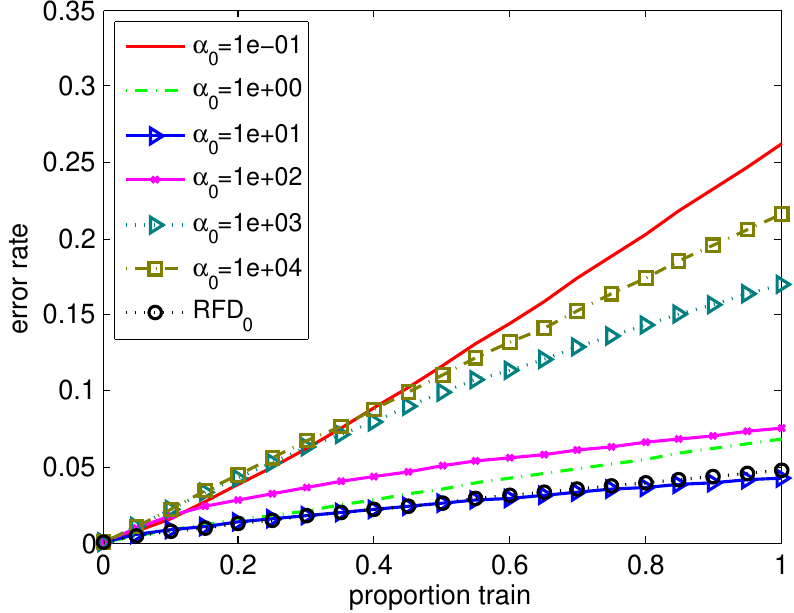} &
        \includegraphics[scale=0.34]{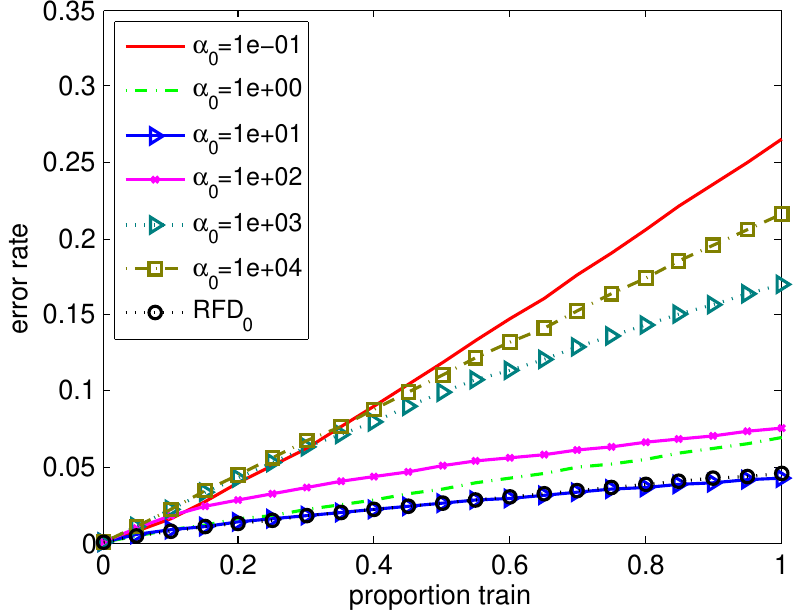} &
        \includegraphics[scale=0.34]{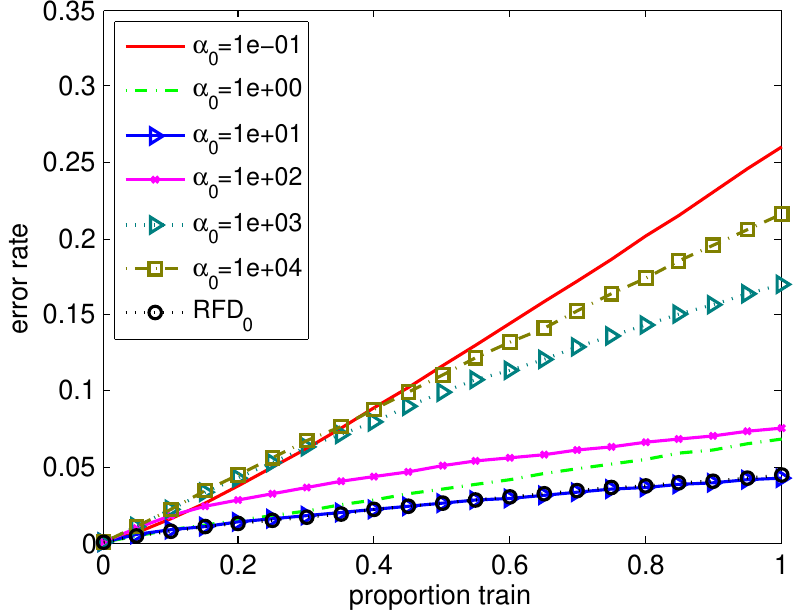} \\
        {\small (a) FD vs RFD$_0$, $m=20$}  & {\small (b) FD vs RFD$_0$}, $m=30$ & {\small (c) FD vs RFD$_0$, $m=50$} \\[0.1cm]
        \includegraphics[scale=0.34]{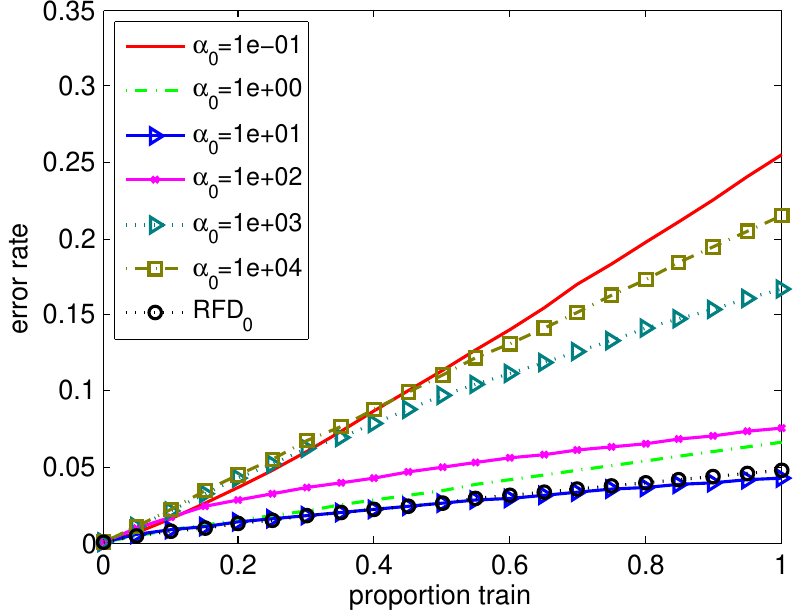} &
        \includegraphics[scale=0.34]{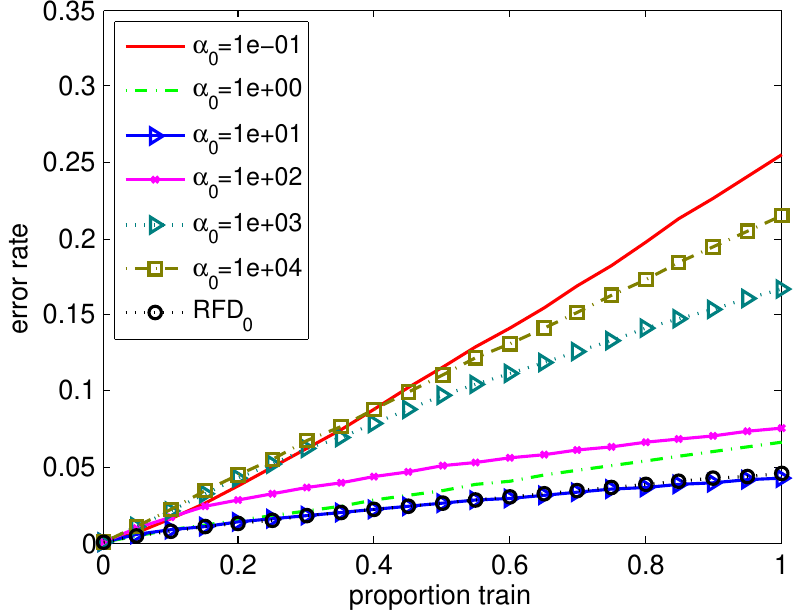} &
        \includegraphics[scale=0.34]{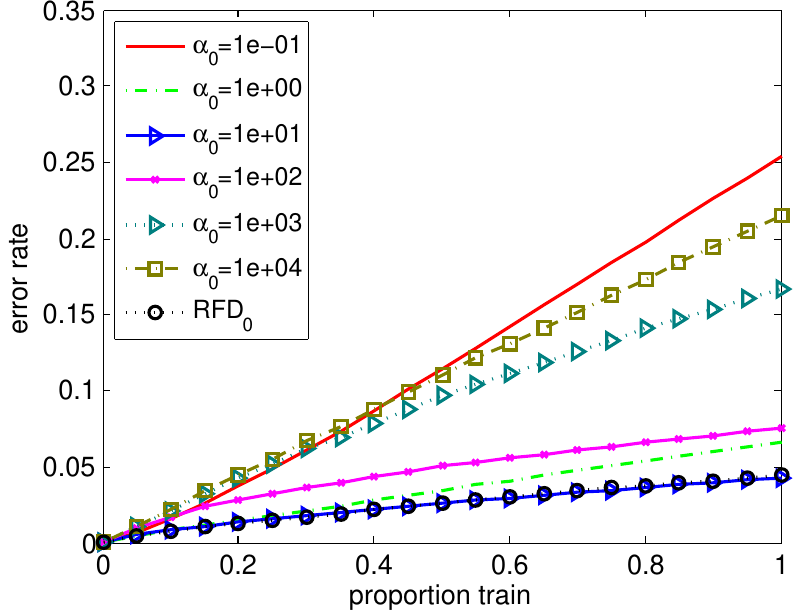} \\
        {\small (a) PFD vs RFD$_0$, $m=20$}  & {\small (b) PFD vs RFD$_0$}, $m=30$ & {\small (c) PFD vs RFD$_0$, $m=50$} \\[0.1cm]
        \includegraphics[scale=0.34]{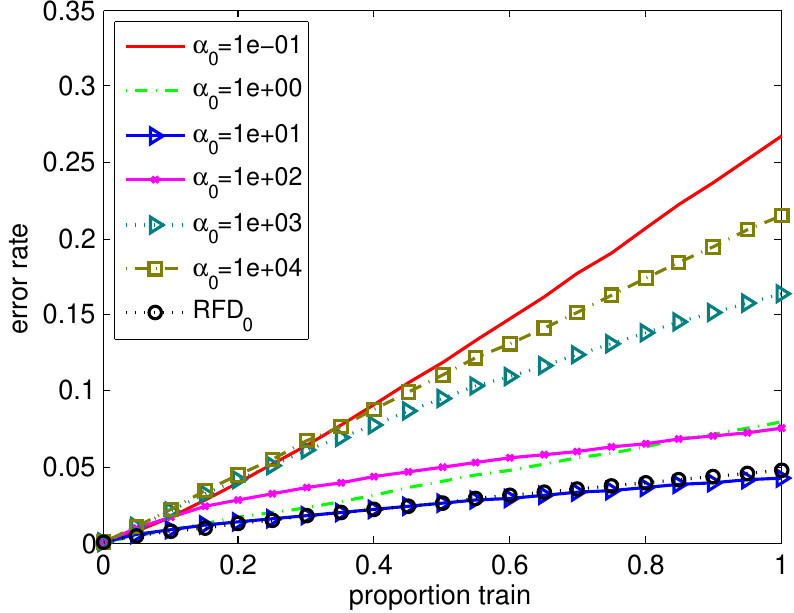} &
        \includegraphics[scale=0.34]{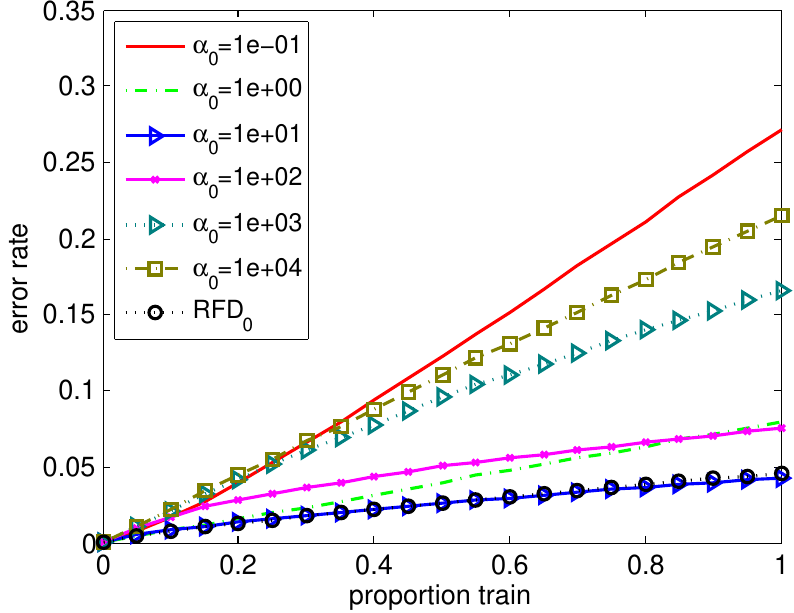} &
        \includegraphics[scale=0.34]{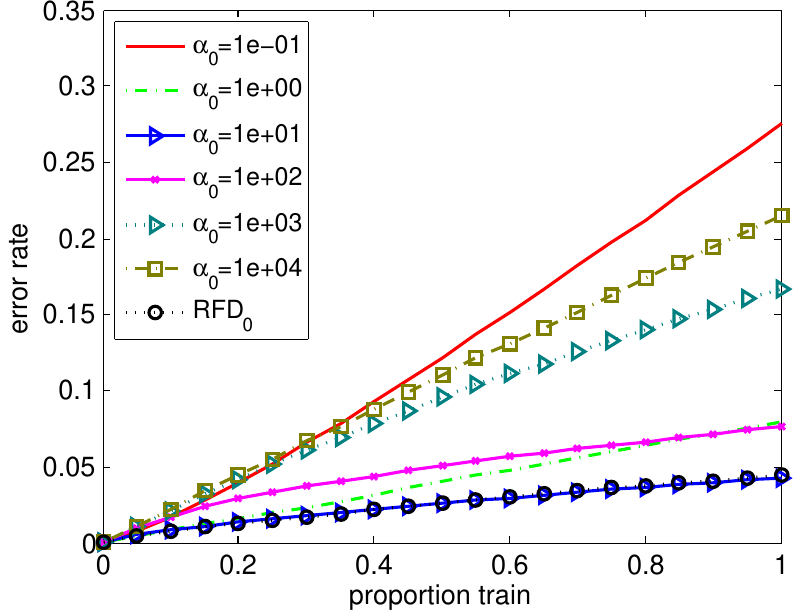} \\
        {\small (d) RP vs RFD$_0$, $m=20$}  & {\small (e) RP vs RFD$_0$}, $m=30$ & {\small (f) RP vs RFD$_0$, $m=50$} \\[0.1cm]
        \includegraphics[scale=0.34]{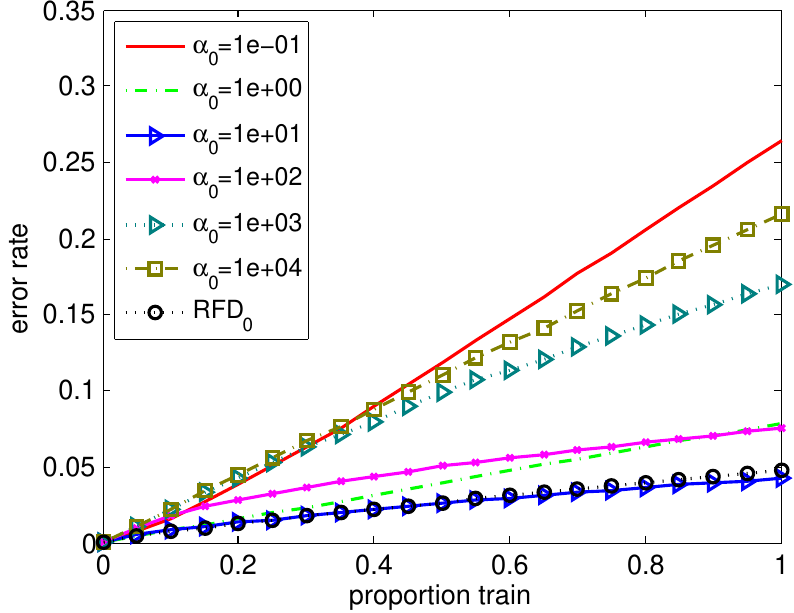} &
        \includegraphics[scale=0.34]{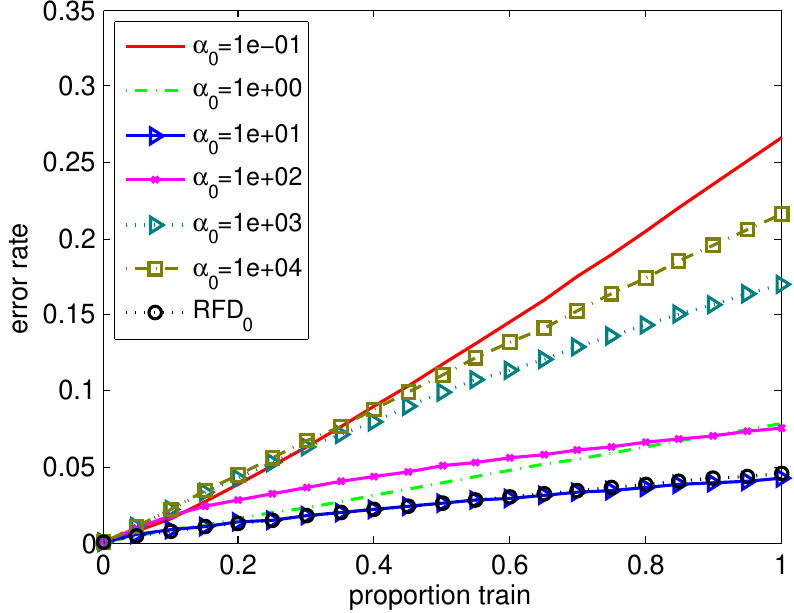} &
        \includegraphics[scale=0.34]{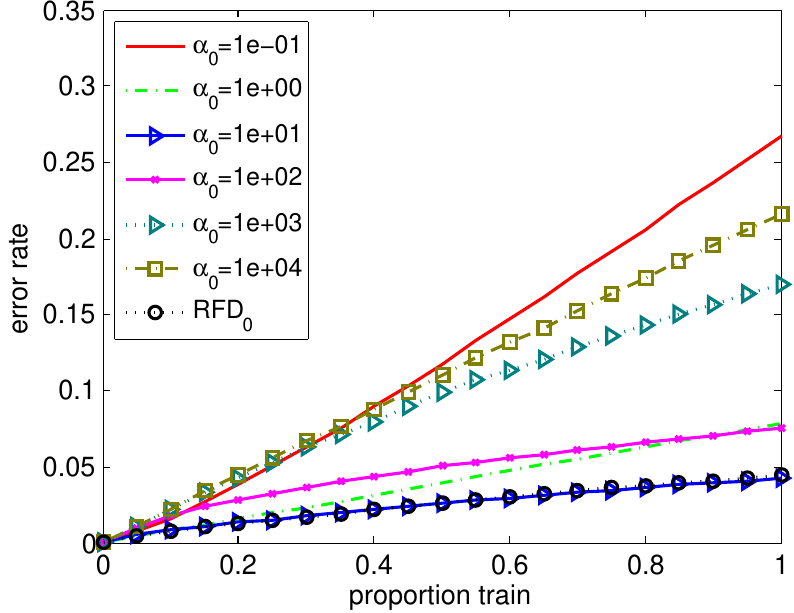} \\
        {\small (g) Oja vs RFD$_0$, $m=20$}  & {\small (h) Oja vs RFD$_0$}, $m=30$ & {\small (i) Oja vs RFD$_0$, $m=50$} \\[0.1cm]
        \includegraphics[scale=0.34]{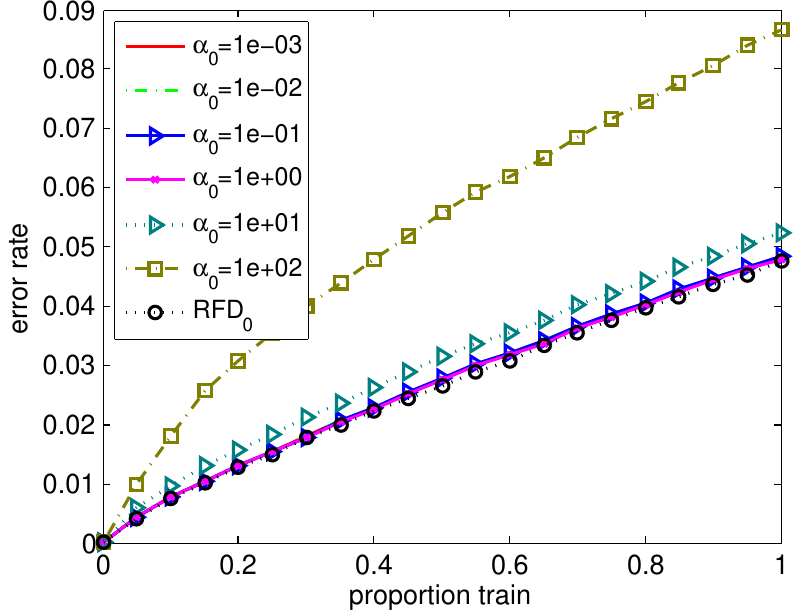} &
        \includegraphics[scale=0.34]{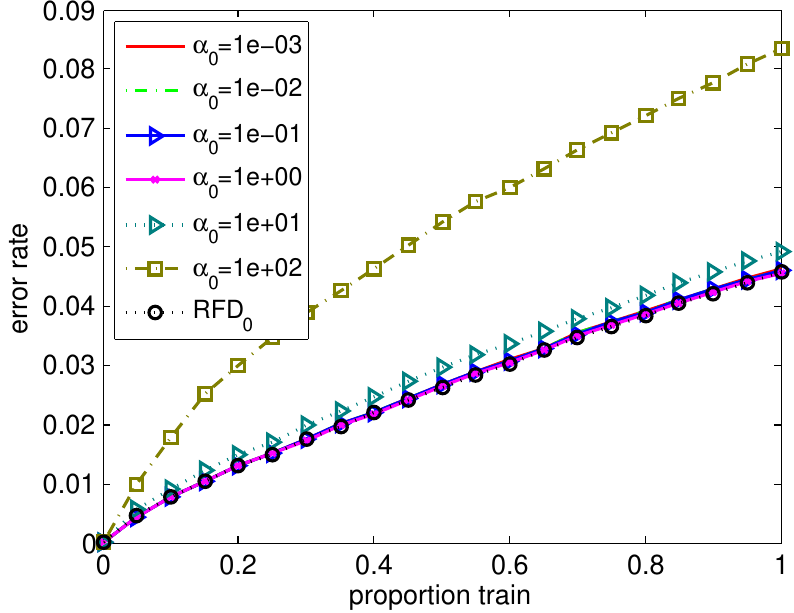} &
        \includegraphics[scale=0.34]{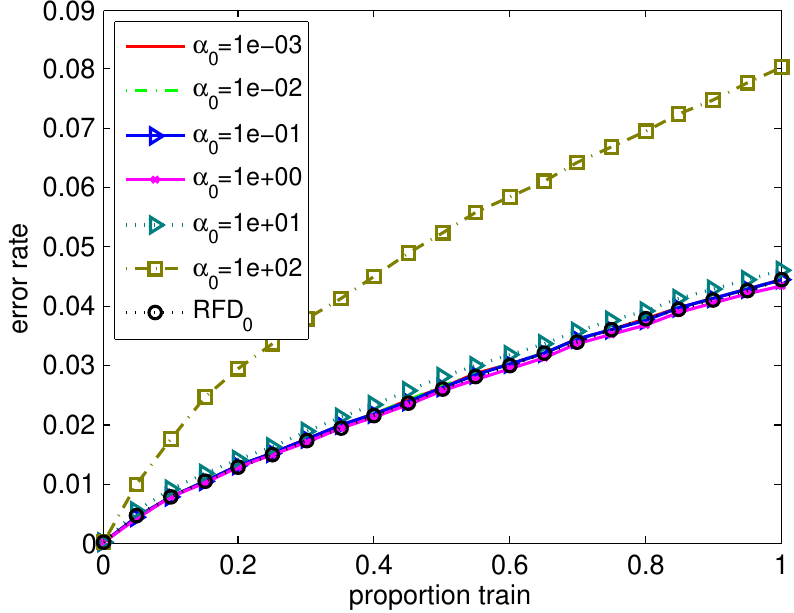} \\
        {\small (j) RFD vs RFD$_0$, $m=20$}  & {\small (k) RFD vs RFD$_0$}, $m=30$ & {\small (l) RFD vs RFD$_0$, $m=50$} \\[0.1cm]
    \end{tabular}
    \caption{Comparison of the online error rate on ``real-sim'' }
    \label{figure:train_real-sim}
\end{figure}

\begin{table}[ht]
    \centering
	\begin{tabular}{|c|c|c|c|c|c|c|}
        \hline
        Algorithms        & a9a     & gisette    & sido0 & farm-ads & rcv1 & real-sim \\\hline
        ADAGRAD	          & 83.4783 &	84.1111 &	94.0326 &	89.3805 &	95.8340 &	96.7315 \\\hline
        FULL-ON	          & 83.8264 &	96.9444 &	97.0557 &	/	&	/	&	/	\\\hline
        FD, $m=m_1$	      & 83.6524 &	96.7222 &	97.0557 &	89.7023 &	95.6200 &	96.7315 \\
        FD, $m=m_2$	      & 83.6728 &	96.7222 &	97.0557 &	89.9437 &	95.6529 &	96.7361 \\
        FD, $m=m_3$	      & 83.6728 &	96.7222 &	97.0557 &	89.9437 &	95.6529 &	96.7361 \\\hline
        PFD, $m=m_1$	  & 83.6728 &	97.0000 &	97.0294  &	90.0241 &	95.8340 &	 96.7407 \\
        PFD, $m=m_2$	  & 83.7138 &	97.0556 &	97.0820 &	89.8632 &	95.8340 &	 96.7407 \\
        PFD, $m=m_3$	  & 83.6626 &	97.0000 &	97.0557 &	 90.0241 &	95.8340 &	96.7176 \\\hline
        RP, $m=m_1$	      & 83.4374 &	96.9444 &	96.5037 &	88.8174 &	95.7188 &	96.7499 \\
        RP, $m=m_2$	      & \textbf{83.9492} &	96.2778 &	96.8980 &	89.7023 &	95.6694 &	96.7499 \\
        RP, $m=m_3$       & 83.7650 &	96.7222 &	97.0557 &	89.4610 &	95.7846 &	96.7315 \\\hline
        Oja, $m=m_1$	  & 83.6831 &	96.3889 &	96.6351 &	89.0587 &	95.7846 &	\textbf{96.7776} \\
        Oja, $m=m_2$	  & 83.6319 &	96.9444 &	96.8980 &	89.1392 &	95.7846 &	\textbf{96.7776} \\
        Oja, $m=m_3$	  & 83.5091 &	\textbf{97.1111} &	96.8980 &	89.1392 &	95.7846 &	\textbf{96.7776} \\\hline
        RFD, $m=m_1$	  & 83.6319 &	96.2222 &	96.6877 &	89.7828 &	95.8340 &	95.7666 \\
        RFD, $m=m_2$	  & 83.6831 &	96.4444 &	96.9243 &	89.8632 &	95.8834 &	96.1267 \\
        RFD, $m=m_3$	  & 83.9390 &	96.9444 &	97.0820 &	\textbf{90.3459} &	96.1139 &	96.4037 \\\hline
        RFD$_0$, $m=m_1$  &	83.2429 &	95.9444 &	96.6877 &	87.6911 &	96.3280 &	96.1636 \\
        RFD$_0$, $m=m_2$  &	83.2634 &	96.2778 &	96.9243 &	88.8174 &	96.3115 &	96.3806 \\
        RFD$_0$, $m=m_3$  &	83.2736 &	96.8889 &	\textbf{97.1083} &	88.8978 &	\textbf{96.4762} &	96.5560 \\\hline
	\end{tabular}
    \caption{We list the accuracy (\%) on test set  at the end of one pass with best choice of $\alpha_0$.
        The sketch size is set $(m_1,m_2,m_3)=(5, 10, 20)$ for
        ``a9a'', ``gisette'', ``sido0'' and $(m_1,m_2,m_3)=(20, 30, 50)$ for ``farm-ads'', ``rcv1'' and ``real-sim''.}
    \label{table:accuracy}\vspace{0.01cm}
\end{table}

\begin{table}[ht]
    \centering \small
	\begin{tabular}{|c|c|c|c|c|c|c|}
        \hline
        Algorithms        & a9a     & gisette    & sido0 & farm-ads & rcv1 & real-sim \\\hline
        ADAGRAD	        & 9.8279e-05 &	1.9677e-04 &	1.9504e-04 &	3.1976e-04 &	7.0888e-04 &	3.2917e-04 \\\hline
        FULL-ON	        & 2.3141e-04 &	2.6296e-01 &	1.6299e-01 &		/  &            /  &	         / \\\hline
        FD, $m=m_1$	& 2.8552e-04 &	5.2073e-04 &	6.0705e-04 &	1.6387e-03 &	2.0727e-02 &	6.6130e-03 \\
        FD, $m=m_2$	& 2.7276e-04 &	7.2830e-04 &	7.8723e-04 &	3.2000e-03 &	4.1148e-02 &	1.3530e-02 \\
        FD, $m=m_3$	& 3.4899e-04 &	3.4821e-03 &	1.4404e-03 &	5.1365e-03 &	7.2211e-02 &	3.0713e-02 \\\hline
        PFD, $m=m_1$	& 2.6090e-04 &	5.8330e-04 &	5.7862e-04  &	2.4524e-02 &	2.0121e-02 &	6.6757e-03 \\
        PFD, $m=m_2$	& 2.8206e-04 &	2.0234e-03 &	7.7824e-04 &	3.2372e-02 &	4.1385e-02 &	1.2965e-02 \\
        PFD, $m=m_3$	& 3.2096e-04 &	3.3901e-03 &	1.4680e-03 &	5.4829e-02 &	7.1159e-02 &	3.0677e-02 \\\hline
        RP, $m=m_1$	& 1.3597e-04 &	3.2097e-04 &	3.6333e-04 &	7.3933e-04 &	1.8736e-03 &	1.5551e-03 \\
        RP, $m=m_2$	& 2.7308e-04 &	7.2830e-04 &	7.8723e-04 &	3.2000e-03 &	2.8015e-03 &	1.8377e-03 \\
        RP, $m=m_3$	& 3.3307e-04 &	3.4821e-03 &	1.4404e-03 &	5.1365e-03 &	4.1585e-03 &	2.1537e-03 \\\hline
        Oja, $m=m_1$	& 1.5500e-04 &	6.0334e-04 &	2.9098e-04 &	1.3061e-03 &	6.7158e-03 &	5.2530e-03 \\
        Oja, $m=m_2$	& 1.6719e-04 &	7.3481e-04 &	2.4472e-04 &	2.8887e-03 &	4.1148e-02 &	1.3530e-02 \\
        Oja, $m=m_3$	& 1.6631e-04 &	3.3918e-03 &	1.3920e-03 &	4.4386e-03 &	7.2211e-02 &	3.0713e-02 \\\hline
        RFD, $m=m_1$	&  2.8549e-04 &	5.1545e-04 &	6.0296e-04 &	2.0484e-03 &	2.0557e-02 &	9.7858e-03 \\
        RFD, $m=m_2$ 	& 3.1813e-04 &	7.5013e-04 &	7.5699e-04 &	3.8129e-03 &	4.0695e-02 &	1.6527e-02 \\
        RFD, $m=m_3$	& 3.3405e-04 &	3.3495e-03 &	1.4472e-03 &	5.2458e-03 &	7.1764e-02 &	3.0175e-02 \\\hline
        RFD$_0$, $m=m_1$ & 2.7466e-04 &	6.8607e-04 &	5.9339e-04 &	2.0843e-03 &	2.0033e-02 &	9.8373e-03 \\
        RFD$_0$, $m=m_2$ & 1.9542e-04 &	8.0961e-04 &	7.6749e-04 &	3.7857e-03 &	4.0779e-02 &	1.6892e-02 \\
        RFD$_0$, $m=m_3$ & 2.3561e-04 &	1.4328e-03 &	7.6749e-04 &	5.5628e-03 &	7.3725e-02 &	3.0490e-02 \\\hline
	\end{tabular}
    \caption{We list the average iteration cost corresponding to Table \ref{table:accuracy}.}
    \label{table:time}\vspace{0.1cm}
\end{table}

\section{Conclusions}
\label{sec:concl}

In this paper we have proposed a novel sketching method robust frequent directions (RFD),
and our theoretical analysis shows that RFD is superior to FD.
We  have also studied the use of RFD in  the second order online learning algorithms.
The online learning algorithm with RFD achieves better performance than baselines.
It is worth pointing out that the application of RFD is not limited to convex online
optimization. In future work, we would like  to explore the use of  RFD in stochastic optimization and non-convex problems.

\acks{
We thank the anonymous reviewers for their helpful suggestions. Luo Luo, Cheng Chen and Zhihua Zhang have been supported by the National Natural Science Foundation of China (No. 61572017 and 11771002)  and by Beijing Municipal Commission of Science and Technology under Grant No. 181100008918005. Wu-Jun Li has been supported by the NSFC-NRF Joint Research Project (No. 61861146001) and by the NSFC (No. 61472182).}

\appendix
\section{Accelerating by Doubling Space}\label{appendix:alg}
The cost of FD (Algorithm \ref{alg:FD}) is dominated by the steps of SVD. It takes $\fO(Tm^2d)$ time by standard SVD in total.
We can accelerate FD by doubling the sketch size \citep{liberty2013simple}.
The details are shown in Algorithm \ref{alg:FFD}.
Then the SVD is called only every $m$ rows of $\A$ come and the time complexity is reduced to $\fO(Tmd)$.
Similarly, RFD can also be speeded up in this way. We demonstrate it in Algorithm \ref{alg:FRFD}.

\begin{algorithm}[ht]
    \caption{Fast Frequent Directions}
	\label{alg:FFD}
	\begin{algorithmic}[1]
    \STATE {\textbf{Input:}} $\A=[\a^{(1)},\dots,\a^{(T)}]^\top\in\BR^{T\times d}$,
            $\B^{(0)}=[\a^{(1)},\dots,\a^{(m-1)}]^\top$ \\[0.2cm]
    \STATE {\textbf{for}} $t=m,\dots,T$ {\textbf{do}} \\[0.2cm]
    \STATE \quad ${\hat\B}^\utm=\begin{bmatrix}\B^\utm \\ (\a^\ut)^\top \end{bmatrix}$ \\[0.2cm]
    \STATE \quad {\textbf{if }} $\hB^\utm$ has $2m$ rows  \\[0.2cm]
    \STATE \quad\quad Compute SVD: ${\hat\B}^\utm =\U^\utm\bSigma^\utm(\V^\utm)^\top$ \\[0.2cm]
    \STATE \quad\quad $\B^\ut = \sqrt{\big(\bSigma_{m-1}^\utm\big)^2-\big(\sigma_{m}^\utm\big)^2\I_{m-1}}\cdot \big(\V_{m-1}^\utm\big)^\top$ \\[0.2cm]
    \STATE \quad {\textbf{else}} \\[0.2cm]
    \STATE \quad\quad $\B^\ut = \hB^\utm$ \\[0.2cm]
    \STATE \quad {\textbf{end if}} \\[0.2cm]
    \STATE {\textbf{end for}} \\[0.2cm]
    \STATE {\textbf{Output:}} $\B=\B^{(T)}$
	\end{algorithmic}
\end{algorithm}

\begin{algorithm}
    \caption{Fast Robust Frequent Directions}
	\label{alg:FRFD}
	\begin{algorithmic}[1]
    \STATE {\textbf{Input:}} $\A=[\a^{(1)},\dots,\a^{(T)}]^\top\in\BR^{T\times d}$,
            $\B^{(m-1)}=[\a^{(1)},\dots,\a^{(m-1)}]^\top$, $\alpha^{(m-1)}=0$ \\[0.2cm]
    \STATE {\textbf{for}} $t=m,\dots,T$ {\textbf{do}} \\[0.2cm]
    \STATE \quad ${\hat\B}^\utm=\begin{bmatrix}\B^\utm \\ (\a^\ut)^\top \end{bmatrix}$ \\[0.2cm]
    \STATE \quad {\textbf{if }} $\hB^\utm$ has $2m$ rows  \\[0.2cm]
    \STATE \quad\quad Compute SVD: ${\hat\B}^\utm =\U^\utm\bSigma^\utm(\V^\utm)^\top$ \\[0.2cm]
    \STATE \quad\quad $\B^\ut = \sqrt{\big(\bSigma_{m-1}^\utm\big)^2-\big(\sigma_{m}^\utm\big)^2\I_{m-1}}\cdot \big(\V_{m-1}^\utm\big)^\top$
    \\[0.2cm]
    \STATE \quad\quad $\alpha^\ut = \alpha^\utm + \big(\sigma^\utm_{m}\big)^2 / 2$ \\[0.2cm]
    \STATE \quad {\textbf{else}} \\[0.2cm]
    \STATE \quad\quad $\B^\ut = \hB^\utm$ \\[0.2cm]
    \STATE \quad\quad $\alpha^\ut = \alpha^\utm$ \\[0.2cm]
    \STATE \quad {\textbf{end if}} \\[0.2cm]
    \STATE {\textbf{end for}} \\[0.2cm]
    \STATE {\textbf{Output:}} $\B=\B^{(T)}$ and $\alpha=\alpha^{(T)}$
	\end{algorithmic}
\end{algorithm}

We can apply similar strategy on RFD-SON, just as Algorithm \ref{alg:FRFD-ONS} shows.
For $\alpha^\ut>0$, the parameter $\w^\ut$ can be updated in $\fO(md)$ cost by Woodbury identity \citep{luo2016efficient}.
Suppose $\B^\ut\in\BR^{m'\times d}$, where $m'\leq 2m$. We have
\begin{align*}
    \bu^\utp &= \w^\ut  - \frac{1}{\alpha^\ut}\big(\g^\ut-(\B^\ut)^\top\M^\ut\B^\ut\g^\ut\big), \\
     \w^\utp &= \bu^\utp - \gamma^\ut \big(\x^\ut -(\B^\ut)^\top\M^\ut\B^\ut\x^\ut\big),
\end{align*}
where
\begin{align*}
    &\gamma^\ut = \frac{\tau\big((\bu^\ut)^\top\x^\ut\big)}{(\x^\ut)^\top\x^\ut-(\x^\ut)^\top(\B^\ut)^\top\M^\ut\B^\ut\x^\ut}, \\[0.1cm]
    & \M^\ut = \big(\B^\ut (\B^\ut)^\top + \alpha^\ut\I_m\big)^{-1}, \\[0.1cm]
    & \tau(z)={\rm sgn}(z)\max\{|z|-1,0\}.
\end{align*}
Let's check the cost of the above steps in detail.
The matrix $\M^\ut$ costs $\fO(m^2)$ space and
the computation of $\M^\ut(\B^\ut\g^\ut)$ or $\M^\ut(\B^\ut\x^\ut)$ takes $\fO(md+m^3)$ time given
$\B^\ut (\B^\ut)^\top + \alpha^\ut\I_{m'}$. The result of $\B^\ut (\B^\ut)^\top$ also can be obtained in $\fO(md)$ time.
If we have executed SVD at current iteration
(when $\hB^\utm$ has $2m$ rows),
$\B^\ut (\B^\ut)$ is diagonal and we can directly obtain the SVD of
\begin{align*}
 \B^\ut (\B^\ut)^\top = \big(\bSigma_{m-1}^\utm\big)^2-\big(\sigma^\utm_{m}\big)^2\I_{m-1},
\end{align*}
otherwise it can be updated incrementally in $\fO(md)$ as follows
\begin{align*}
 \B^\ut (\B^\ut)^\top = \begin{bmatrix}
    \B^\utm (\B^\utm)^\top & \sqrt{\mu_t+\eta_t} \B^\utm\g^\ut \\
    \sqrt{\mu_t+\eta_t} (\B^\utm\g^\ut)^\top & (\mu_t+\eta_t)(\g^\ut)^\top\g^\ut
    \end{bmatrix}.
\end{align*}
Since we have $m\ll d$, all above operations only require $\fO(md)$ time and space complexity in total for each iteration.

In the case of $\alpha^\ut=0$, we can iterate  $\bu^\ut$ and $\w^\ut$ by using SVD on $\B^\ut$.
Let the condensed SVD of $\B^\ut$ be $\B^\ut=\U\bSigma\V^\top$,
where $\U\in\BR^{m'\times r}$, $\bSigma\in\BR^{r\times r}$, $\V\in\BR^{r\times d}$ and $r={\rm rank}(\B^\ut)<m'$.
Then we have
\begin{align*}
\bH^\ut &=(\B^\ut)^\top\B^\ut =  \V\bSigma^{2}\V^\top, \\
(\bH^\ut)^\dag &= \V\bSigma^{-2}\V^\top.
\end{align*}
We can update $\bu^\utp$ and $ \w^\utp$ as follows
\begin{align*}
    \bu^\utp &= \w^\ut - \V^\top(\bSigma^{-2}(\V^\top\g^\ut)), \\
    \w^\utp &= \argmin_{\w\in\{\w_1^\ut,\w_2^\ut\}} \| \w - \bu^\utp \|_{\bH^\ut},
\end{align*}
where
\begin{align*}
\w_1^\ut =&\bu^\utp - \frac{\tau\big((\bu^\utp)^\top\x^\ut\big)}{(\x^\ut)^\top \V\bSigma^{-2}\V^\top \x^\ut} \V\bSigma^{-2}\V^\top\x^\ut, \\
\w_2^\ut =&\bu^\utp - \frac{\tau\big((\bu^\utp)^\top\x^\ut\big)}{(\x^\ut)^\top\x^\ut - (\x^\ut)^\top \V\bSigma^{-2}\V^\top \x^\ut}
        (\x^\ut -\V\bSigma^{-2}\V^\top\x^\ut).
\end{align*}
The iteration costs $\fO(m^2d)$ in total (dominated by the SVD of $\B^\top$),
but $\alpha^\ut=0$ only appears at a few early iterations.
Hence the average iteration complexity of RFD-SON is dominated by the case $\alpha^\ut>0$ which takes  $\fO(md)$.
Note that the algorithm is also valid without the smoothness of $f_t$.
We can replace the gradient $\nabla f_t$ with the corresponding subgradient.

\begin{algorithm}[ht]
    \caption{Fast RFD for Online Newton Step}
	\label{alg:FRFD-ONS}
	\begin{algorithmic}[1]
    \STATE {\textbf{Input:}} $\alpha^{(0)}=\alpha_0$, $m<d$, $\eta_t=\fO(1/t)$ and
            $\B^{(0)}=\bz^{m\times d}$. \\[0.1cm]
    \STATE {\textbf{for}} $t=1,\dots,T$ {\textbf{do}} \\[0.2cm]
    \STATE \quad Receive example $\x^\ut$, and loss function $f_t(\w)$ \\[0.2cm]
    \STATE \quad Predict the output of $\x^\ut$ by $\w^\ut$ and suffer loss $f_t(\w^\ut)$ \\[0.2cm]
    \STATE \quad $\g^\ut=\nabla f_t(\w^\ut)$ \\[0.2cm]
    \STATE \quad ${\hat\B}^\utm=\begin{bmatrix}\B^\utm \\ (\sqrt{\mu_t+\eta_t} \g^\ut)^\top \end{bmatrix}$ \\[0.2cm]
    \STATE \quad {\textbf{if }} $\hB^\utm$ has $2m$ rows  \\[0.2cm]
    \STATE \quad\quad Compute SVD: $\hB^\utm = \U^\utm\bSigma^\utm(\V^\utm)^\top$ \\[0.2cm]
    \STATE \quad\quad $\B^\ut = \sqrt{\big(\bSigma_{m-1}^\utm\big)^2-\big(\sigma^\utm_{m}\big)^2\I_{m-1}}\cdot (\V^\utm_{m-1})^\top$ \\[0.2cm]
    \STATE \quad\quad $\alpha^\ut = \alpha^\utm + \big(\sigma^\utm_{m}\big)^2 / 2$ \\[0.2cm]
    \STATE \quad {\textbf{else}} \\[0.2cm]
    \STATE \quad\quad $\B^\ut = \hB^\utm$ \\[0.2cm]
    \STATE \quad\quad $\alpha^\ut = \alpha^\utm$ \\[0.2cm]
    \STATE \quad {\textbf{end if}} \\[0.2cm]
    \STATE \quad $\bH^\ut = (\B^\ut)^\top\B^\ut + \alpha^\ut\I_d$ \\[0.2cm]
    \STATE \quad $\bu^\utp =  \w^\ut - (\bH^\ut)^\dag \g^\ut$  \\[0.2cm]
    \STATE \quad $\w^\utp = \argmin_{\w\in\fK_t} \| \w - \bu^\utp \|_{\bH^\ut}$ \\[0.2cm]
    \STATE {\textbf{end for}} \\[0.2cm]
	\end{algorithmic}
\end{algorithm}

\section{The Proof of Theorem \ref{thm:specMRA}}\label{appendix:specMRA}
In this section, we firstly provide several lemmas from the book ``Topics in matrix analysis'' \citep{horn1991topics},
then we prove Theorem \ref{thm:specMRA}.
The proof of Lemma \ref{lemma:uinorm1} and \ref{lemma:uinorm2} can be found in the book and
we give the proof of Lemma \ref{lemma:uinorm3} here.
\begin{lem}[Theorem 3.4.5 of \citep{horn1991topics}]\label{lemma:uinorm1}
    Let $\A,\B\in\BR^{m\times n}$ be given, and suppose $\A$, $\B$ and $\A-\B$ have decreasingly ordered singular values,
    $\sigma_1(\A) \geq \dots \geq \sigma_q(\A)$, $\sigma_1(\B) \geq \dots \geq \sigma_q(\B)$, and
    $\sigma_1(\A-\B) \geq \dots \geq \sigma_q(\A-\B)$, where $q=\min\{m,n\}$.
    Define $s_i(\A,\B)\equiv |\sigma_i(\A)-\sigma_i(\B)|$, $i=1,\dots,q$ and let
    $s_{[1]}(\A,\B) \geq \dots \geq s_{[i]}(\A,\B)$ denote a decreasingly ordered rearrangement of the values $s_i(\A,\B)$.
    Then
    \begin{align*}
        \sum_{i=1}^k s_{[i]}(\A,\B) \leq \sum_{i=1}^k \sigma_{i}(\A-\B) \text{  for } k=1,\dots,q.
    \end{align*}
\end{lem}

\begin{lem}[Corollary 3.5.9 of \citep{horn1991topics}]\label{lemma:uinorm2}
    Let $\A,\B\in\BR^{m\times n}$ be given, and let $q=\min\{m,n\}$. The following are equivalent
    \begin{enumerate}
        \item $\uin \A \uin \leq \uin \B \uin$ for every unitarily invariant norm $\uin \cdot  \uin$ on $\BR^{m\times n}$.
        \item $N_k(\A) \leq N_k(\B)$ for $k=1,\dots,q$ where $N_k(\X)\equiv\sum_{i=1}^k\sigma_k(\X)$ denotes Ky Fan $k$-norm.
    \end{enumerate}
\end{lem}
\begin{lem}[Page 215 of \citep{horn1991topics}]\label{lemma:uinorm3}
    Let $\A,\B\in\BR^{m\times n}$ be given, and let $q=\min\{m,n\}$.
    Define the diagonal matrix $\bSigma(\A)=[\sigma_{ij}]\in\BR^{m\times n}$ by $\sigma_{ii}=\sigma_i(\A)$,
    all other $\sigma_{ij}=0$, where $\sigma_1(\A)\geq,\dots,\geq\sigma_q(\A)$ are the decreasingly ordered singular values of $\A$.
    We define $\bSigma(\B)$ similarly.
    Then we have $\uin\A-\B\uin \geq \uin\bSigma(\A)-\bSigma(\B)\uin$ for every unitarily invariant norm $\uin \cdot \uin$.
\end{lem}
\begin{proof}
    Using the notation of Lemma \ref{lemma:uinorm1} and \ref{lemma:uinorm2}, matrices $\A-\B$ and $\bSigma(\A)-\bSigma(\B)$
    have the decreasingly ordered singular values $\sigma_1(\A-\B) \geq \dots \geq \sigma_q(\A-\B)$ and
    $s_{[1]}(\A,\B) \geq \dots \geq s_{[q]}(\A,\B)$.
    Then we have
    \begin{align}
             & N_k(\A-\B) = \sum_{i=1}^k \sigma_i(\A-\B)
        \geq  \sum_{i=1}^k s_{[i]}(\A,\B)
            =  N_k(\bSigma(\A)-\bSigma(\B)), \label{ieq:uinproof}
    \end{align}
    where the inequality is obtained by Lemma \ref{lemma:uinorm1}.
    The Lemma \ref{lemma:uinorm2} implies (\ref{ieq:uinproof}) is equivalent to $\uin\A-\B\uin \geq \uin\bSigma(\A)-\bSigma(\B)\uin$
    for every unitarily invariant norm $\uin \cdot \uin$.
\end{proof}
Then we give the proof of Theorem \ref{thm:specMRA} as follows:
\begin{proof}
    Using the notation in above lemmas, we can bound the objective function as follows
    \begin{align*}
          \big\| \M-\C\C^\top-\delta\I_d  \big\|_2
     & \geq \big\| \bSigma(\M)-\bSigma(\C\C^\top+\delta\I_d) \big\|_2 \\[0.25cm]
       & = \max_{i\in\{1,\dots,d\}} \big| \sigma_i(\M)-\sigma_i(\C\C^\top)-\delta \big| \\[0.25cm]
     & \geq \max_{i\in\{k+1,\dots,d\}} \big| \sigma_i(\M)-\sigma_i(\C\C^\top)-\delta \big| \\[0.25cm]
       & = \max_{i\in\{k+1,\dots,d\}} \big| \sigma_i(\M)-\delta \big| \\[0.25cm]
     & \geq \max_{i\in\{k+1,\dots,d\}} \big| \sigma_i(\M)-{\hat\delta} \big|.
    \end{align*}

    The first inequality is obtained by Lemma \ref{lemma:uinorm3} since the spectral norm is unitarily invariant,
    and the second inequality is the property of maximization operator.
    The last inequality can be checked easily by the property of max operation and the equivalence of SVD and
    eigenvector decomposition for positive semi-definite matrix.
    The first equality is based on the definition of spectral norm.
    The second equality holds due to the fact $\rk(\C\C^\top)\leq k$ which leads $\sigma_i(\C\C^\top)=0$ for any $i>k$.
    Note that all above equalities occur for $\C={\hat\C}=\U_k(\bSigma_k-{\xi}\I_k)^{1/2}\Q$, $\delta={\hat\delta}$ and $\xi\in[\sigma_d,\sigma_{k+1}]$.
    Hence we prove the optimality of $({\hat\C}, {\hat\delta})$.

    The approximation error of rank-$k$ SVD corresponds to the objective of (\ref{prob:thm1})
    by taking $\C=\U_k(\bSigma_k)^{1/2}$ and $\delta=0$, which is impossible to be smaller than the minimum.
    It is easy to verify we have ${\hat \C}=\U_k(\bSigma_k)^{1/2}$ and $\hat\delta=0$ if and only of $\rk(\M)\leq k$.
\end{proof}

Theorem \ref{thm:specMRA} means the choice of $\xi$ in the solution of problem (\ref{prob:RFD}) is not unique,
but taking $\xi=\sigma_{k+1}$ minimizes the condition number of $\hat\C\hat\C^\top+{\hat\delta}\I_d$.
Hence, we use $\xi=\sigma_{k+1}$ in the derivation of RFD.
We also demonstrate similar result with respect to Frobenius norm in Corollary \ref{cor:froMRA}.
This analysis includes the global optimality of the problem,
while \citet{zhang2014matrix}'s analysis only prove the solution is locally optimal.

\begin{cor}\label{cor:froMRA}
    Using the same notation in Theorem \ref{thm:specMRA},
    the pair $({\tilde\C}, {\tilde\delta})$ defined as
    \begin{align*}
        \tilde\C = \U_k(\bSigma_k-{\tilde\delta}\I_k)^{1/2}\V
        \quad \text{and} \quad
       {\tilde\delta} = \frac{1}{d-k}\sum_{i=j+1}^d \sigma_{i}
    \end{align*}
    is the global minimizer of
    \begin{align*}
    \min_{\C\in\BR^{d\times k}, \delta\in\BR} \| \M-\C\C^\top-\delta\I_d  \|_F^2, \\
    \end{align*}
    where $\V$ is an arbitrary $k\times k$ orthogonal matrix.
\end{cor}
\begin{proof}
    We have the result similar to Theorem \ref{thm:specMRA}.
    \begin{align*}
          \| \M-\C\C^\top-\delta\I_d  \|_F^2
     & \geq \| \bSigma(\M)-\bSigma(\C\C^\top+\delta\I_d) \|_F^2 \\[0.25cm]
       & = \sum_{i=1}^d (\sigma_i(\M)-\sigma_i(\C\C^\top)-\delta)^2 \\[0.25cm]
     & \geq \sum_{i=k+1}^d (\sigma_i(\M)-\sigma_i(\C\C^\top)-\delta)^2 \\[0.25cm]
       & = \sum_{i=k+1}^d (\sigma_i(\M)-\delta)^2 \\[0.25cm]
     & \geq \sum_{i=k+1}^d (\sigma_i(\M)-{\tilde\delta})^2
    \end{align*}
    The first four steps are similar to the ones of Theorem \ref{thm:specMRA},
    but replace the spectral norm and absolute operator with Frobenius norm and square function.
    The last step comes from the property of the mean value.

    We can check that all above equalities occur for $\C={\tilde\C}$ and $\delta={\tilde\delta}$,
    which completes the proof.
\end{proof}

\section{The Proof of Theorem \ref{thm:optRFD}} \label{appendix:optRFD}

\begin{proof}
The Algorithm \ref{alg:RFD} implies the singular values of $(\hB^\utm)^\top\hB^\utm + \alpha^\utm\I$ are
\begin{align*}
(\sigma^\utm_1)^2 + \alpha^\utm \geq \dots \geq (\sigma^\utm_m)^2 + \alpha^\utm \geq \alpha^\utm = \dots = \alpha^\utm.
\end{align*}
Then we can use Theorem \ref{thm:specMRA} by taking
\begin{align*}
        \M = & (\hB^\utm)^\top\hB^\utm + \alpha^\utm\I, \\[0.1cm]
        k = & m - 1, \\
        \xi = & \sigma_{k+1} = (\sigma^\utm_m)^2 + \alpha^\utm \\
    \hat\C = & \V^\utm_{m-1}\sqrt{\big(\bSigma_{m-1}^\utm\big)^2-\big(\sigma^\utm_{m}\big)^2\I_{m-1}}=(\B^\ut)^\top, \\[0.2cm]
\hat\delta = & [(\sigma^\utm_m)^2 + \alpha^\utm + \alpha^\utm]/2 = \alpha^\ut,
\end{align*}
which just means that $(\B^\ut, \alpha^\ut)$ is the minimizer of the problem in this theorem.
\end{proof}

\section{The Proof of Theorem \ref{thm:RFDbound}}

The algorithms of FD and RFD share the same $\B^\ut$ and we have Lemma \ref{property:FD}
\citep{DBLP:journals/siamcomp/GhashamiLPW16} as follows.
\begin{lem}\label{property:FD}
For any $k < m$ and using the notation of Algorithm \ref{alg:FD} or Algorithm \ref{alg:RFD}, we have
\begin{align}
& \A^\top\A - \B^\top\B \succeq \bz,  \label{bound:psd} \\
& \sum_{t=1}^{T-1} (\sigma_m^\ut)^2 \leq \frac{1}{m-k} \| \A - [\A]_k \|_F^2 \label{bound:sv}.
\end{align}
\end{lem}
Then we prove the Theorem \ref{thm:RFDbound} based on Lemma \ref{property:FD}.
\begin{proof}
    Define $(\B^{(0)})^\top\B^{(0)}=\bz^{d\times d}$, then we can derive the error bound as follows
    \begin{align*}
      & \big\|\A^\top\A-(\B^\top\B+\alpha\I_d) \big\|_2 \\
    = & \Bigg\|\sum_{t=1}^T \Big[(\a^\ut)^\top\a^\ut - (\B^\ut)^\top\B^\ut
     + (\B^\utm)^\top\B^\utm - \frac{1}{2}(\sigma^\utm_{m})^2\I_d \Big]\Bigg\|_2 \\
    \leq& \sum_{t=1}^T \Big\| (\a^\ut)^\top\a^\ut - (\B^\ut)^\top\B^\ut
      + (\B^\utm)^\top\B^\utm - \frac{1}{2}(\sigma^\utm_{m})^2\I_d \Big\|_2  \\
    =& \sum_{t=1}^T \Big\| \V_{m-1}^\utm(\bSigma_{m-1}^\utm)^2(\V_{m-1}^\utm)^\top
      - \V_{m-1}^\utm[(\bSigma_{m-1}^\utm)^2-(\sigma^\utm_m)^2\I_d](\V_{m-1}^\utm)^\top
       - \frac{1}{2}(\sigma^\utm_{m})^2\I_d \Big\|_2  \\
    =& \sum_{t=1}^T \Big\| (\sigma^\utm_m)^2\V_{m}^\utm(\V_{m}^\utm)^\top - \frac{1}{2}(\sigma^\utm_{m})^2\I_d \Big\|_2  \\
    =& \frac{1}{2} \sum_{t=1}^{T-1} (\sigma^\ut_{m})^2 \\
    \leq& \frac{1}{2(m-k)}\| \A - [\A]_k \|_F^2.
    \end{align*}
    The first three equalities are direct from the procedure of the algorithm, and the last one is based on the fact that $\V^\utm$ is column orthonormal.
    The first inequality comes from the triangle inequality of spectral norm.
    The last one can be obtained by the result (\ref{bound:sv}) of Lemma \ref{property:FD}.
\end{proof}

We also have similar error bound for fast RFD with doubling space.
We first rewrite Algorithm \ref{alg:FRFD} as the block formulation.
Consider the procedure of Algorithm \ref{alg:FRFD},
we suppose that matrix $\hB^\utm$ has $2m$ rows at round $t=p_1,p_2,\dots,p_{T'-1}$, where
$1<p_1 < p_2 < \dots < p_{T'-1} \leq T$.
Letting $p_0=0$ and $p_{T'}=T$, we can partition matrix $\A$ into $T'$ blocks
\begin{align}
    \A=\begin{bmatrix}
        \A^{(1)}, \A^{(2)}, \cdots,\A^{(T')}
    \end{bmatrix}^\top, \label{block:A}
\end{align}
where
\begin{align}
    \A^{(t')} =
    \begin{bmatrix}
        \a^{(p_{t'-1}+1)}, \a^{(p_{t'-1}+2)}, \cdots, \a^{(p_{t'})}
    \end{bmatrix}^\top \text{ for } t'=1,2,\dots,T'. \label{block:At}
\end{align}
Based on the notation of (\ref{block:A}) and (\ref{block:At}),
we can rewrite Algorithm \ref{alg:FRFD} as Algorithm \ref{alg:FRFDB}.
It is obvious that two algorithms have the same output results.
We present Lemma \ref{property:FDB} which extends Lemma \ref{property:FD} to block version and
establishes the error bound for Algorithm \ref{alg:FRFDB} in Corollary \ref{cor:RFD}.
\begin{algorithm}
    \caption{Fast Robust Frequent Directions (Block Formulation) }
	\label{alg:FRFDB}
	\begin{algorithmic}[1]
    \STATE {\textbf{Input:}} $\A=[\A^{(1)},\dots,\A^{(T')}]^\top\in\BR^{T\times d}$,
            $\B^{(1)}=(\A^{(1)})^\top$, $\alpha^{(1)}=0$ \\[0.2cm]
    \STATE {\textbf{for}} $t'=2,\dots,T'$ {\textbf{do}} \\[0.2cm]
    \STATE \quad ${\hat\B}^\utm=\begin{bmatrix}\B^\uttm \\ (\A^\utt)^\top \end{bmatrix}$ \\[0.2cm]
    \STATE \quad Compute SVD: ${\hat\B}^\uttm =\U^\uttm\bSigma^\uttm(\V^\uttm)^\top$ \\[0.2cm]
    \STATE \quad $\B^\utt = \sqrt{\big(\bSigma_{m-1}^\uttm\big)^2-\big(\sigma_{m}^\uttm\big)^2\I_{m-1}}
        \cdot \big(\V_{m-1}^\uttm\big)^\top$ \\[0.2cm]
    \STATE \quad $\alpha^\utt = \alpha^\uttm + \big(\sigma^\uttm_{m}\big)^2 / 2$ \\[0.2cm]
    \STATE {\textbf{end for}} \\[0.2cm]
    \STATE {\textbf{Output:}} $\B=\B^{(T')}$ and $\alpha=\alpha^{(T')}$
	\end{algorithmic}
\end{algorithm}

\begin{lem}\label{property:FDB}
For any $k < m$ and using the notation of Algorithm \ref{alg:FRFDB}, we have
\begin{align*}
  \sum_{t'=1}^{T'-1} (\sigma_m^\utt)^2 \leq \frac{1}{m-k} \| \A - [\A]_k \|_F^2 .
\end{align*}
\end{lem}
\begin{proof}
We let $\sigma_m^{(T')}=0$.
For any unit vector $\x\in\BR^d$, the procedure of Algorithm \ref{alg:FRFDB} implies
\begin{align}
 &   \|\A\x\|^2 - \|\B\x\|^2 \nonumber\\
=& \sum_{t'=1}^{T'} \Big( \|\A^\utt\x\|^2 + \|\B^\uttm\x\|^2 - \|\B^\utt\x\|^2 \Big) \nonumber\\
=& \sum_{t'=1}^{T'} \x^\top \Big( (\A^\utt)^\top\A^\utt + (\B^\uttm)^\top(\B^\uttm) - (\B^\utt)^\top\B^\utt  \Big) \x \nonumber\\
=& \sum_{t'=1}^{T'} \x^\top \Big( (\hB^\utt)^\top\hB^\utt - (\B^\utt)^\top\B^\utt  \Big) \x \nonumber\\
\leq& \sum_{t'=1}^{T'} \big(\sigma_m^\utt\big)^2 \label{bound:BRFDB1}.
\end{align}
Using the property of Frobenius norm, we have
\begin{align}
  \|\hB^\uttm\|_F^2
=& \|\bSigma^\uttm\|_F^2 \nonumber\\
\geq& \Big\|\sqrt{\big(\bSigma^\uttm_{m}\big)^2-\big(\sigma_m^\uttm\big)^2\I_{m}}\Big\|_F^2 + m(\sigma_m^\uttm)^2 \nonumber\\
=& \|\B^\utt\|_F^2 + m(\sigma_m^\uttm)^2.\label{bound:BRFDB2}
\end{align}
The term $\|\A\|_F^2$ satisfies
\begin{align}
  \|\A\|_F^2
=& \sum_{t'=1}^{T'}  \| \A^\utt \|_F^2 \nonumber\\
=& \sum_{t'=1}^{T'}  \Big( \| \hB^\uttm \|_F^2 - \| \B^\uttm \|_F^2 \Big) \nonumber\\
\geq& \sum_{t'=1}^{T'}  \Big( \|\B^\utt\|_F^2 + m(\sigma_m^\uttm)^2 - \| \B^\uttm \|_F^2 \Big) \nonumber\\
=& \|\B\|_F^2 + m\sum_{t'=1}^{T'}\big(\sigma_m^\utt\big)^2, \label{bound:BRFDB3}
\end{align}
where the inequality is due to (\ref{bound:BRFDB2}).

Let $\y_i$ be the singular vectors of $\A$ with respect to $\sigma_i(\A)$.
Then we have
\begin{align}
      m\sum_{t'=1}^{T'}(\sigma_m^\uttm)^2
\leq & \|\A\|_F^2 - \|\B\|_F^2 \nonumber\\
   = & \sum_{i=1}^k \|\A\y_i\|_F^2 + \sum_{i=k+1}^d \|\A\y_i\|_F^2 - \|\B\|_F^2 \nonumber\\
   = & \sum_{i=1}^k \|\A\y_i\|_F^2 + \|\A-[\A]_k\|_F^2 - \|\B\|_F^2 \nonumber\\
\leq &  \|\A-[\A]_k\|_F^2 + \sum_{i=1}^k \Big( \|\A\y_i\|_F^2 - \|\B\y_i\|_F^2 \Big)  \nonumber\\
\leq &  \|\A-[\A]_k\|_F^2 + k \sum_{t'=1}^{T'} \big(\sigma_m^\utt\big)^2, \label{bound:BRFDB4}
\end{align}
where the first inequality comes from (\ref{bound:BRFDB3}),
the second inequality is based on the fact $\sum_{i=1}^k\|\B\y_i\|^2 \leq \|\B\|_F^2$,
and the last one comes from (\ref{bound:BRFDB1}). We can obtain the result of this lemma by (\ref{bound:BRFDB4}) directly.
\end{proof}

\begin{cor}\label{cor:RFD}
    For any $k<m$ and using the notation of Algorithm \ref{alg:FRFDB}, we have
    \begin{align*}
        \big\|\A^\top\A-(\B^\top\B+\alpha\I_d) \big\|_2 \leq \frac{1}{2(m-k)}\| \A - [\A]_k \|_F^2,
    \end{align*}
    where $[\A]_k$ is the best rank-$k$ approximation to $\A$ in both the Frobenius and spectral norms.
\end{cor}
\begin{proof}
    Define $(\B^{(0)})^\top\B^{(0)}=\bz^{d\times d}$, then we can derive the error bound as follows
    {\small\begin{align*}
      & \big\|\A^\top\A-(\B^\top\B+\alpha\I_d) \big\|_2 \\
    = & \Bigg\|\sum_{t'=1}^{T'} \Big[(\A^\utt)^\top\A^\utt - (\B^\utt)^\top\B^\utt
     + (\B^\uttm)^\top\B^\uttm - \frac{1}{2}(\sigma^\uttm_{m})^2\I_d \Big]\Bigg\|_2 \\
    \leq& \sum_{t'=1}^{T'} \Big\| (\A^\utt)^\top\A^\utt - (\B^\utt)^\top\B^\utt
      + (\B^\uttm)^\top\B^\uttm - \frac{1}{2}(\sigma^\uttm_{m})^2\I_d \Big\|_2  \\
    =& \sum_{t'=1}^{T'} \Big\| \V_{m-1}^\uttm(\bSigma_{m-1}^\uttm)^2(\V_{m-1}^\uttm)^\top
      - \V_{m-1}^\uttm[(\bSigma_{m-1}^\uttm)^2-(\sigma^\uttm_m)^2\I_d](\V_{m-1}^\uttm)^\top
       - \frac{1}{2}(\sigma^\uttm_{m})^2\I_d \Big\|_2  \\
    =& \sum_{t'=1}^{T'} \Big\| (\sigma^\uttm_m)^2\V_{m}^\uttm(\V_{m}^\uttm)^\top - \frac{1}{2}(\sigma^\uttm_{m})^2\I_d \Big\|_2  \\
    =& \frac{1}{2} \sum_{t'=1}^{T'-1} (\sigma^\utt_{m})^2 \\
    \leq& \frac{1}{2(m-k)}\| \A - [\A]_k \|_F^2.
    \end{align*}}
    The last inequality is based on Lemma \ref{property:FDB} and
    other steps are similar to the proof of Theorem \ref{thm:RFDbound}.
\end{proof}

\section{The Proof of Theorem \ref{thm:condition}}
\begin{proof}
    We can compare $\kappa(\M_{\rm RFD})$ and  $\kappa(\M_{\rm FD})$ by the fact $\alpha \geq \alpha^{(0)}$ as follows
    \begin{align*}
          \kappa(\M_{\rm RFD})
        & = \frac{\sigma_{\max}(\B^\top\B)+\alpha}{\alpha} \\
     & \leq \frac{\sigma_{\max}(\B^\top\B)+\alpha_0}{\alpha_0} \\
       & = \kappa(\M_{\rm FD}).
    \end{align*}
    The other inequality can be derived as
    \begin{align*}
          \kappa(\M_{\rm RFD})
      & =  \frac{\sigma_{\max}(\B^\top\B)+\alpha}{\alpha} \\
     & \leq \frac{\sigma_{\max}(\A^\top\A)+\alpha}{\alpha} \\
     & \leq \frac{\sigma_{\max}(\A^\top\A)+\alpha_0}{\alpha_0} \\
       & = \kappa(\M),
    \end{align*}
     where the first inequality comes from (\ref{bound:psd}) of Lemma \ref{property:FD} and the others are easy to obtain.
\end{proof}

\section{The Greedy Low-rank Approximation}\label{appendix:example}
We present the greedy low-rank approximation
\citep{brand2002incremental,hall1998incremental,levey2000sequential,ross2008incremental}  as Algorithm \ref{alg:greedy}.
The algorithm does not work in general although $(\B'^\ut)^\top\B'^\ut$ is the best low-rank approximation to $(\hB^\ut)^\top\hB^\ut$.

\begin{algorithm}
    \caption{Greedy Low-rank Approximation}
	\label{alg:greedy}
	\begin{algorithmic}[1]
    \STATE {\textbf{Input:}} $\A=[\a^{(1)},\dots,\a^{(T)}]^\top\in\BR^{T\times d}$,
            $\B'^{(m-1)}=[\a^{(1)},\dots,\a^{(m-1)}]^\top$ \\[0.1cm]
    \STATE {\textbf{for}} $t=m,\dots,T$ {\textbf{do}} \\[0.2cm]
    \STATE \quad ${\hat\B}^\utm=\begin{bmatrix}\B'^\utm \\ (\a^\ut)^\top \end{bmatrix}$ \\[0.2cm]
    \STATE \quad Compute SVD: ${\hat\B}^\utm =\U^\utm\bSigma^\utm(\V^\utm)^\top$ \\[0.2cm]
    \STATE \quad $\B'^\ut = \bSigma_{m-1}^\utm \big(\V_{m-1}^\utm\big)^\top$ \\[0.2cm]
    \STATE {\textbf{end for}} \\[0.1cm]
    \STATE {\textbf{Output:}} $\B'=\B'^{(T)}$
	\end{algorithmic}
\end{algorithm}

We provide an example to show the failure of this method.
We define $\tA=[\tA_1^\top, \tA_2^\top]^\top\in\BR^{(m-s+1)\times d}$,
where $\tA_1\in\BR^{(m-1)\times d}$, $\tA_2\in\BR^{s\times d}$ and $m\ll s$.
Suppose that the smallest singular value of $\tA_1$ is $\lambda$, and
each row of $\tA_2$ is $\a\in\BR^d$ that satisfies $\|\a\|=\lambda-\epsilon$ and $\tA_1^\top\a=\bz$,
where $\epsilon$ is a very small positive number.
Since $s$ is much larger than $m$, a good approximation to $\tA^\top\tA$ is dominated by $\tA_2^\top\tA_2$.
If we use Algorithm \ref{alg:greedy} with $\A=\tA$, any row of $\tA_2$ will be neglected
because the $m$-th singular value of $\hB^\utm$ is $\|\a\|=\lambda-\epsilon<\lambda$,
which leads to the fact that output is  $\B'=\tA_1$. Apparently, $\tA_1^\top\tA_1$ is not a good approximation to $\A^\top\A$.
Hence the shrinking of FD or RFD is necessary.
In this example, it reduces the impact of $\tA_1$ and let $\tA_2$ be involved in final result.

Besides above discussion, \citet{DBLP:journals/tkde/DesaiGP16} has shown that the greedy algorithm
is much worse than FD based methods on data sets ``Adversarial''  and ``ConnectUS''.

\section{The Proof of Theorem \ref{thm:regret}}
Lemma \ref{lemma:regret} shows the general regret bound for any choice of $\bH^\ut\succ\bz$ in update (\ref{update:ONS})
\begin{align*}
     \bu^\utp &=  \w^\ut-\beta_t(\bH^\ut)^{-1} \g^\ut, \nonumber \\
     \w^\utp &= \argmin_{\w\in\fK_{t+1}} \| \w - \bu^\utp \|_{\bH^\ut}.
\end{align*}
\begin{lem}[Proposition 1 of \cite{luo2016efficient}]\label{lemma:regret}
    For any sequence of positive definite matrices $\bH^\ut$ and sequences of losses satisfying
    Assumption \ref{asm:bound} and \ref{asm:curv}, regret of updates (\ref{update:ONS}) satisfies
    \begin{align*}
        2R_T(\w) \leq \|\w\|_{\bH^{(0)}}^2 + R_G + R_D,
    \end{align*}
    where
    \begin{align*}
        &R_G = \sum_{t=1}^T (\g^\ut)^\top(\bH^{(t)})^{-1}\g^\ut,  \\
        &R_D = \sum_{t=1}^T (\w^\ut-\w)^\top(\bH^\ut-\bH^\utm-\mu^\ut\g^\ut(\g^\ut)^\top)(\w^\ut-\w).
    \end{align*}
\end{lem}
Then we prove the regret bound for RFD-SON based on Lemma \ref{lemma:regret} and property of RFD.
\begin{proof}
    Let $\V^\ut_\perp$ be the orthogonal complement of $\V^\ut_{m-1}$'s column space,
    that is $\V^\ut_{m-1}(\V^\ut_{m-1})^\top+\V^\ut_\perp(\V^\ut_\perp)^\top =\I_d$, then we have
    \begin{align}
           & \bH^\ut - \bH^\utm \nonumber\\
          =& \alpha^\ut\I + (\B^\ut)^\top\B^\ut - \alpha^\utm\I - (\B^\utm)^\top\B^\utm \nonumber\\
          =& \frac{1}{2}(\sigma_m^\utm)^2\I -(\sigma_m^\utm)^2\V^\utm_{m-1}(\V^\utm_{m-1})^\top
          + (\mu_t+\eta_t)\g^\ut(\g^\ut)^\top \nonumber\\
          = &  \frac{1}{2}(\sigma_m^\utm)^2\big[\V^\utm_\perp(\V_\perp^\utm)^\top-\V^\utm_{m-1}(\V^\utm_{m-1})^\top\big]
          + (\mu_t+\eta_t)\g^\ut(\g^\ut)^\top. \label{eq:hessian}
    \end{align}
    Since $\bH^\ut$ is positive semidefinite for any $t$, we have $(\bH^\ut)^\dag=(\bH^\ut)^{-1}$.
    Combining with Lemma \ref{lemma:regret}, we have
    \begin{align*}
        2R_T(\w) \leq \alpha_0 \|\w\|^2 + R_G + R_D,
    \end{align*}
    where
    \begin{align*}
        R_G = \sum_{t=1}^T (\g^\ut)^\top(\bH^\ut)^{-1}\g^\ut,
    \end{align*}
    and
    \begin{align*}
        R_D = \sum_{t=1}^T  (\w^\ut-\w)^\top[\bH^\ut - \bH^\utm  - \mu_t\g^\ut(\g^\ut)^\top](\w^\ut-\w).
    \end{align*}
    We can bound $R_G$ as follows
    \begin{align*}
         & \sum_{t=1}^T (\g^\ut)^\top(\bH^\ut)^{-1}\g^\ut \\
        =& \sum_{t=1}^T \big\langle (\bH^\ut)^{-1}, \g^\ut(\g^\ut)^\top \big\rangle \\
        =& \sum_{t=1}^T \frac{1}{\mu_t+\eta_t} \big\langle (\bH^\ut)^{-1}, \bH^\ut - \bH^\utm +   \frac{1}{2}(\sigma_m^\utm)^2[\V^\utm_\perp(\V^\utm_\perp)^\top-\V^\utm(\V^\utm)^\top\big] \big\rangle \\
        \leq& \frac{1}{\mu+\eta_T}\sum_{t=1}^T
                \big\langle (\bH^\ut)^{-1}, \bH^\ut - \bH^\utm
            + \frac{1}{2}(\sigma_m^\utm)^2\V^\utm(\V^\utm)^\top \big\rangle \\
        =& \frac{1}{\mu+\eta_T}\sum_{t=1}^T \Big[\big\langle (\bH^\ut)^{-1}, \bH^\ut - \bH^\utm \big\rangle
            + \frac{1}{2}(\sigma_m^\utm)^2\tr\big(\V^\utm(\bH^\ut)^{-1}(\V^\utm)^\top\big) \Big].
    \end{align*}
    The above equalities come from the properties of trace operator and (\ref{eq:hessian}) and
    the inequality is due to the fact that $\eta_t$ is non-increasing.

    The term $\sum_{t=1}^T \langle (\bH^\ut)^{-1}, \bH^\ut - \bH^\utm \rangle$ can be bounded as
    \begin{align*}
             \sum_{t=1}^T \langle (\bH^\ut)^{-1}, \bH^\ut - \bH^\utm \rangle
        & \leq  \sum_{t=1}^T \ln \frac{\det(\bH^\ut)}{\det(\bH^\utm)} \\
        & = \ln \frac{\det(\bH^{(T)})}{\det(\bH^{(0)})} \\
        & = \ln \frac{\prod_{i=1}^d\sigma_i(\bH^{(T)})}{{\alpha_0}} \\
        & = \sum_{i=1}^d\ln \frac{\sigma_i\big((\B^{(T)})^\top\B^{(T)}+\alpha^{(T)}\I_d\big)}{{\alpha_0}} \\
        & = \sum_{i=1}^m\ln \frac{\sigma_i\big((\B^{(T)})^\top\B^{(T)}\big)+\alpha^{(T)}}{{\alpha_0}} +
                (d-m)\ln\frac{\alpha^{(T)}}{{\alpha_0}} \\
        & \leq m \ln  \frac{\sum_{i=1}^m [\sigma_i\big((\B^{(T)})^\top\B^{(T)}\big)+\alpha^{(T)}]}{m\alpha_0} +
                (d-m)\ln\frac{\alpha^{(T)}}{{\alpha_0}} \\
         & = m \ln \Big(\frac{\tr\big((\B^{(T)})^\top\B^{(T)}\big)}{m\alpha_0}+\frac{\alpha^{(T)}}{\alpha_0}\Big) + (d-m)\ln\frac{\alpha^{(T)}}{{\alpha_0}}. \\
    \end{align*}
    The first inequality is obtained by the concavity of the log determinant function \citep{DBLP:journals/eor/Lemarechal06},
    the second inequality comes from the Jensen's inequality and the other steps are based on the procedure of the algorithm.

    The other one $\frac{1}{2}\sum_{t=1}^T(\sigma_m^\ut)^2\tr\big(\V^\ut(\bH^\ut)^{-1}(\V^\ut)^\top\big)$ can be bounded as
    \begin{align}
            \frac{1}{2}\sum_{t=1}^T(\sigma_m^\ut)^2\tr\big(\V^\ut(\bH^\ut)^{-1}(\V^\ut)^\top\big) \nonumber
        & \leq \frac{1}{2}\sum_{t=1}^T\frac{(\sigma_m^\ut)^2}{\alpha^\ut} \tr\big(\V^\ut(\V^\ut)^\top\big) \nonumber \\
        & = \frac{m}{2}\sum_{t=1}^T\frac{(\sigma_m^\ut)^2}{\alpha^\ut}.
    \end{align}
    Hence, we have
    \begin{align}
            R_G
        \leq  \frac{1}{\mu+\eta_T}\Bigg[m \ln \Big(\tr\big((\B^{(T)})^\top\B^{(T)}\big)+\frac{\alpha^{(T)}}{\alpha_0}\Big)
             + (d-m)\ln\frac{\alpha^{(T)}}{{\alpha_0}} + \frac{m}{2}\sum_{t=1}^T\frac{(\sigma_m^\ut)^2}{\alpha^\ut}\Bigg]. \label{ieq:RG}
    \end{align}
    Then we bound the term $R_D$ by using equation (\ref{eq:hessian}), Assumption \ref{asm:bound} and Assumption \ref{asm:curv}.
    \begin{align}
         R_D =& \sum_{t=1}^T  (\w^\ut-\w)^\top
            \big[\eta_t \g^\ut(\g^\ut)^\top + \frac{1}{2}(\sigma_m^\utm)^2\I
             - \V^\utm(\V^\utm)^\top\big](\w^\ut-\w) \nonumber\\
          \leq& \sum_{t=1}^T \eta_t(\w^\ut-\w)^\top \g^\ut(\g^\ut)^\top(\w^\ut-\w)
             + \frac{1}{2}\sum_{t=1}^T(\sigma_m^\utm)^2(\w^\ut-\w)^\top(\w^\ut-\w) \nonumber\\
          \leq& 4(CL)^2 \sum_{t=1}^T \eta_t + 2C^2 \sum_{t=1}^T(\sigma_m^\utm)^2. \label{ieq:RD}
    \end{align}
    Finally, we obtain the result by combining (\ref{ieq:RG}) and (\ref{ieq:RD}).
\end{proof}
Additionally, the term $\sum_{t=1}^T\frac{(\sigma_m^\ut)^2}{\alpha^\ut}$ in $\Omega_{\rm RFD}$ can be bounded by $\fO(\ln T)$
if we further assume all $\sigma_m^\ut$ are bounded by positive constants.
Exactly, suppose that $0 < C_1 \leq \big(\sigma_m^\ut\big)^2 \leq C_2$ for any $t=0,\dots,T-1$, then we have
\begin{align*}
   \sum_{t=1}^T\frac{(\sigma_m^\ut)^2}{\alpha^\ut}
& =  \sum_{t=1}^T\frac{(\sigma_m^\ut)^2}{\alpha_0+\frac{1}{2}\sum_{i=0}^{t-1}\big(\sigma_m^{(i)}\big)^2} \\
& \leq  \sum_{t=1}^T\frac{C_2}{\alpha_0+\frac{1}{2}C_1t} \\
& \leq  \frac{2C_2}{C_1}\sum_{t=1}^T\frac{1}{t} \\
& = \fO(\ln T).
\end{align*}
The last inequality is due to the property of harmonic series.

\section{The Proof of Theorem \ref{thm:regret1}}
Considering the update without positive semidefinite assumption on $\bH^\ut$
\begin{align}
     \bu^\utp &=  \w^\ut-\beta_t(\bH^\ut)^\dag \g^\ut, \nonumber \\
     \w^\utp &= \argmin_{\w\in\fK_{t+1}} \| \w - \bu^\utp \|_{\bH^\ut}, \label{update:pinv}
\end{align}
we have the results as follows.
\begin{lem}[Appendix D of \cite{luo2016efficient}]\label{lemma:regretpinv}
Let ${\hat\bH}^\ut=\sum_{s=1}^t \g^\ut(\g^\ut)^\top$ with $\rk({\hat\bH}^{(T)})=r$ and $\sigma^*$ be the minimum among
the smallest non-zeros singular values of ${\hat\bH}^{(T)}$. Then the regret of update (\ref{update:pinv}) satisfies
\begin{align}
     R_T(\w) \leq  2(CL)^2 \sum_{t=1}^T\eta^\ut + \frac{1}{2}\sum_{t=1}^T (\g^\ut)^\top(\bH^\ut)^\dag\g^\ut. \label{bound:pinv01}
\end{align}
and
\begin{align}
    \sum_{t=1}^T \g_t^\top({\hat\bH}^\ut)^\dag\g_t
    \leq m-1 + \frac{m(m-1)}{2}\ln\Big(1+\frac{2\sum_{t=1}^T\|\g^\ut\|_2^2}{m(m-1)\sigma^*}\Big) \label{bound:pinv02}
\end{align}
\end{lem}
Then we can derive Theorem \ref{thm:regret1}.
\begin{proof}
    For any $t\leq T'$, RFD-SON with $\alpha^{(0)}=0$ satisfies $\rk(\hB^\ut)\leq m-1$. Hence we have
    \begin{align}
            \bH^\ut &=(\B^\ut)^\top\B^\ut \nonumber\\
            &= \sum_{s=1}^t \frac{1}{\eta_t+\mu_t} (\g^\ut)(\g^\ut)^\top \nonumber\\
            & \succeq  \frac{1}{\eta_1+\mu'} \sum_{s=1}^t (\g^\ut)(\g^\ut)^\top \nonumber\\
            & = \frac{1}{\eta_1+\mu'}{\hat\bH}^\ut \label{bound:pinv03}.
    \end{align}
    Then we can bound the regret for the first $T'$ iterations
    \begin{align*}
         R_{1:T'}(\w)
        & \leq 2(CL)^2 \sum_{t=1}^{T'}\eta^\ut +  \frac{1}{2}\sum_{t=1}^{T'} (\g^\ut)^\top(\bH^\ut)^\dag\g^\ut \\
        & \leq 2(CL)^2 \sum_{t=1}^{T'}\eta^\ut+ \frac{1}{2(\eta_1+\mu')}\sum_{t=1}^{T'} (\g^\ut)^\top({\hat\bH}^\ut)^\dag\g^\ut  \\
        & \leq 2(CL)^2 \sum_{t=1}^{T'}\eta^\ut+ \frac{m-1}{2(\eta_1+\mu')} +
        \frac{m(m-1)}{2(\eta_1+\mu')}\ln\Big(1+\frac{2\sum_{t=1}^{T'}\|\g^\ut\|_2^2}{m(m-1)\sigma^*}\Big).
    \end{align*}
    The first inequality is based on the inequality (\ref{bound:pinv01}).
    The second inequality comes from (\ref{bound:pinv03}) that is $\bH^\ut \succeq \frac{1}{\eta_1+\mu'}{\hat\bH}^\ut$.
    And the last one is due to the result (\ref{bound:pinv02}).
\end{proof}

\section{The Proof of Theorem \ref{thm:regret2}}
\begin{proof}
    Since $\bH^\ut\succ\bz$ for $t\geq T'$, by similar proof of Theorem \ref{thm:regret}, we have
    \begin{align}
        2R_{T'+1:T}(\w) \leq \alpha_0 \|\w^{(T')}\|_{\bH^{(T')}}^2 + R'_G + R'_D, \label{ieq:RGD0}
    \end{align}
    where
    \begin{align}
          R'_G
          = & \sum_{t=T'+1}^T (\g^\ut)^\top(\bH^\ut)^{-1}\g^\ut \nonumber\\
        \leq & \frac{1}{\mu+\eta_T}
            \Bigg[m \ln \Big(\frac{\tr\big((\B^{(T)})^\top\B^{(T)}\big)}{m\alpha^{(T')}}+\frac{\alpha^{(T)}}{\alpha'_0}\Big)
             + (d-m)\ln\frac{\alpha^{(T)}}{{\alpha'_0}} + \frac{m}{2}\sum_{t=T'+1}^T\frac{(\sigma_m^\ut)^2}{\alpha^\ut}\Bigg] \label{ieq:RGD01}, \\
        R'_D =& \sum_{t=T'+1}^T  (\w^\ut-\w)^\top[\bH^\ut - \bH^\utm  - \mu_t\g^\ut(\g^\ut)^\top](\w^\ut-\w) \nonumber\\
             \leq&  4(CL)^2 \sum_{t=T'+1}^T \eta_t + 2C^2 \sum_{t=T'+1}^T(\sigma_m^\utm)^2. \label{ieq:RGD02}
    \end{align}
    Combining Theorem \ref{thm:regret1}, (\ref{ieq:RGD0}), (\ref{ieq:RGD01}) and (\ref{ieq:RGD02}),
    we have the final result of (\ref{bound:RFDSON0}).
\end{proof}

\bibliography{reference}

\end{document}